\documentclass{article}

\usepackage[preprint]{neurips_2026}

\usepackage[utf8]{inputenc} %
\usepackage[T1]{fontenc}    %
\usepackage{hyperref}       %
\usepackage{url}            %
\usepackage{booktabs}       %
\usepackage{amsfonts}       %
\usepackage{nicefrac}       %
\usepackage{microtype}      %
\usepackage{xcolor}         %

\usepackage{mathtools}
\usepackage{amsthm}
\usepackage{amsmath}
\usepackage{amsxtra}
\usepackage{enumitem}      
\usepackage{amsthm}
\usepackage{mathtools}
\usepackage{bbm}
\usepackage{amsfonts}
\usepackage{amssymb}

\usepackage{MnSymbol} %

\usepackage{xpatch}

\xpatchcmd{\proof}{\itshape}{\normalfont\proofnameformat}{}{}
\newcommand{\proofnameformat}{\bfseries}

\usepackage{prettyref}
\newcommand{\pref}[1]{\prettyref{#1}}

\newcommand{\savehyperref}[2]{\texorpdfstring{\hyperref[#1]{#2}}{#2}}
\newrefformat{eq}{\savehyperref{#1}{\textup{(\ref*{#1})}}}
\newrefformat{eqn}{\savehyperref{#1}{Equation~\ref*{#1}}}
\newrefformat{con}{\savehyperref{#1}{Conjecture~\ref*{#1}}}
\newrefformat{lem}{\savehyperref{#1}{Lemma~\ref*{#1}}}
\newrefformat{def}{\savehyperref{#1}{Definition~\ref*{#1}}}
\newrefformat{line}{\savehyperref{#1}{line~\ref*{#1}}}
\newrefformat{thm}{\savehyperref{#1}{Theorem~\ref*{#1}}}
\newrefformat{corr}{\savehyperref{#1}{Corollary~\ref*{#1}}}
\newrefformat{sec}{\savehyperref{#1}{Section~\ref*{#1}}}
\newrefformat{app}{\savehyperref{#1}{Appendix~\ref*{#1}}}
\newrefformat{ass}{\savehyperref{#1}{Assumption~\ref*{#1}}}
\newrefformat{ex}{\savehyperref{#1}{Example~\ref*{#1}}}
\newrefformat{fig}{\savehyperref{#1}{Figure~\ref*{#1}}}
\newrefformat{tab}{\savehyperref{#1}{Table~\ref*{#1}}}
\newrefformat{alg}{\savehyperref{#1}{Algorithm~\ref*{#1}}}
\newrefformat{rem}{\savehyperref{#1}{Remark~\ref*{#1}}}
\newrefformat{conj}{\savehyperref{#1}{Conjecture~\ref*{#1}}}
\newrefformat{prop}{\savehyperref{#1}{Proposition~\ref*{#1}}}
\newrefformat{proto}{\savehyperref{#1}{Protocol~\ref*{#1}}}
\newrefformat{prob}{\savehyperref{#1}{Problem~\ref*{#1}}}
\newrefformat{claim}{\savehyperref{#1}{Claim~\ref*{#1}}}
\newrefformat{eq}{\savehyperref{#1}{Equation~\ref*{#1}}}

\allowdisplaybreaks

\DeclarePairedDelimiter{\crl}{\{}{\}}

\let\Pr\undefined

\DeclareMathOperator{\Pr}{Pr}

\newcommand{\ind}[1]{\mathbbm{1}\crl*{#1}}    %

\def\ddefloop#1{\ifx\ddefloop#1\else\ddef{#1}\expandafter\ddefloop\fi}
\def\ddef#1{\expandafter\def\csname bb#1\endcsname{\ensuremath{\mathbb{#1}}}}
\ddefloop ABCDEFGHIJKLMNOPQRSTUVWXYZ\ddefloop
\def\ddefloop#1{\ifx\ddefloop#1\else\ddef{#1}\expandafter\ddefloop\fi}
\def\ddef#1{\expandafter\def\csname b#1\endcsname{\ensuremath{\mathbf{#1}}}}
\ddefloop ABCDEFGHIJKLMNOPQRSTUVWXYZ\ddefloop
\def\ddef#1{\expandafter\def\csname c#1\endcsname{\ensuremath{\mathcal{#1}}}}
\ddefloop ABCDEFGHIJKLMNOPQRSTUVWXYZ\ddefloop
\def\ddef#1{\expandafter\def\csname h#1\endcsname{\ensuremath{\widehat{#1}}}}
\ddefloop ABCDEFGHIJKLMNOPQRSTUVWXYZabcdefghijklmnopqrsuvwxyz\ddefloop    %
\def\ddef#1{\expandafter\def\csname hc#1\endcsname{\ensuremath{\widehat{\mathcal{#1}}}}}
\ddefloop ABCDEFGHIJKLMNOPQRSTUVWXYZ\ddefloop
\def\ddef#1{\expandafter\def\csname t#1\endcsname{\ensuremath{\widetilde{#1}}}}
\ddefloop ABCDEFGHIJKLMNOPQRSTUVWXYZ\ddefloop
\def\ddef#1{\expandafter\def\csname tc#1\endcsname{\ensuremath{\widetilde{\mathcal{#1}}}}}
\ddefloop ABCDEFGHIJKLMNOPQRSTUVWXYZ\ddefloop

\usepackage{tikz}

\newcommand{\grad}{\nabla}

\newcommand{\E}{\mathbb{E}}
\newcommand{\V}{\mathrm{Var}}
\newcommand{\Cov}{\mathrm{Cov}}

\newcommand{\wE}{w_{\mathrm{\textsc{erm}}}}

\newcommand{\cc}{\rho_{\mathrm{cc}}}

\newcommand{\I}{\mathrm{MI}}

\newcommand{\diam}{\mathcal{D}}

\newcommand{\Cw}{c_0(w,q,g,\Delta)}
\newcommand{\Cwone}{c_0(1,q,g,\Delta)}
\newcommand{\Cx}{c_0(1,x,g,\Delta)}

\newcommand{\Pfam}{\bm{\mathcal{P}}}

\newcommand{\gmm}{\hat{\theta}^{\textsc{gmm}}}

\newcommand{\bayes}{\hat{\theta}^{*}}

\usepackage{etoc}
\usepackage{bm}
\usepackage{enumitem}
\usepackage{thm-restate}

\usepackage{physics}
\usepackage{natbib}
\usepackage{amsmath}
\usepackage{amssymb}
\usepackage{mathtools}
\usepackage{amsthm}

\theoremstyle{plain}
\newtheorem{theorem}{Theorem}[section]

\newtheorem{lemma}[theorem]{Lemma}

\theoremstyle{definition}
\newtheorem{definition}[theorem]{Definition}
\newtheorem{assumption}[theorem]{Assumption}
\newtheorem{remark}[theorem]{Remark}

\usepackage{color-edits}  
\addauthor[Linda Edit]{LL}{blue}
\addauthor[Karthik Edit]{KS}{red}

\title{Predictability as a Fine-Grained Measure for Privacy}

\author{%
  Linda Lu \\
  Cornell University \\
  \texttt{lgl46@cornell.edu} \\
  \And
  Karthik Sridharan \\
  Cornell University \\
  \texttt{sridharan@cs.cornell.edu} \\
}

\begin{document}

\maketitle

\begin{abstract}
    Differential privacy (DP) ensures rigorous individual-level privacy guarantees against even the most knowledgeable attackers, but its worst-case nature can impose a costly privacy-accuracy tradeoff. We introduce privacy via predictability, a fine-grained framework that explicitly incorporates the attacker’s core knowledge, a compromised portion of the dataset generated by a stochastic process, and a specified family of sensitive queries. Predictability measures privacy leakage as the incremental gain in an attacker’s ability to predict sensitive information about unknown individuals after observing the algorithm’s output, beyond what can already be inferred from the compromised data. We show that predictability and DP are generally incomparable: each can be small while the other is large. However, in the worst-case regime where all but one individual is compromised, and all binary queries are considered sensitive, predictability implies mutual-information DP. More generally, predictability provides a finer-grained privacy metric tailored to specific sensitive information and specific attacker models. We introduce a general framework, using the generalized method of moments (GMM), to analyze asymptotic predictability when the compromised data is generated by a stationary, ergodic, mixing process. Using this analysis, we derive a predictability-calibrated output perturbation scheme for ERM. Our approach is complementary to DP and can be used alongside DP to provide additional fine-grained privacy control.
\end{abstract}

\section{Introduction}

Data privacy has always been a central problem in machine learning, and it is even more critical today. Today's machine learning models consume internet-scale data, much of which could be sensitive. Researchers have reconstructed training data from deep vision models \citep{zhu2019visionleakage, zhu2020visionleakage} and large language models \citep{carlini2021llmleakage, wang2024llmleakage}; thus, privacy surrounding these systems is increasingly paramount.

Differential privacy (DP) \citep{dwork2006dp, dwork2006differentialprivacy, dwork2008dpsurvey, dwork2014privacy} has long been considered the gold standard for privacy, providing rigorous privacy guarantees for unknown individuals. Differentially private algorithms inject noise during model training to ensure privacy \citep{dwork2014privacy, abadi2016dldp, chaudhuri2011dperm}. Many differentially private algorithms 
have been deployed in practice \citep{appleprivacy, googleprivacy, wikipediaprivacy, rogers2020linkedins, tumultanalyticswhitepaper}. However, the noise injected by differentially private algorithms degrades the accuracy, and the resulting tradeoff between privacy and performance is very costly. Due to this costly tradeoff, differentially private algorithms are often deployed with a large privacy parameter $\varepsilon$, as large as $\varepsilon = 20$ \citep{dpregistry, rogers2020linkedins, census}. 
DP guarantees that, even against the most knowledgeable attacker (one with full knowledge of every individual but one), the attacker's posterior beliefs about any feature of the unknown individual change by at most a factor of $e^\varepsilon$ after observing the algorithm. For practical systems, the privacy guarantee from DP for these large $\varepsilon$ is extremely weak. A system designer may want stronger privacy protections for specific types of sensitive information. \textit{Can we provide a finer-grained privacy metric tailored to specific sensitive information?}
Furthermore, this worst-case attacker that compromises every individual but one is extremely improbable; we would expect an attacker to either compromise some fraction of the dataset or the entire dataset. In the latter case, we cannot hope to provide any privacy anyway. Training data for large-scale machine learning is often sharded across many servers \citep{mcmahan2017datasharding, jiang2023datasharding}. Data breaches of these systems often target a single misconfigured data server containing a subset of training data \citep{newman2019databreach, upguard2019databreach, sasson2023databreach, mauran2023fbdatabreach}. \textit{If only a fraction of the dataset is compromised, how do we ensure privacy for the entire population of unknown individuals?} Consider a differentially private release of the proportion of smokers in a dataset. If the portion of the dataset compromised by the attacker contains a small number of smokers, then the release of the proportion of smokers reveals that there must be a large proportion of smokers among the unknown individuals. Even though DP ensures that the attacker cannot identify whether any particular unknown individual is a smoker, a lot of information is revealed about the overall set of unknown individuals.

We propose studying the privacy guarantees of algorithms through the \textit{predictability} of sensitive information on the set of unknown individuals, given the attacker's knowledge. Predictability can be used to study the fine-grained guarantees of differentially private algorithms against specific types of attackers and sensitive queries, to inform users of more nuanced privacy guarantees. Predictability also informs us of new mechanisms to protect specific sensitive queries on the set of unknown individuals against specific attackers. 
Our contributions are as follows.
\vspace*{-1mm}
\begin{itemize}[nosep, leftmargin=1em]
    \item We introduce predictability, a privacy metric that explicitly incorporates the attacker's knowledge (a portion of compromised data generated by a stochastic process) and the set of sensitive queries, and measures how well the attacker can predict sensitive information about the set of unknown individuals before and after observing the algorithm.
    \item We show that predictability is incomparable to (neither implies nor is implied by) the privacy guarantees of differential privacy. However, under worst-case assumptions on the attacker's knowledge, predictability implies mutual information differential privacy.
    \item We introduce a framework using the generalized method of moments to compute the asymptotic predictability of an algorithm with respect to a single sensitive query when the attacker's knowledge is generated by a known stationary, ergodic, strongly mixing process. We prove that, for a noisy algorithm release, asymptotic predictability is governed by a canonical correlation term between the query and moment condition of the algorithm, along with a noise-attenuation term.
    \item We derive an improved predictability-calibrated output perturbation scheme for ERM. The scheme adapts the noise covariance to the loss curvature, gradient covariance, and attacker process, yielding improved accuracy bounds for linear regression compared with isotropic perturbation.
    \item We extend the framework to the asymptotic predictability for query families (finite, linear, Lipschitz) and the asymptotic predictability when the process is unknown.
\end{itemize}

\subsection{Related Work}
Many variations of DP have been proposed, modifying its stringent requirements. Some examples include R\'{e}nyi DP \citep{Mironov_2017_renyidp}, KL-DP \citep{barber2014klprivacystatisticalriskformalisms}, mutual information DP \citep{cuff2016differential}, distributional DP \citep{blum2008learningtheory}, $f$-DP \citep{dong2019gaussiandifferentialprivacy}, random DP \citep{hall2011randomdifferentialprivacy}, label DP \citep{ghazi2021deeplearninglabeldifferential}, coupled worlds DP \citep{bassily2013coupled}, zero knowledge privacy \citep{gehrke2011zkprivacy}, and Bayesian DP \citep{triastcyn2020bayesiandifferentialprivacymachine}. Some variations of DP also account for the attacker's knowledge. Blowfish privacy \citep{he2014blowfish} specifies the attacker's knowledge through deterministic constraints on the dataset. Inferential privacy \citep{ghosh2017inferentialprivacyguaranteesdifferentially} specifies the attacker's knowledge through prior distributions on the data.
Bounded-leakage DP \citep{ligett2020boundedleakage} assumes the attacker has additional information about the dataset through a leakage function.
In the streaming setting, pan-privacy \citep{Dwork2010PanPrivateSA} assumes the attacker's knowledge is the internal state of the algorithm. %
In distributed learning, shuffle DP \citep{cheu2019distributeddp} assumes the attacker's knowledge is a shuffled set of messages. 
Pufferfish privacy \citep{kifer2014pufferfish} specifies the attacker's knowledge through prior distributions on the data and specifies sensitive information through secret pairs. However, the mechanisms proposed for the Pufferfish framework are inefficient to compute without distributional assumptions on the dataset \citep{kifer2014pufferfish, Nuradha_2023_pufferfish}. We find that predictability can be computed efficiently, given any dataset.
DP has also been extended to provide privacy for the membership of groups of individuals through group privacy \citep{dwork2014privacy}. Many of the above definitions have extensions to analogous group privacy definitions. Additionally, attribute privacy \citep{zhang2022attributeprivacy}, an instantiation of the Pufferfish framework, has been proposed to provide privacy for global properties of a dataset.

\section{Privacy via Predictability} \label{sec:def-pred}

The goal of a private learning algorithm is to prevent an attacker from extracting sensitive unknown information about a specific or random individual from the uncompromised population through the algorithm's output. To better understand this goal, we take the perspective of the attacker. Before observing the algorithm's output, the attacker has some initial core knowledge about the dataset. We model this via a stochastic process that reveals a compromised portion of the dataset to the attacker. Let $\mathcal{X}$ be an instance space, and let $S = (x_1, \dots, x_N) \in \mathcal{X}^N$ be a dataset over $N$ individuals.

\begin{definition} [Core Knowledge]
    Let $\mathcal{P}$ be a stochastic process that, given a dataset $S$, generates a sequence $x_1, \dots, x_n$, such that each $x_i \in S$. The attacker's core knowledge is $C(S) = \{x_1, \dots, x_n\}$.
\end{definition}
\vspace*{-1mm}
A key example is an attacker that compromises a random shard of the dataset. This can be modeled by the stochastic process $\mathcal{P}_{\mathrm{RS}}$ that produces each $x_i$ by uniformly sampling from $S$ without replacement. Generally, $\mathcal{P}$ can be much more complex. $\mathcal{P}$ could have a Markovian structure such that the choice of each $x_i$ depends on the previously selected $x_{i-1}$, or $\mathcal{P}$ could be biased towards a (potentially adversarial) subset of $S$. We keep $\mathcal{P}$ general to capture a variety of potential core knowledges. 

Let $A : \mathcal{X}^N \mapsto \Delta (\mathcal{F})$ be a randomized algorithm that takes in a dataset $S$ and outputs a distribution over a model class $\mathcal{F}$. We assume that the designer of $A$ knows the process $\mathcal{P}$ but does not know the realized $C(S)$. This captures a setting where the algorithm designer knows how a data breach occurs but does not know the exact set of compromised data. In \pref{sec:query-families}, we relax the assumption that the process is known to the algorithm designer. The algorithm designer wants to protect the privacy of a family of sensitive queries $\mathcal{Q}$. 
Using their core knowledge $C(S)$ and $A(S)$, the attacker will attempt to answer a query $q \in \mathcal{Q}$ about an unknown individual $x$ as accurately as they can, where 
\begin{equation}
    x \sim \Pi,\; \Pi = \mathrm{Unif}(S \setminus C(S)).
\end{equation}
We select $x$ uniformly at random from the unknown individuals to measure the attacker's prediction for the average unknown individual.

\begin{definition}[Bayes Optimal Estimator]
    Given knowledge $K(S)$ and query $q$, the Bayes optimal estimator for loss $\ell: \Theta \times \mathcal{Y} \mapsto \mathbb{R}$ is $\bayes_{q \mid K(S)} = \arg \min_{\hat{\theta}} \mathop{\E}_{x \sim \Pi} [\ell (\hat{\theta}, q(x)) \mid K(S)]$.
\end{definition}
\vspace*{-1mm}
The Bayes optimal estimator represents the best prediction of $q(x)$ that an attacker could make. We primarily use $K(S) = C(S)$ or $K(S) = (C(S), A(S))$ to compare the difference in prediction before and after observing $A(S)$. The information contained in $C(S)$ is already exposed to the attacker. The best that a private algorithm can do is to ensure that the attacker does not gain any additional information through the algorithm's output. We formally define this goal as predictability.

\begin{definition} [Predictability]
    Let process $\mathcal{P}$ produce the core knowledge of the attacker $C(S)$. Let $\mathcal{Q} = \{q: \mathcal{X} \mapsto \mathcal{Y}\}$ be a family of sensitive queries. Let $\ell: \Theta \times \mathcal{Y} \mapsto \mathbb{R}$ be a loss function.  Let $\bayes_{q \mid C(S)}$ and $\bayes_{q \mid C(S), A(S)}$ be the Bayes optimal estimator for $q(x)$ given $C(S)$, and given $C(S)$ and $A(S)$, respectively. Then algorithm $A$ has $\gamma(\mathcal{P}, \mathcal{Q}, \ell, A)$-predictability $\gamma$ if
    \begin{align*}
        \gamma(\mathcal{P}, \mathcal{Q}, \ell, A) = \sup_{S} \mathop{\E}_{C(S) \sim \mathcal{P}, A}[ \sup_{\substack{q \in \mathcal{Q}}} \mathop{\E}_{\substack{x \sim \Pi}} [\ell(\bayes_{q \mid C(S)}, q(x)) - \ell(\bayes_{q \mid C(S), A(S)}, q(x))]] \leq \gamma.
    \end{align*}
\end{definition}

Given core knowledge $C(S)$, predictability measures the increase in accuracy of the attacker's prediction for any sensitive query $q \in \mathcal{Q}$ on an unknown individual before and after observing the algorithm's output. %
If the difference in prediction is small, the algorithm's output reveals very little additional information beyond what an attacker could have already learned about the average unknown individual from $C(S)$. 
Predictability as a measure of the additional information exposed by $A(S)$ about the unknown individuals with respect to $q(x)$ is formalized when $\ell$ is the log loss. 
\begin{restatable}{theorem}{predtokl} \label{thm:predtokl}
    Let $P_X$ denote the distribution of $X$. If $\mathcal{Q} \subseteq \{q : \mathcal{X} \rightarrow [K]\}$ is a family of categorical queries and $\ell(\hat{\theta}, y) = - y_i \mathrm{log} (\hat{\theta}_i) $ is the log loss, predictability is equivalent to
    \begin{align*}
        \gamma(\mathcal{P}, \mathcal{Q}, \ell, A) &= \sup_{\substack{S}} \mathop{\E}_{C(S) \sim \mathcal{P}, A}[ \sup_{\substack{q \in \mathcal{Q}}}  \mathop{\E}_{\substack{x \sim \Pi}} [D_{\textsc{kl}}(P_{q(x) \mid C(S), A(S)} \| P_{q(x) \mid{} C(S)})]] \\
        &\geq \sup_{\substack{S}} \sup_{q \in \mathcal{Q}} \I(q(x) ; A(S) \mid C(S))
    \end{align*}
\end{restatable}

Under the log loss, predictability ensures that the amount of extra information $A(S)$ reveals about a query $q(x)$ on the \textit{average} unknown individual beyond what is present in the core knowledge $C(S)$ is bounded by $\gamma$. %
We can think of predictability as bounding the information exposure of a specific sensitive query over the whole set of unknown individuals. Compare this to differential privacy, which guarantees that information revealed by $A(S)$ is bounded by $\varepsilon$ for \textit{each} unknown individual. Both privacy notions bound a form of information leakage, but are generally incomparable.

\begin{restatable}{theorem}{dppredcontrast} \label{thm:dp-pred-contrast}
    Let $\ell(\hat{\theta}, y) = (\hat{\theta}-y)^2$ be the squared loss. There exists a process $\mathcal{P}$, query $q$, and algorithm $A$ such that $A$ is $\varepsilon$-differentially private, $\varepsilon = O(\frac{1}{\sqrt{N}})$, but $\gamma(\mathcal{P}, \mathcal{Q}, \ell, A)$-predictability is $\Omega(1)$. Conversely, there exists a process $\mathcal{P}$, query $q$, and algorithm $A$ such that $\gamma(\mathcal{P}, \mathcal{Q}, \ell, A)$-predictability is $O(\frac{1}{N})$, but $A$ is not differentially private. 
\end{restatable}
Differential privacy measures, relative to the information in the rest of the dataset $S \setminus \{x_i\}$, how much additional information can be learned about the last remaining individual. Predictability measures, relative to the information in $C(S)$, how much additional information can be learned about the overall set of unknown individuals. When $C(S)$ and $S \setminus \{x_i\}$ are very different, then the privacy leakage of predictability and differential privacy are measured relative to very different quantities. Consider when $C(S)$ is very uninformative about the dataset (i.e. $C(S) = \emptyset$), then a differentially private algorithm that does not reveal much about any particular individual still reveals a lot about the overall dataset of unknown individuals.

We also point out a key distinction between predictability and group privacy \citep{dwork2014privacy}, a generalization of differential privacy. The individuals in the unexposed shards could be treated as a group. Group privacy guarantees that (almost) no information can be learned about the group of unknown individuals from the algorithm, including all group-level statistics. However, under predictability, if a group-level statistic on the set of unknown individuals can be well-estimated using only $C(S)$ (i.e. $C(S)$ is a random subset), then this information can be revealed by the algorithm because this information is already revealed to the attacker through $C(S)$. Thus, an algorithm that reveals fairly accurate group-level statistics can maintain small predictability while providing little to no group privacy.

We find that in worst-case settings, predictability implies mutual information differential privacy. Consider a process $\mathcal{P}_{\mathrm{DP}}$ that reveals the worst-case subset of $n = N-1$ individuals from the dataset $S$ and all binary queries are considered sensitive. This is exactly the worst-case setting considered by differential privacy and its variants.

\begin{definition}[$\varepsilon$-mutual-information differential privacy \citep{cuff2016differential}] \label{def:mi-dp}
    A randomized algorithm $A(S)$ is $\varepsilon$-mutual-information differentially private (MI-DP) if %
    \begin{align*}
        \sup_{\chi \in \Delta (\mathcal{X}^N), i \in [N]} \I (x_i ; A(S) \mid{} S \setminus \{x_i\}) \leq \varepsilon.
    \end{align*}
\end{definition}

\begin{restatable}{theorem}{predimpliesdp} \label{thm:pred-implies-dp}
    Let $\mathcal{Q}_{\mathrm{bin}} = \{ q : \mathcal{X} \mapsto \{0,1\}\}$ be the set of all binary queries on $\mathcal{X}$ and let $\ell$ be the log loss. If algorithm $A$ has $\gamma(\mathcal{P}_{\mathrm{DP}}, \mathcal{Q}_{\mathrm{bin}}, \ell, A)$-predictability $\gamma$, then $A$ is $\gamma$-MI-DP.
\end{restatable}

Additionally, \cite{cuff2016differential} showed that $\varepsilon$-MI-DP $\Rightarrow$ $(0, \sqrt{2\varepsilon})$-DP, establishing a connection between MI-DP and DP. However, when $\varepsilon \geq 0.5$, which is very typical, $(0, \sqrt{2\varepsilon})$-DP provides no individual-level privacy. Thus, we do not find the connection to be meaningful, and we believe this emphasizes the distinctiveness between predictability and differential privacy.

Furthermore, the worst-case setting is overly stringent. An attacker may only have the ability to compromise a subset of the dataset via some process that is likely subject to stochasticity. An algorithm designer may only care about protecting privacy leakages of specific queries, but not all information leakages are created equal. Against these attackers and queries, predictability guarantees can be much more favorable. For the same algorithm, we can derive different predictability guarantees for different attacker models and query families, giving users a detailed understanding of an algorithm's privacy guarantees beyond the worst-case setting.

\begin{remark}
    In \pref{app:other-dists}, we discuss an extension of predictability beyond the estimation of $q(x)$ for the average unknown individual (in particular, when the attacker estimates $\E_{x \sim D'} [q(x)]$ for some other distribution $D'$ on the dataset $S$, instead of $\Pi$). This extension allows us to measure predictability guarantees for targeted subgroups of the dataset.
\end{remark}

In \pref{app:pp-and-composition}, we establish post-processing and composition properties of predictability, analogous to the properties of conditional mutual information from \cite{steinke2020reasoninggeneralizationconditionalmutual}.

\subsection{Computing Predictability} \label{sec:comp-pred}

We focus on understanding the privacy guarantees of large-scale machine learning systems. Thus, we present a general framework to analyze the large sample, asymptotic predictability of a query $q$ when $n, N \rightarrow \infty$. 
To bound predictability, we must bound the difference in loss between two estimators of $q$. For a convex and smooth loss (i.e. log loss or squared loss), the difference in the loss of the estimators can be characterized by the difference of the estimators themselves.

\begin{restatable}{lemma}{losslem} \label{lem:loss-lem}
    Let $\ell (\hat{\theta}) = \E_{x \sim \Pi}[\ell(\hat{\theta}, q(x)]$ be a convex, $H$-smooth loss. Let $\theta_0^q = \E_{x \sim \Pi}[q(x)]$. Then for any pair of estimators $\hat{\theta}_{q \mid C(S)}$ and $\hat{\theta}_{q \mid C(S), A(S)}$,
    \begin{align*}
        \mathop{\E}_{\substack{x_i \sim \Pi}}[\ell(\hat{\theta}_{q \mid C(S)}, q(x)) - \ell(\hat{\theta}_{q \mid C(S), A(S)}, q(x))] \leq H |\hat{\theta}_{q \mid C(S)} - \theta^q_0| \cdot |\hat{\theta}_{q \mid C(S)} - \hat{\theta}_{q \mid C(S), A(S)}|.
    \end{align*}
    For log loss, if the true $p_0 = \Pr[q(x)=1]$ satisfies $p_0 \in [\tau, 1-\tau]$, $H = \frac{1}{\tau^2}$. For squared loss, $H=2$.
\end{restatable}
In particular, the difference in loss can be upper-bounded by the variance of $\hat{\theta}_{q \mid C(S)}$ and $\hat{\theta}_{q \mid C(S), A(S)}$ around the true $\theta_0^q$. Thus, the best that an attacker could do is to select the estimator with the smallest variance among all possible estimators. In statistics, this is known as the efficient estimator. The efficient estimator represents the most accurate prediction that an attacker could make, given infinite computational resources. We will find that asymptotically, the efficient estimator coincides with the Bayes optimal estimator. Predictability only decreases if the attacker chooses a worse estimator. Importantly, we do not assume that the attacker uses the efficient estimator; it simply allows us to measure the worst-case increase in predictability.
If the variance of the efficient estimator given $C(S)$ and given $C(S)$ and $A(S)$ is known, then $|\hat{\theta}_{q \mid C(S)} - \theta^q_0|$ and $|\hat{\theta}_{q \mid C(S)} - \hat{\theta}_{q \mid C(S), A(S)}|$ can be bounded. Receiving $C(S)$ is equivalent to receiving a sample from dataset $S$ generated from a stochastic process. $A(S)$ then reveals some additional information about $S$. %
If the process is stationary, ergodic, and mixing, and the information from $A(S)$ can be captured by a moment condition, the behavior of the efficient estimator can be found through the generalized method of moments.

\section{Generalized Method of Moments} \label{sec:gmm}
The generalized method of moments (GMM) \citep{hansen1982gmm, hansen2010gmm, zsohar-gmm} is a method for parameter estimation in semiparametric statistics. In semiparametric statistics, the parameter of interest is finite-dimensional, but the data distribution is assumed to be infinite-dimensional and not limited to a specific parametric family; thus, standard maximum likelihood estimation cannot be applied. GMM allows population moment conditions to be incorporated into the estimation along with the random sample to improve parameter estimation. GMM estimators are known to be consistent, asymptotically normal, and most efficient in the class of all estimators that do not use any additional information other than that contained in the moment conditions. Importantly, the behavior of the most efficient GMM estimator is known.
Consider an algorithm that releases the mean of the dataset. The mean reveals a population moment condition on the dataset. For a binary $q$ under the log loss, the attacker wants to estimate $p = \Pr[q(x)=1]$. Given only the compromised data, the attacker can estimate $p$ with some accuracy. The additional mean moment condition can improve the accuracy of the estimate. The improvement in the variance of the best estimate of $p$ with and without the moment condition can be determined using GMM. GMM assumes that the sample is produced by a stationary, ergodic, $\alpha$-mixing process. %

\begin{assumption} \label{ass:process}
    $\mathcal{P}$ is a stationary, ergodic, and $\alpha$-mixing process over $N$ elements (see definitions in \pref{app:definitions}). $\mathcal{P}$ has stationary distribution $D = (p_1, \dots, p_N)$, such that $p_i > 0$, $\forall i \in [N]$.
\end{assumption}

Some examples of such processes are i.i.d.\@ sampling, sampling without replacement from a dataset (this models an attacker that compromises a random shard of the data), a process that samples certain portions of the dataset with higher probability (this models a system where certain data servers are more susceptible to attacks), or a process that samples from an aperiodic, irreducible Markov chain.

Let $x_1, x_2, \dots, x_n$ be a random sample drawn from a stochastic process satisfying \pref{ass:process}. Let $\theta \in \mathbb{R}^r$ be an unknown vector of parameters to estimate, with true value $\theta_0$. Let population moment conditions $f(x, \theta)$ be a continuous and differentiable $\mathbb{R}^r \mapsto \mathbb{R}^m$ function of $\theta$ such that
\begin{align*}
    \E [f(x, \theta)] = 0 \text{, only at the true parameter $\theta_0$.}
\end{align*}
The first moment condition specifies the parameter of estimation, and additional moment conditions specify known population parameters. Additional moment conditions have the form $g(x, \lambda)$ for some (possibly vector-valued) function $g$ and known population parameter $\lambda$.
For example, consider the problem of estimating $p = \E[q(x)]$ for some query $q$, given a known population mean $\mu$. We have $\theta = (p)$ and $f(x, \theta) = [q(x) - p,~ x - \mu]^\top$
where $g(x, \mu)=x - \mu$. 

GMM also allows for noisy moment conditions. If instead of knowing $\lambda$ directly, we only know $\tilde{\lambda} = \lambda + \Delta$, where $\Delta$ is independent mean-zero noise, then we have $\theta = (p, \lambda)$ and moment conditions
\begin{align*}
    f(x, \theta) = [q(x) - p,~ g(x, \lambda),~ \tilde{\lambda} - \lambda]^\top.
\end{align*}

Let $f_n (\theta) = \frac{1}{n} \sum_{i=1}^n f(x_i, \theta)$ be the sample moments. Given the sample, the GMM estimator of $\theta$ is %
\begin{align*}
    \hat{\theta} = \arg \min_{\theta \in \Theta \subseteq \mathbb{R}^r} f_n (\theta)^\top W f_n (\theta)
\end{align*}
where $W$ is some positive definite $m \times m$ weight matrix. %
Under standard regularity conditions (listed in \pref{app:reg-conditions}), GMM estimators are regular, consistent, and asymptotically normal, and we can characterize the behavior of the efficient GMM estimator $\gmm$.

\begin{restatable}[Efficient GMM Estimator \cite{hansen1982}]{theorem}{gmmefficiency}  \label{thm:gmm-efficiency}
    Let $G = \E [\frac{\partial f(x, \theta)}{\partial \theta}]$ be the Jacobian and let $\Omega = \sum_{k=-\infty}^{\infty} \mathbb{E}\left[f(x_0,\theta) f(x_k,\theta)^\top\right]$ be the long-run covariance. %
    As $n \rightarrow \infty$, with $W = \Omega^{-1}(\theta_0)$, the efficient GMM estimator $\gmm$ is regular and asymptotically efficient 
    with asymptotic distribution
    {\small \begin{align*}
        \sqrt{n} (\gmm - \theta_0) \sim \mathcal{N}(0, [ G(\theta_0)^\top \Omega^{-1}(\theta_0) G(\theta_0)]^{-1}).
    \end{align*}}
\end{restatable}
The proof of efficiency is standard from \cite{hansen1982gmm}. Importantly, the sample moments must follow an appropriate uniform law of large numbers (\pref{thm:ulln-se}) and central limit theorem (\pref{thm:clt-se}).

We assume the posterior induced by both $C(S)$ and $C(S), A(S)$ satisfies a Bernstein–von Mises (BvM) theorem for the query functional. Under this assumption, the Bayes optimal estimators $\hat{\theta}^*_{q|C(S)}$ and $\hat{\theta}^*_{q|C(S),A(S)}$ are known to be regular and asymptotically efficient. In this paper, we assume that the (functional) BvM holds, and under this condition, we show in the theorem below that one can replace the Bayes optimal estimator in the definition of predictability by the asymptotically efficient GMM estimator instead. Standard sufficient conditions for the BvM theorem have been extensively studied in statistics literature (see \cite[Chapter 10]{vaart1998asymptotic}, \cite{bickel1993efficient} \cite[Chapter 12]{ghosal2017fundamentals}). 

\begin{restatable}{theorem}{bayesandgmm} \label{thm:bayes-and-gmm}
    Let $\ell (\hat{\theta}) = \E_{x \sim \Pi}[\ell(\hat{\theta}, q(x)]$ be a twice continuously differentiable loss function with $\nabla \ell(\theta_0) = 0$ and $\nabla^2 \ell(\theta_0)$ positive definite. Then %
    {\small \begin{align*}
        \mathop{\E}_{\substack{x \sim \Pi}}[\ell(\bayes_{q \mid C(S)}, q(x)) - \ell(\bayes_{q \mid C(S), A(S)}, q(x))] \leq \mathop{\E}_{\substack{x \sim \Pi}}[\ell(\gmm_{q \mid C(S)}, q(x)) - \ell(\gmm_{q \mid C(S), A(S)}, q(x))] + o(\tfrac{1}{n})
    \end{align*}}
\end{restatable}
As discussed above, under our assumption, both estimators, $\gmm$ and $\bayes$, are regular and asymptotically efficient. Thus, both estimators must share the same influence function and therefore, the same behavior up to lower order terms. %
As $n \rightarrow \infty$, the lower order difference vanishes, and we can replace the difference in the loss of $\bayes$ with the difference in the loss of $\gmm$ in the predictability bound. Then we can use the behavior of $\gmm$ to bound predictability.

\subsection{Applying GMM to Predictability}

In the setting of predictability, the attacker receives $C(S)$, a random sample of dataset $S$ from process $\mathcal{P}$. Using GMM, we want to compare the behavior of the estimator given only $C(S)$ versus given $C(S)$ and $A(S)$. %
For GMM estimation, we must first specify our parameter of interest. 
For a single $q$, the attacker wants to estimate $q(x)$ for the average individual in the dataset, $p = \E_{D_{\mathrm{U}}}[q(x)]$, $D_{\mathrm{U}} = \mathrm{Unif}(S)$, given the information they have. This is almost in the right form for a moment condition, but the expectation is under a different distribution $D_{\mathrm{U}}$, instead of $D$. However, notice that a moment condition on $D_{\mathrm{U}}$ implies a moment condition on $D$. In particular, at the unique $\theta_0$,
\begin{align*}
    \E_{D_{\mathrm{U}}} [f(x, \theta_0)] = \E_D [w(x) f(x, \theta_0)] = 0, \text{ where } w(x_i) = \tfrac{d D_{\mathrm{U}}}{d D}(x_i) = \tfrac{1}{p_i N}.
\end{align*}
To incorporate the information from the algorithm, consider $A(S) = \tilde{\lambda} = \lambda + \Delta$, for some $\lambda$ and $\Delta \sim \mathcal{N}(0, \Sigma_\Delta)$. $A(S)$ exactly reveals a noisy moment $\lambda$ on the entire dataset $S$. If we can express $\lambda$ in the form $\E[g(x, \lambda)] = \frac{1}{N} \sum_{i=1}^N g(x_i, \lambda) = 0$ for some $g$, then the information from $A(S)$ can be expressed as a moment condition on $D_{\mathrm{U}}$ and thus, a weighted moment condition on $D$. Then the moment conditions for predictability are 
\begin{equation}
    \begin{aligned}
        f_{C(S)}(x, \theta) &= [w(x)(q(x) - p)] \\
        f_{C(S), A(S)}(x, \theta) &= [w(x)(q(x) - p),~ w(x)g(x, \lambda),~ \tilde{\lambda} - \lambda]^\top
    \end{aligned}
    \label{eq:moment-conditions}
\end{equation}
with $\theta = (p)$ and $\theta = (p, \lambda)$, respectively. We note that $\theta$ must capture all unknown parameters; however, we focus only on the variance of the estimate of $p$.

Recall from \pref{lem:loss-lem}, to bound predictability, we must characterize the behavior of $|\gmm_{q \mid C(S)} - \theta^q_0|$ and $|\gmm_{q \mid C(S)} - \gmm_{q \mid C(S), A(S)}|$. \pref{thm:gmm-efficiency} characterizes the behavior of $|\gmm_{q \mid C(S)} - \theta^q_0|$. For the behavior of $|\gmm_{q \mid C(S)} - \gmm_{q \mid C(S), A(S)}|$, we employ one additional property of GMM estimators.
\begin{restatable}{theorem}{distofgmmdiff} \label{thm:dist-of-gmm-diff}
    When the efficient GMM estimators $\gmm_{q \mid C(S)}$ and $\gmm_{q \mid C(S), A(S)}$ 
    are constructed on the same sample, their difference has asymptotic distribution
    {\small \begin{align*}
        \gmm_{q \mid C(S)} - \gmm_{q \mid C(S), A(S)} \sim \mathcal{N}(0, \V[\gmm_{q \mid C(S)}] - \V[\gmm_{q \mid C(S),A(S)}])
    \end{align*}}
    \vspace{-2\belowdisplayskip}
\end{restatable}

$\V[\gmm_{q \mid C(S)}]$ and $\V[\gmm_{q \mid C(S), A(S)}]$ are found by directly computing $\frac{1}{n} [ G^\top \Omega^{-1}G]^{-1}$, plugging the respective weighted moment conditions $f(x, \theta)$ into $G$ and $\Omega^{-1}$. Then we use the concentration of these estimators to bound the predictability.

\begin{definition} [Canonical Correlation \citep{anderson2003multivariate}] \label{def:can-corr}
    Let $\rho (X, Y) = \frac{\Cov(X, Y)}{\sigma_X \sigma_Y}$ be standard correlation. The canonical correlation between 1-dimensional random variable $X$ and vector random variable $V$ is $\cc (X, V) = \max_{c \in \mathbb{R}^d} \rho (X, c^\top V)$, the max correlation between $X$ and any linear projection of $V$.
\end{definition}

\begin{restatable}[Asymptotic Predictability of a Query]{theorem}{predbound} \label{thm:pred-bound}
    Let $A(S) = \tilde{\lambda} = \lambda + \Delta$, $\Delta = \mathcal{N}(0, \Sigma_\Delta)$.
    Assume process $\mathcal{P}$ satisfies \pref{ass:process}. 
    Assume the moment conditions in \pref{eq:moment-conditions} satisfy the regularity conditions (\pref{app:reg-conditions}). We have 
    $\sqrt{n} (\gmm_{q \mid C(S)} - \theta_0^q) \sim \mathcal{N}(0, \sigma_1^2)$ and $\sqrt{n}(\gmm_{q \mid C(S)} - \gmm_{q \mid C(S), A(S)}) \sim \mathcal{N}(0, \sigma_1^2 - \sigma_2^2)$.
    For a convex, $H$-smooth loss, with probability $1 - \delta$,

    \vspace{-2\belowdisplayskip}
    {\small \begin{align*}
        \mathop{\E}_{\substack{x_i \sim \Pi}}[\ell(\gmm_{q \mid C(S)}, q(x)) - \ell(\gmm_{q \mid C(S), A(S)}, q(x))]] \leq \frac{H \log(4 /\delta)}{n} \cdot \sigma_1 \sqrt{\sigma_1^2 - \sigma_2^2}.
    \end{align*}}
    
    \vspace{-3mm}
    Computing the expectation of the above, asymptotically, $\gamma (\mathcal{P}, q, \ell, A) \leq \frac{H(1 + \log(4))}{n} \cdot \sigma_1 \sqrt{\sigma_1^2 - \sigma_2^2}$.
    
    Let $q' = w(x) (q(x)-p)$ and let $g' = w(x) g(x, \lambda)$. Let
    {\small \begin{align*}
        \Cw &= \V[q']^{-1} \Cov(q', g') \allowbreak \V[g']^{-1} \allowbreak \E\left[\tfrac{\partial g'}{\partial \lambda}\right] \allowbreak \Big (\E\left[\tfrac{\partial g'}{\partial \lambda}\right]^\top \allowbreak \V[g']^{-1} \allowbreak \E\left[\tfrac{\partial g'}{\partial \lambda}\right] \allowbreak + \allowbreak \Sigma_\Delta^{-1} \Big)^{-1} \\
        &\;\;\;\;\;\;\cdot \E\left[\tfrac{\partial g'}{\partial \lambda}\right]^\top \allowbreak \V[g']^{-1} \allowbreak \Cov(q', g').
    \end{align*}}
    Plugging in the form of $\sigma_1$ and $\sigma_2$, the algorithm has predictability

    \vspace{-2\belowdisplayskip}
    {\small \begin{align*}
        \gamma (\mathcal{P}, q, \ell, A) \leq \frac{H(1 + \log (4))}{n} \cdot \V[q'] \cdot \sqrt{\cc (q', g')^2 - \Cw}.
    \end{align*}}
\end{restatable}

All expectations in \pref{thm:pred-bound} are weighted averages on the dataset using weights $w(x_i) = \frac{1}{p_i N}$. Therefore, given a dataset, all terms (expectations, variances, covariances) in \pref{thm:pred-bound} can be computed efficiently.

\begin{remark}
    The resulting predictability looks unwieldy, but we break down the key terms.
    \vspace*{-1mm}
    \begin{itemize}[nosep, leftmargin=1em]
        \item $\cc (q', g')^2$ - The canonical correlation captures how correlated $q$ and $A(S)$ are. Intuitively, it quantifies the most information we can extract from $A(S)$ about $q(x)$. Given a dataset and stationary distribution, $\cc$ can be efficiently computed, $\cc^2(q', g') = \tfrac{1}{\V[q']} \cdot \Cov(q', g') \V[g']^{-1} \mathrm{Cov}(g', q')$.
        \item $\Cw$ - $c_0$ captures how adding noise affects predictability. $c_0$ ranges from $0$ to $\cc (q', g')^2$. Given a dataset and stationary distribution, $c_0$ can also be efficiently computed.
    \end{itemize}
\end{remark}
\vspace*{-1mm}

To better understand predictability, let $\Sigma_\Delta = \sigma^2 I_d$, then consider:

\begin{tabular}{lll}
\toprule
Condition & Predictability & Comment \\ 
\midrule
$\sigma^2 \to 0$ 
& $\gamma = \tfrac{H(1+\log 4)}{n} \cdot \V[q'] \cdot |\cc(q', g')|$ 
& Noiseless case, $c_0 \to 0$
\\
$\sigma^2 \to \infty$ 
& $\gamma = 0$ 
& Infinite noise, $c_0 \to \cc(q', g')^2$
\\

$\cc(q', g') = 0$ 
& $\gamma = 0$ 
& Privacy for free; $q$ and $A$ are unrelated
\\
\bottomrule
\end{tabular}

\vspace*{1mm}
\pref{thm:pred-bound} yields a general recipe for computing asymptotic predictability.
\vspace*{-2mm}
\begin{itemize}[nosep, leftmargin=1em]
    \item For a private algorithm $A(S) = \tilde{\lambda} = \lambda + \Delta$, express $\lambda$ in the form $\E[g(x, \lambda)] = 0$.
    \item Define moment conditions $f_{C(S)}(x, \theta)$ and $f_{C(S), A(S)}(x, \theta)$, accordingly.
    \item Compute the variance of efficient estimators and apply \pref{thm:pred-bound}.
\end{itemize}

\textbf{Random Shard is Compromised \hspace*{2mm}}
Consider the process $\mathcal{P}_{\mathrm{RS}}$ where a sample of size $n = \alpha N$ is drawn uniformly without replacement from the dataset. This is the setting in which the attacker compromises a random shard of data. Under $\mathcal{P}_{\mathrm{RS}}$, the stationary distribution $D = D_{\mathrm{U}}$, and $w(x) = 1$. %

\begin{restatable}{corollary}{predboundsigmaderived} \label{corr:pred-bound-sigma-derived}

    Let $A(S) = \tilde{\lambda} = \lambda + \Delta$, $\Delta = \mathcal{N}(0, \Sigma_\Delta)$.
    As $n, N \rightarrow \infty$, $n/N \rightarrow \alpha$, for a convex, $H$-smooth $\ell$, $\gamma (\mathcal{P}_{\mathrm{RS}}, q, \ell, A) \leq \frac{H(1 + \log(4))}{(1-\alpha) \alpha N} \cdot \V[q(x)] \cdot \sqrt{\cc (q(x), g(x, \lambda))^2 - \Cwone}$.

\end{restatable}

For typical choices of losses (log loss and squared loss), we can plug in $H = \frac{1}{\tau^2}$, $H=2$, appropriately. Given a differentially private algorithm and query of interest, predictability bounds can be determined simultaneously for all random fractions $\alpha$ of compromised data that hold alongside worst-case guarantees from differential privacy, giving users a more detailed understanding of privacy. %

\section{Analyzing the Predictability of Private Algorithms Using GMM} \label{sec:pred-of-algs}

For several important algorithms, we demonstrate how to express the algorithm as $\E[g(x, \lambda)] = 0$.

\subsection{Private Dataset Statistics} 
\label{sec:stats}
Consider an algorithm that releases a noisy aggregate statistic from the dataset, such as the proportion of smokers. %
Here, specifying the algorithm as a moment condition is straightforward. A statistic $\lambda$ takes the form, $\lambda = \frac{1}{N} \sum_{i=1}^N \Gamma (x_i)$. For example, if $\lambda$ is the mean, $\Gamma(x) = x$, or if $\lambda$ is a count statistic, $\Gamma(x) = \ind{x \text{ is a smoker}}$. Then $g(x, \lambda) = \Gamma(x) - \lambda$ ensures $\E[g(x, \lambda)]=0$ only at the true $\lambda$.

\begin{restatable}{theorem}{noisystat} \label{thm:noisy-stat}
    Let $A(S) = \tilde{\lambda} = \frac{1}{N} \sum_{i=1}^N \Gamma (x_i) + \Delta$, $\Delta \sim \mathcal{N}(0, \Sigma_\Delta)$. Let $g(x, \lambda) = \Gamma(x) - \lambda$. %
    For a convex, $H$-smooth $\ell$, $\gamma (\mathcal{P}_{\mathrm{RS}}, q, \ell, A) = \frac{H(1 + \log(4))}{(1-\alpha) \alpha N} \cdot \V[q(x)] \cdot \rho(q(x), \Gamma (x)) \cdot \sqrt{\frac{\V[\Gamma(x)]}{\V[\Gamma(x)] + \Sigma_\Delta}}$.

\end{restatable}
We plug the appropriate $g(x, \lambda)$ into \pref{corr:pred-bound-sigma-derived} to obtain the expression for predictability. When $g(x, \lambda) = \Gamma(x) - \lambda$, we obtain a simplified expression, $\Cwone = \rho (q(x), \Gamma(x))^2 \cdot \left( 1- \frac{\V[\Gamma(x)]}{\V[\Gamma(x)] + \Sigma_\Delta} \right)$. Here, we can clearly observe how $\Sigma_\Delta$ affects $c_0$ and predictability.

\subsection{Private Empirical Risk Minimization} \label{sec:private-erm}
Next, consider an algorithm that releases a noisy empirical risk minimizer (ERM). Many common machine learning problems can be reduced to ERM; thus, analyzing and understanding the privacy and predictability guarantees under the release of ERM is crucial.

Dataset entries consist of $z = (x, y)$ pairs. Consider some loss function $\ell (z, w)$. The ERM of $\ell$ is $\wE = \arg \min_w \frac{1}{N} \sum_{i=1}^N \ell (z_i, w)$.
This condition cannot be directly written in the form $\E[f(z, \theta)]=0$. However, at $\wE$, we have $\frac{1}{N} \sum_{i=1}^N \grad \ell (z_i, \wE) = \E[\grad \ell (z, \wE)] = 0.$
Thus, let $\lambda = \wE$ and $g(x, \wE) = \grad \ell (z, \wE).$ We restrict our attention to strictly convex $\ell$ to ensure $\E[g(x, \wE)]=0$ only holds at the $\wE$ released by the algorithm. %

\begin{restatable}{theorem}{noisyerm} \label{thm:noisy-erm}
    Let $\wE = \arg \min_w \frac{1}{N} \sum_{i=1}^N \ell (z_i, w)$, $\ell$ strictly convex and $\E[\grad^2 \ell (z, \wE)]$ positive definite. Let $A(S) = \tilde{w} = \wE + \Delta$, $\Delta \sim \mathcal{N}(0, \Sigma_\Delta)$. Let $g(z, \wE) = \grad \ell (z, \wE)$. 
    For a convex, $H$-smooth $\ell$, $\gamma (\mathcal{P}_{\mathrm{RS}}, q, \ell, A) \leq \frac{H(1 + \log(4)) }{(1-\alpha) \alpha N} \cdot \V[q(z)] \cdot \sqrt{\cc (q(z), g (z, \wE))^2 - \Cw}$.

\end{restatable}

To achieve small canonical correlation, $q(z)$ must be fairly unrelated to all coordinates of the loss gradient. \textit{We note that, surprisingly, it is not enough for the loss itself and the query to be uncorrelated.} 

In \pref{app:bgd}, we discuss how to express batch gradient descent (BGD) as a moment condition, thus allowing us to handle approximate or non-unique ERM computed by BGD. A more complete characterization of predictability for the ERM of a nonconvex loss is deferred to future work.

\subsection{Predictability-Calibrated Noise Scheme for Private ERM}
One could add isotropic noise to the ERM, $\Sigma_\Delta = \sigma^2 I_d$, but this may not be the best choice. Instead, one could adapt the noise to the specific ERM and the process generating the compromised data.

\begin{restatable}{lemma}{targetednoise} \label{lem:targeted-noise}
    Let $\grad_z = \grad \ell (z, \wE)$ and $H_z = \grad^2 \ell (z, \wE)$. Let $A(S) = \tilde{w} = \wE + \Delta$, $$\Delta \sim \mathcal{N}(0, \sigma^2 \E[w(x) H_z]^{-1} \E[w(x)^2 \grad_z \grad_z^\top] \E[w(x) H_z]^{-1}), $$ where the expectation is taken with respect to the stationary distribution of process $\mathcal{P}$. For a convex, $H$-smooth $\ell$, $\gamma (\mathcal{P}, q, \ell, A) \leq \frac{H(1 + \log(4)) }{(1-\alpha) \alpha N} \cdot \V[q'] \cdot \cc (q', g') \cdot \sqrt{\frac{1}{\sigma^2 + 1}}$.
\end{restatable}

The calibrated noise scheme adapts to the specific ERM and stochastic process and results in a clear relationship between a desired level of predictability and the amount of noise required. We note that $w(x)$ must be known. The complexity to compute the covariance of the noise is $O(Nd^2 + d^3)$.

\textbf{Linear Regression \hspace*{2mm}}
We instantiate the calibrated noise scheme for linear regression under the random sharding process to understand the structure of the noise and to highlight the improved tradeoff between predictability and accuracy. For linear regression under $\mathcal{P}_{\mathrm{RS}}$, we have $\ell (z, w) = \frac{1}{2}(w^\top x - y)^2$ and $w(x)=1$. %
Thus, we add noise $\Delta \sim \mathcal{N}(0, \sigma^2 \E[xx^\top]^{-1} \E[(\wE^\top x - y)^2 xx^\top] \E[xx^\top]^{-1})$.
Analyzing the structure, we see that more noise is added in directions where $\wE$ has a large loss or $x$'s are sparse. 
When $\wE$ has a small empirical loss, $\tilde{w}$ is less noisy for the same predictability. %

\begin{restatable}{theorem}{accuracychangetargeted} \label{thm:accuracy-change-targeted}
    Let $\tilde{w} = \wE + \Delta$, $\Delta \sim \mathcal{N}(0, \sigma^2 \cdot \E[xx^\top]^{-1} \E[(\wE^\top x - y)^2 xx^\top] \E[xx^\top]^{-1})$. Then
    \vspace{-0.5\belowdisplayskip}
    {\small \begin{align*}
        \E &\left[ \frac{1}{N} \sum_{i=1}^N (\tilde{w}^\top x_i - \wE^\top x_i)^2 \right] \leq \min \bigg \{ \frac{\sigma^2(\max_i \|x_i\|^2)}{\lambda_{\min}(\E[xx^\top])} \cdot \hat{\mathcal{L}}_S (\wE), \sigma^2d \cdot \hat{\mathcal{L}}^{\max}_S (\wE) \bigg \}.
    \end{align*}}

    \vspace{-\belowdisplayskip}
    where $\hat{\mathcal{L}}_S (\wE) = \frac{1}{N} \sum_{i=1}^N (\wE^\top x_i - y_i)^2$ and $\hat{\mathcal{L}}^{\max}_S (\wE) = \max_{i} (\wE^\top x_i - y_i)^2$.
\end{restatable}

We can upper bound the change in accuracy of $\tilde{w}$ in terms of the empirical loss. %
\begin{restatable}{lemma}{accuracychangeuniform} \label{lem:accuracy-change-uniform}
    When $\Delta \sim \mathcal{N}(0, \sigma^2 I_d)$, $\E \left[ \frac{1}{N} \sum_{i=1}^N (\tilde{w}^\top x_i - \wE^\top x_i)^2 \right] = \sigma^2 \left(\frac{1}{N} \sum_{i=1}^N \|x_i\|^2\right)$.
\end{restatable}
When $\hat{\mathcal{L}}_S (\wE)$ or $\hat{\mathcal{L}}^{\max}_S (\wE)$ is small (which we would expect), the calibrated noise scheme leads to a smaller change in accuracy compared to an isotropic noise scheme.

Using post-processing of DP, we can compose our calibrated noise scheme for predictability with any differentially private algorithm to ensure stronger privacy protections for specific queries on unknown populations against specific types of attackers, in addition to typical worst-case DP guarantees. In particular, predictability-calibrated perturbation can be combined with DP mechanisms,
provided that any data-dependent quantities (i.e. covariance of $\Delta$) are computed in a differentially private manner.

\section{Predictability of a Family of Queries or Unknown Process} \label{sec:query-families}
So far, we have only analyzed the predictability of an algorithm with respect to a single query $q$ when the algorithm knows the process $\mathcal{P}$. However, in practice, we often want privacy guarantees against an entire family of sensitive queries. Additionally, the exact process $\mathcal{P}$ generating the attacker's knowledge may be unknown to the algorithm designer. We extend the predictability analysis to a family of queries and an unknown process from a family of processes $\Pfam$. %

\textbf{Finite Family of Queries and Processes \hspace*{2mm}}
Notice that for a particular process $\mathcal{P} \in \Pfam$, the predictability depends on the weight function $w \in \mathcal{W}_{\Pfam}$ derived from the stationary distribution of $\mathcal{P}$. From \pref{thm:pred-bound}, we know that for a fixed weight function $w$ and query $q$, with high probability
\begin{align*}
    \mathop{\E}_{\substack{x_i \sim \Pi}}[\ell(\gmm_{q \mid C(S)}, q(x)) - \ell(\gmm_{q \mid C(S), A(S)}, q(x))]]
\end{align*}
is small. We use a union bound to show that for all $w \in \mathcal{W}_{\Pfam}$, $q \in \mathcal{Q}$, the above is small with high probability. Then, taking the expectation, the worst-case predictability over $\Pfam$ is bounded.
\begin{restatable}{theorem}{predfinite} \label{thm:pred-finite}
    Let $A(S) = \tilde{\lambda} = \lambda + \Delta$, $\Delta = \mathcal{N}(0, \Sigma_\Delta)$. For a family $\Pfam$ satisfying \pref{ass:process} and a convex, $H$-smooth $\ell$, the worst-case predictability is $\sup_{\mathcal{P} \in \Pfam} \gamma (\mathcal{P}, \mathcal{Q}, \ell, A)$ %

    \vspace{-2\belowdisplayskip}
    {\small \begin{align*}
         \leq \frac{H(1 + \log(4 |\mathcal{Q}| |\mathcal{W}_{\Pfam}|))}{n} \sup_{w \in \mathcal{W}_{\Pfam}} \sup_{q \in \mathcal{Q}} \V[w(x) q(x)] \sqrt{\cc(w(x) q(x), w(x) g(x, \lambda))^2 - \Cw}.
    \end{align*}    }
    \vspace{-\belowdisplayskip}
\end{restatable}

\textbf{Linear Query Family \hspace*{2mm}}
Consider the family $\mathcal{Q}_{\mathrm{lin}} = \{q(x) = u^\top x \mid{} u \in \mathcal{U}_{\mathrm{lin}} \subseteq \mathbb{R}^d\}$. For infinitely large families, we must rely on something stronger than the concentration of the estimators of single queries. Consider $\gmm_{x \mid C(S)}$ and $\gmm_{x \mid C(S), A(S)}$, estimators for $\E[x \mid C(S)]$ and $\E[x \mid C(S), A(S)]$. Using GMM, we know that $\sqrt{n} (\gmm_{x \mid C(S)} - \mu) \sim \mathcal{N}(0, \Sigma_1)$ and $\sqrt{n} (\gmm_{x \mid C(S), A(S)} - \mu) \sim \mathcal{N}(0, \Sigma_2)$. If estimates of $x$ concentrate well, then estimates of linear functions of $x$ also concentrate well. %

\begin{restatable}{theorem}{predeuclidball} \label{thm:pred-euclid-ball}
    Let query family $\mathcal{Q}_{\mathrm{lin}} = \{q(x) = u^\top x \mid{} u \in \mathcal{U}_{\mathrm{lin}} \subseteq \mathbb{R}^d\}$. 
    If $\mathcal{U}_{\mathrm{lin}}$ is contained within the Euclidean ball of radius $B$, for a convex, $H$-smooth $\ell$, the $\gamma (\mathcal{P}_{\mathrm{RS}}, \mathcal{Q}_{\mathrm{lin}}, \ell, A)$-predictability is $ \leq \frac{4HB (2 + \log (2)) }{(1-\alpha) \alpha N} \cdot \mathrm{tr}(\V[x]) \cdot \sqrt{\cc^1 (x, g(x, \lambda))^2 - \frac{\lambda_{\min}(\Cx)}{\mathrm{tr}(\V[x])}}$.
\end{restatable}

\begin{definition}[First Canonical Correlation from CCA]
    For two vector random variables, $X$ and $Y$, $\cc^1 (X, Y)^2 = \lambda_{\max} (\V[X]^{-1} \Cov (X,Y) \V[Y]^{-1} \Cov (Y,X))$.
\end{definition}
We recover a dependence on a correlation-type quantity, now between $x$ and $g (x, \lambda)$. Canonical correlation analysis (CCA) is a foundational technique in dimension reduction and multiview learning \citep{kakade2007multiview, chaudhuri2009multiview, sindhwani2005coregularization, sridharan2008multiview}, and we find its connection to predictability compelling. Importantly, CCA, and thus $\cc^1 (x, g(x, \lambda))$, can be computed efficiently.
For general $\mathcal{U}_{\mathrm{lin}}$, see the bound on predictability for general linear query families (\pref{thm:pred-linear}), where the bound depends on the Gaussian complexity (\pref{def:gaussian-complexity}) of $\mathcal{U}_{\mathrm{lin}}$.

\textbf{Lipschitz Query Family \hspace*{2mm}}
Similarly, if estimates of $x$ concentrate well, then estimates of Lipschitz functions of $x$ also concentrate well.

\begin{restatable}{theorem}{predlipschitz} \label{thm:pred-lipschitz}
    Let query family $\mathcal{Q}_{\mathrm{L}} = \{ f(x) \mid \text{$f$ is $L$-Lipschitz}\}$.
    For a convex, $H$-smooth loss $\ell$, $\gamma (\mathcal{P}_{\mathrm{RS}}, \mathcal{Q}_{\mathrm{L}}, \ell, A) \leq \frac{4HL^2 (2 + \log (2)) }{(1-\alpha) \alpha N} \cdot \mathrm{tr}(\V[x]) \cdot \sqrt{\cc^1 (x, g(x, \lambda))^2 - \frac{\lambda_{\min}(\Cx)}{\mathrm{tr}(\V[x])}}$.
\end{restatable}

\section{Discussion and Conclusion}
We present predictability as a new fine-grained measure for privacy. We present a general framework for studying asymptotic predictability using GMM. %
We discuss a few extensions of our analysis. Under regularity assumptions (sub-Gaussian moment conditions, smoothness, etc.), one could convert the central limit theorem of the moments to appropriate tail bounds and follow this through to obtain a finite sample version of the asymptotic bound on predictability. We can also generalize predictability beyond a privacy measure for the average unknown individual by selecting the unknown individual $x$ from another distribution, other than $\Pi$, and adjusting the weights $w(x)$ accordingly. Open questions include developing a predictability-calibrated noise scheme for ERM when the process is unknown. Additionally, developing DP-style adaptive composition for predictability would be valuable.

\begin{ack}
LL acknowledges support from the Cornell CIS LinkedIn Fellowship.
\end{ack}

\bibliographystyle{abbrv}
\bibliography{refs}

\newpage
\appendix

\section{Appendix}
\addcontentsline{toc}{section}{Appendix}

\etocsetnexttocdepth{3} %

\localtableofcontents

\newpage

\subsection{Batch Gradient Descent Discussion} \label{app:bgd}

When the exact, unique ERM is released, we apply \pref{thm:noisy-erm} to compute the predictability. However, in many machine learning settings, the algorithm releases an approximate ERM, typically found through a version of gradient descent. We can express the result of batch gradient descent (BGD) as a moment condition. Recall the BGD update, 
\begin{align*}
    w_{t+1} = w_t - \eta \cdot \frac{1}{N} \sum_{i=1}^N \grad \ell (z_i, w_t).
\end{align*}
We observe that each step of BGD can be expressed as a moment condition, %
\begin{align*}
    g(z, w_{t-1, t}) = w_{t-1} - w_{t} - \eta \grad \ell (z, w_{t-1}).
\end{align*}
Given a sequence of $w_t$'s from BGD on a dataset, we define the algorithm's moment condition $g(z, w_{1:T})$ as a vector containing all $T$ steps of BGD. 
\begin{align*}
    g(z, w_{1:T}) &= \begin{bmatrix}
        w_{0} - w_{1} - \eta \grad \ell (z, w_{0}) \\
        w_{1} - w_{2} - \eta \grad \ell (z, w_{1}) \\
        \vdots \\
        w_{T-1} - w_{T} - \eta \grad \ell (z, w_{T-1})
    \end{bmatrix}
\end{align*}
Thus, using \pref{thm:pred-bound}, we can compute the predictability for any approximate ERM obtained from BGD. Note that for a fixed starting point $w_0$, the sequence of $w_t$'s that satisfy $\E[g(z, w_{1:T})]=0$ is unique for any $\ell$, thus eliminating the need to assume strictly convex $\ell$ for proper estimation. 

\subsection{Predictability Beyond Estimating the Average Value of $q(x)$} \label{app:other-dists}

Consider the following extension of predictability. We say algorithm $A$ has $\gamma_{D'}(\mathcal{P}, \mathcal{Q}, \ell, A)$-predictability $\gamma$ if
    \begin{align*}
        \gamma_{D'}(\mathcal{P}, \mathcal{Q}, \ell, A) = \sup_{S} \mathop{\E}_{C(S) \sim \mathcal{P}, A}[ \sup_{\substack{q \in \mathcal{Q}}} \mathop{\E}_{\substack{x \sim D'}} [\ell(\bayes_{q \mid C(S)}, q(x)) - \ell(\bayes_{q \mid C(S), A(S)}, q(x))]] \leq \gamma.
    \end{align*}
for some known distribution $D' = (p_1', \dots, p_N')$ on dataset $S$. A simple example is when the attacker cares about the average of $q(x)$ on some particular subgroup $G \subseteq S$, and we have $D' = \mathrm{Unif}(G)$.

In \pref{sec:gmm}, we demonstrate how to use the generalized method of moments to compute asymptotic $\gamma (\mathcal{P}, \mathcal{Q}, \ell, A)$-predictability, where the $x$ is drawn uniformly from the set of unknown individuals, $x \sim \Pi$. We can compute asymptotic $\gamma_{D'}(\mathcal{P}, \mathcal{Q}, \ell, A)$-predictability with one small change to the moment conditions.

Consider when the attacker wants to estimate $p = \E_{D'}[q(x)]$, $D' = (p_1', \dots, p_N')$, given the information they have. Then, at the unique $\theta_0$,
\begin{align*}
    \E_{D'} [q(x)-p] = \E_D [w'(x) (q(x)-p)] = 0, \text{ where } w'(x_i) = \tfrac{d D_{\mathrm{U}}}{d D}(x_i) = \tfrac{p'_i}{p_i}.
\end{align*}
Note that we keep 
\begin{align*}
    \E_{\mathrm{D_U}} [g(x, \lambda)] = \E_D [w(x) g(x, \lambda)] = 0, \text{ where } w(x_i) = \tfrac{d D'}{d D}(x_i) = \tfrac{1}{p_i N},
\end{align*}
since we still expect $A(S)$ to have the form $\frac{1}{N} \sum_{i=1}^N g(x_i, \lambda) = 0$. Then we use moment conditions
    \begin{align*}
        f_{C(S)}(x, \theta) &= [w'(x)(q(x) - p)] \\
        f_{C(S), A(S)}(x, \theta) &= [w'(x)(q(x) - p),~ w(x)g(x, \lambda),~ \tilde{\lambda} - \lambda]^\top
    \end{align*}
to compute $\V[\gmm_{q \mid C(S)}]$, $\V[\gmm_{q \mid C(S), A(S)}]$, and the resulting predictability.

\begin{restatable}{theorem}{predbound} \label{thm:pred-bound-d-prime}
    Let $A(S) = \tilde{\lambda} = \lambda + \Delta$, $\Delta = \mathcal{N}(0, \Sigma_\Delta)$.
    Assume process $\mathcal{P}$ satisfies \pref{ass:process}. 
    Assume the moment conditions in \pref{eq:moment-conditions} satisfy the regularity conditions (\pref{app:reg-conditions}). %

    Let $q'' = w'(x) (q(x)-p)$ and let $g' = w(x) g(x, \lambda)$. Then $A$ has predictability

    \vspace{-2\belowdisplayskip}
    {\small \begin{align*}
        \gamma (\mathcal{P}, q, \ell, A) \leq \frac{H(1 + \log (4))}{n} \cdot \V[q''] \cdot \sqrt{\cc (q'', g')^2 - \Cw},
    \end{align*}}

    \vspace{-0.75\belowdisplayskip}
    where
    {\small \begin{align*}
        \Cw &= \V[q'']^{-1} \Cov(q'', g') \allowbreak \V[g']^{-1} \allowbreak \E\left[\tfrac{\partial g'}{\partial \lambda}\right] \allowbreak \Big (\E\left[\tfrac{\partial g'}{\partial \lambda}\right]^\top \allowbreak \V[g']^{-1} \allowbreak \E\left[\tfrac{\partial g'}{\partial \lambda}\right] \allowbreak + \allowbreak \Sigma_\Delta^{-1} \Big)^{-1} \\
        &\;\;\;\;\;\;\cdot \E\left[\tfrac{\partial g'}{\partial \lambda}\right]^\top \allowbreak \V[g']^{-1} \allowbreak \Cov(q'', g').
    \end{align*}}
    
\end{restatable}

\newpage

\subsection{Notation}
\begin{itemize}[leftmargin=2em]
    \item $[N] = \{1, \dots, N\}$
    \item $\mathrm{Unif}(S)$ is the uniform distribution over the element in set $S$
    \item $X \sim \mathcal{N}(\mu, \Sigma)$ is normally distributed with mean $\mu$ and variance $\Sigma$
    \item $\E[X]$ is the expected value of $X$ 
    \item $\V[X] = \E[(X-\E[X])^2]$ for a one-dimensional random variable $X$
    \item $\V[V] = \E[(X-\E[X])(X-\E[X])^\top]$ for a vector random variable $V$
    \item $\Cov (U, V) = \E[(U-\E[U])(V-\E[V])^\top]$
    \item $\rho (X, Y) = \frac{\Cov(X, Y)}{\sqrt{\V[X] \cdot \V[Y]}}$ is the correlation between $X$ and $Y$
    \item $D_{\textsc{KL}}(P \| Q)$ is the KL-divergence between distributions $P$ and $Q$
    \item $\I (X ; Y \mid Z)$ is the conditional mutual information between $X$ and $Y$ given $Z$
    \item $X_n \rightarrow_p Y_0$, $X_n$ converges to $Y_0$ in probability 
    \item $X_n \rightarrow_d Y_0$, $X_n$ converges to $Y_0$ in distribution
    \item $A[i, j]$ denotes the $(i, j)$ entry of matrix $A$ 
\end{itemize}

\newpage

\subsection{Definitions} \label{app:definitions}

\begin{definition}[$(\varepsilon, \delta)$-Differential Privacy] \label{def:dp}
    A randomized algorithm $A$ is $(\varepsilon, \delta)$-differentially private, if for all neighboring datasets $S, S'$, and for all $F \subseteq \mathcal{F}$,
    \begin{align*}
        \Pr[A(S) \in F] \leq e^\varepsilon \Pr[A(S') \in F] + \delta.
    \end{align*}
\end{definition}

\begin{definition}[Stationarity]
    The stochastic process $(X_t)_{t \in \mathbb{Z}}$ is stationary if for all $k \in \mathbb{N}$, all $t_1, \dots, t_k \in \mathbb{Z}$, and all $h \in \mathbb{Z}$, $(X_{t_1}, \dots, X_{t_k}) \overset{d}{=} (X_{t_1 + h}, \dots, X_{t_k + h})$.
\end{definition}

\begin{definition}[Ergodicity]
    The stochastic process $(X_t)_{t \in \mathbb{Z}}$ is ergodic if for any integrable function $f$, the time average $\frac{1}{n} \sum_{t=1}^n f(X_t)$ converges almost surely to $\E[f(X_0)]$.
\end{definition}

\begin{definition} [$\alpha$-mixing]
    For $k \geq 1$, the $\alpha$-mixing coefficients are defined as 
    \begin{align*}
        \alpha(k) = \sup_{t \in \mathbb{Z}} \sup_{\substack{A \in \sigma(X_s : s \le t), B \in \sigma(X_s : s \ge t+k)}} \left| \Pr[A \cap B] - \Pr[A]\Pr[B] \right|.
    \end{align*}
    The stochastic process $(X_t)_{t \in \mathbb{Z}}$ is $\alpha$-mixing if $\alpha(k) \to 0$, as $k \to \infty$.
\end{definition}

\begin{definition}[Convex Function]
    A function $f$ is convex if $f(y) \geq f(x) + \grad f(x)^\top \norm{y-x}$.
\end{definition}

\begin{definition}[Smooth Function]
    A function $f$ is $H$-smooth if $\norm{\grad f(x) - \grad f(y)} \leq H \cdot \norm{x-y}$.
\end{definition}

\begin{definition}[Lipschitz Function]
    A function $f$ is $L$-Lipschitz if $|f(x) - f(y)| \leq L \cdot \norm{x-y}$.
\end{definition}

\newpage

\subsection{Proofs from \pref{sec:def-pred}}

\subsubsection{\pref{thm:predtokl}}

\predtokl*

\begin{proof}
    From predictability, we know that $A(S)$ satisfies
    \begin{align*}
        \sup_{\substack{S}} \E_{C(S) \sim \mathcal{P}, A} [ \sup_{ q \in \mathcal{Q}} \E_{x \sim \Pi} [\ell(\bayes_{q \mid C(S)}, q(x)) - \ell(\bayes_{q \mid C(S), A(S)}, q(x))]] \leq \gamma.
    \end{align*}

    Under the log loss, the optimal $\bayes_{q \mid C(S)}$ and $\bayes_{q \mid C(S), A(S)}$ predict $\Pr[q(x)=i \mid{} C(S)]$ and $\Pr[q(x)=i \mid{} C(S), A(S)]$, respectively, for each $i \in [K]$. Plugging these in and expanding the loss,
    \begin{align*}
        &\E_{x \sim \Pi, A} [\ell(\Pr[q(x)=i \mid{} C(S)], q(x)) - \ell(\Pr[q(x)=i \mid{} C(S), A(S)], q(x))] \\
        &\;\;\;= \E_{x \sim \Pi, A} \left[ - \sum_{i=1}^K \ind{y=i} \log(\Pr[q(x)=i \mid{} C(S)]) + \sum_{i=1}^K \ind{y=i} \log(\Pr[q(x)=i \mid{} C(S), A(S)]) \right] \\
        &\;\;\;= \E_{x \sim \Pi, A} \left[ - \sum_{i=1}^K \ind{y=i} \log(\frac{\Pr[q(x)=i \mid{} C(S), A(S)]}{\Pr[q(x)=i \mid{} C(S)]}) \right] \\
        &\;\;\;= \E_{x \sim \Pi, A} [D_{\textsc{kl}}(P_{q(x) \mid{} C(S), A(S)} \| P_{q(x) \mid{} C(S)})].
    \end{align*}

    Thus,
    \begin{align*}
        \sup_{\substack{S}} \E_{C(S) \sim \mathcal{P}} [\sup_{ q \in \mathcal{Q}} \E_{x \sim \Pi, A} [ D_{\textsc{kl}}(P_{q(x) \mid{} C(S), A(S)} \| P_{q(x) \mid{} C(S)})]] \leq \gamma
    \end{align*}

    From here, observe that
    \begin{align*}
        \sup_{\substack{S}} \E_{C(S) \sim \mathcal{P}} [\sup_{ q \in \mathcal{Q}} \E_{x \sim \Pi, A} &[ D_{\textsc{kl}}(P_{q(x) \mid{} C(S), A(S)} \| P_{q(x) \mid{} C(S)})]] \\
        &\geq \sup_{\substack{S}} \sup_{ q \in \mathcal{Q}} \E_{C(S) \sim \mathcal{P}, A} [ \E_{x \sim \Pi} [ D_{\textsc{kl}}(P_{q(x) \mid{} C(S), A(S)} \| P_{q(x) \mid{} C(S)})]] \\
        &\geq \sup_{\substack{S}} \sup_{ q \in \mathcal{Q}} \I(q(x) ; A(S) \mid C(S))
    \end{align*}
\end{proof}

\subsubsection{\pref{thm:dp-pred-contrast}, Predictability and Differential Privacy are Incomparable}

\dppredcontrast*
\begin{proof}
Let $S$ be a binary dataset, $x_i \in \{0, 1\}$. Let $p = \frac{1}{N} \sum_{i=1}^N x_i$ be the proportion of $1$'s. Consider a single query of interest $q(x) = x$.

For the first part of the statement, let $A(S)$ be an $\varepsilon$-differentially private release of $p$, $A(S)=p + \mathrm{Lap}(0, \frac{\varepsilon}{N})$, $\varepsilon = \frac{1}{\sqrt{N}}$. Consider a process $\mathcal{P}$ that only selects entries such that $x_i = 1$ to form $C(S)$ of size $\frac{N}{2}$. Let $p_{S \setminus C(S)}$ be the proportion of $1$'s. Given only $C(S)$, $\bayes_{q \mid C(S)}$ cannot estimate $p_{S \setminus C(S)}$ accurately. The best that $\bayes_{q \mid C(S)}$ can do is estimate $0.5$ because $p_{S \setminus C(S)}$ can range anywhere from $0$ to $1$ and $C(S)$ is always a set of $1$'s. 

However, given $C(S)$ and $A(S)$, $\bayes_{q \mid C(S), A(S)}$ can estimate $p_{S \setminus C(S)}$ very accurately. In particular,
\begin{align*}
    \frac{1}{2} p_{C(S)} + \frac{1}{2} p_{S \setminus C(S)} = p \\
    p_{S \setminus C(S)} = 2p - p_{C(S)} 
\end{align*}

Consider when $\bayes_{q \mid C(S), A(S)}$ constructs the estimator
\begin{align*}
    \hat{p}_{S \setminus C(S)} &= 2\cdot A(S) - p_{C(S)} \\
    &= 2\cdot (p + \Delta)- p_{C(S)} \tag{$\Delta \sim \mathrm{Lap}(0, \frac{\varepsilon}{N})$} \\
    &= p_{S \setminus C(S)} + 2 \Delta \\
    &= p_{S \setminus C(S)} + \Delta' \tag{$\Delta \sim \mathrm{Lap}(0, \frac{2\varepsilon}{N})$}
\end{align*}

Let $\ell(\hat{y}, y) = (\hat{y}-y)^2$ be the squared loss. Consider the predictability of $q$ and plug in the estimators $\bayes_{q \mid C(S)}$ and $\bayes_{q \mid C(S), A(S)}$,
\begin{align*}
    \sup_{S} \mathop{\E}_{C(S) \sim \mathcal{P}}[ \mathop{\E}_{\substack{x \sim \Pi, A}} &[\ell(\bayes_{q \mid C(S)}, q(x)) - \ell(\bayes_{q \mid C(S), A(S)}, q(x))]] \\
    &= \sup_{S} \mathop{\E}_{C(S) \sim \mathcal{P}}[ \mathop{\E}_{\substack{x \sim \Pi, A}} [(\bayes_{q \mid C(S)} - q(x)) - (\bayes_{q \mid C(S), A(S)} - q(x))]] \\
    &= \sup_{S} \mathop{\E}_{C(S) \sim \mathcal{P}, A}[ (\bayes_{q \mid C(S)} - \mathop{\E}_{\substack{x \sim \Pi}}[q(x)])^2 + \mathop{\E}_{\substack{x \sim \Pi}} [(\mathop{\E}_{\substack{x \sim \Pi}}[q(x)] - q(x))^2] - (\bayes_{q \mid C(S), A(S)} - \mathop{\E}_{\substack{x \sim \Pi}}[q(x)])^2 \\
    &\;\;\;\;\;\;- \mathop{\E}_{\substack{x \sim \Pi}} [(\mathop{\E}_{\substack{x \sim \Pi}}[q(x)] - q(x))^2] ] \tag{bias-variance decomposition of squared loss} \\
    &= \sup_{S} \mathop{\E}_{C(S) \sim \mathcal{P}, A}[ (\bayes_{q \mid C(S)} - \mathop{\E}_{\substack{x \sim \Pi}}[q(x)])^2 - (\bayes_{q \mid C(S), A(S)} - \mathop{\E}_{\substack{x \sim \Pi}}[q(x)])^2 \\
    &= \sup_{S} \mathop{\E}_{C(S) \sim \mathcal{P}, A}[ (\bayes_{q \mid C(S)} - p_{S \setminus C(S)})^2 - (\bayes_{q \mid C(S), A(S)} - p_{S \setminus C(S)})^2] \\
    &= \sup_{S} \mathop{\E}_{C(S) \sim \mathcal{P}, A}[ (0.5 - p_{S \setminus C(S)})^2 - (p_{S \setminus C(S)} + \Delta' - p_{S \setminus C(S)})^2] \\
    &= \sup_{S} \mathop{\E}_{C(S) \sim \mathcal{P}}[ (0.5 - p_{S \setminus C(S)})^2 - \E_A[(\Delta')^2]] \\
    &= \sup_{S} \mathop{\E}_{C(S) \sim \mathcal{P}}\left[ (0.5 - p_{S \setminus C(S)})^2 - \frac{8\varepsilon^2}{N^2}\right] \\
    &= \sup_{S} \mathop{\E}_{C(S) \sim \mathcal{P}}\left[ (0.5 - p_{S \setminus C(S)})^2 - \frac{8}{N}\right] \tag{$\varepsilon = \frac{1}{\sqrt{N}}$} \\
    &= \sup_{S} \left((0.5 - p_{S \setminus C(S)})^2 - \frac{8}{N} \right) \tag{$\mathcal{P}$ always selects the same $C(S)$}
\end{align*}

For the worst case dataset $S$, $p_{S \setminus C(S)} \rightarrow 0$, and predictability $\rightarrow 0.5^2 - \frac{8}{N}$. Thus, $\gamma \geq \Omega (1)$.

For the second part of the statement, let $A(S)$ release the exact proportion, $A(S) = p$. Consider a process $\mathcal{P}$ which selects $\frac{N}{2}$ entries i.i.d. from $S$ to form $C(S)$. Clearly, $A(S)$ is not differentially private.

We demonstrate that $A$ maintains small predictability. Given $A(S)$ and $C(S)$, $\bayes_{q \mid C(S), A(S)}$ can determine $p_{S \setminus C(S)}$ exactly. However, given only $C(S)$, $\bayes_{q \mid C(S)}$ can still estimate $p_{S \setminus C(S)}$ fairly accurately using $p_{C(S)}$ because $C(S)$ is an i.i.d. draw from the overall dataset.

Let $\ell(\hat{y}, y) = (\hat{y}-y)^2$ be the squared loss. Consider the predictability of $q$ and plug in the estimators $\bayes_{q \mid C(S)}$ and $\bayes_{q \mid C(S), A(S)}$,
\begin{align*}
    \sup_{S} \mathop{\E}_{C(S) \sim \mathcal{P}}[ \mathop{\E}_{\substack{x \sim \Pi, A}} &[\ell(\bayes_{q \mid C(S)}, q(x)) - \ell(\bayes_{q \mid C(S), A(S)}, q(x))]] \\
    &= \sup_{S} \mathop{\E}_{C(S) \sim \mathcal{P}}[ \mathop{\E}_{\substack{x \sim \Pi}} [(\bayes_{q \mid C(S)} - q(x))^2 - (\bayes_{q \mid C(S), A(S)} - q(x))^2]] \\
    &= \sup_{S} \mathop{\E}_{C(S) \sim \mathcal{P}}[ \mathop{\E}_{\substack{x \sim \Pi}} [(p_{C(S)} - q(x))^2 - (p_{S \setminus C(S)} - q(x))^2]] \\
    &= \sup_{S} \mathop{\E}_{C(S) \sim \mathcal{P}}[ (p_{C(S)} - \mathop{\E}_{\substack{x \sim \Pi}}[q(x)])^2 + \mathop{\E}_{\substack{x \sim \Pi}} [(\mathop{\E}_{\substack{x \sim \Pi}}[q(x)] - q(x))^2] - (p_{S \setminus C(S))} - \mathop{\E}_{\substack{x \sim \Pi}}[q(x)])^2 \\
    &\;\;\;\;\;\;- \mathop{\E}_{\substack{x \sim \Pi}} [(\mathop{\E}_{\substack{x \sim \Pi}}[q(x)] - q(x))^2] ] \tag{bias-variance decomposition} \\
    &= \sup_{S} \mathop{\E}_{C(S) \sim \mathcal{P}}[ (p_{C(S)} - \mathop{\E}_{\substack{x \sim \Pi}}[q(x)])^2 - (p_{S \setminus C(S))} - \mathop{\E}_{\substack{x \sim \Pi}}[q(x)])^2] \\
    &= \sup_{S} \mathop{\E}_{C(S) \sim \mathcal{P}}[ (p_{C(S)} - p_{S \setminus C(S)})^2 - (p_{S \setminus C(S))} - p_{S \setminus C(S)})^2] \\
    &= \sup_{S} \mathop{\E}_{C(S) \sim \mathcal{P}}[ (p_{C(S)} - p_{S \setminus C(S)})^2 ] \\
    &= \sup_{S} \mathop{\E}_{C(S) \sim \mathcal{P}}[ ((p_{C(S)} - p_S) - (p_{S \setminus C(S)} - p_S))^2 ] \\
    &= \sup_{S} \mathop{\E}_{C(S) \sim \mathcal{P}}[ ((p_{C(S)} - p_S))^2 ] + \mathop{\E}_{C(S) \sim \mathcal{P}}[ ((p_{S \setminus C(S)} - p_S))^2 ] - 2 \mathop{\E}_{C(S) \sim \mathcal{P}}[ (p_{C(S)} - p_S) (p_{S \setminus C(S)} - p_S)] \\ 
    &\leq O(\tfrac{1}{N}) + O(\tfrac{1}{N}) + 0 \\
    &= O(\tfrac{1}{N}),
\end{align*}
where the last line follows from the concentration of $p_{C(S)}$ and $p_{S \setminus C(S)}$ around $p_S$.
\end{proof}

\subsubsection{\pref{thm:pred-implies-dp}, Connection to Mutual Information Differential Privacy}

\predimpliesdp*
\begin{proof}
    We know that $A(S)$ satisfies
    \begin{align*}
        \sup_{S} \E_{C(S) \sim \mathcal{P}_{\mathrm{DP}}, A} [ \sup_{ q \in \mathcal{Q}_{\mathrm{bin}}} \E_{x \sim \mathrm{Unif}(S \setminus C(S))} [\ell(\bayes_{q \mid C(S)}, q(x)) - \ell(\bayes_{q \mid C(S), A(S)}, q(x))]] \leq \gamma.
    \end{align*}
    Plugging in the form of $\mathcal{P}_{\mathrm{DP}}$, we have
    \begin{align*}
        \sup_{S} \sup_{i \in [N]} \sup_{q \in \mathcal{Q}_{\mathrm{bin}}} \E_A[\ell(\bayes_{q \mid S \setminus \{x_i\}}, q(x_i)) - \ell(\bayes_{q \mid S \setminus \{x_i\}, A(S)}, q(x_i))]] \leq \gamma.
    \end{align*}
    This is equivalent to
    \begin{align*}
        \sup_{\substack{\chi \in \Delta (\mathcal{X}^N)}} &\E_{S \sim \chi} [ \sup_{i \in [N]} \sup_{ q \in \mathcal{Q}_{\mathrm{bin}}} \E_A[\ell(\bayes_{q \mid S \setminus \{x_i\}}, q(x_i)) - \ell(\bayes_{q \mid S \setminus \{x_i\}, A(S)}, q(x_i))]] \leq \gamma.
    \end{align*}
    Observe that
    \begin{align*}
        \E_{S \sim \chi} [ &\sup_{i \in [N], q \in \mathcal{Q}_{\mathrm{bin}}} \E_A[\ell(\bayes_{q \mid S \setminus \{x_i\}}, q(x_i)) - \ell(\bayes_{q \mid S \setminus \{x_i\}, A(S)}, q(x_i))]] \\
        &\geq \sup_{i \in [N], q \in \mathcal{Q}_{\mathrm{bin}}} \E_{S \sim \chi, A} [ \ell(\bayes_{q \mid S \setminus \{x_i\}}, q(x_i)) - \ell(\bayes_{q \mid S \setminus \{x_i\}, A(S)}, q(x_i))] 
    \end{align*}

    Thus,
    \begin{align*}
        \sup_{i \in [N], q \in \mathcal{Q}_{\mathrm{bin}}} \E_{S \sim \chi, A} [ \ell(\bayes_{q \mid S \setminus \{x_i\}}, q(x_i)) - \ell(\bayes_{q \mid S \setminus \{x_i\}, A(S)}, q(x_i))] \leq \gamma.
    \end{align*}

    Under the log loss, the optimal $\bayes_{q \mid C(S)}$ and $\bayes_{q \mid C(S), A(S)}$ predict $\Pr[q(x_i)=1 \mid{} S \setminus \{x_i\}]$ and $\Pr[q(x_i)=1 \mid{} S \setminus \{x_i\}, A(S)]$, respectively.
    
    Plugging these in and expanding the loss,
    \begin{align*}
        &\E_{S \sim \chi, A}[\ell(\Pr[q(x_i)=1 \mid{} S \setminus \{x_i\}], q(x_i)) - \ell(\Pr[q(x_i)=1 \mid{} S \setminus \{x_i\}, A(S)], q(x_i))] \\
        &\;\;\;=\E_{S \setminus \{x_i\} \sim \chi}[ \E_{x_i \mid S \setminus \{x_i\}, A}[ \ell(\Pr[q(x_i)=1 \mid{} S \setminus \{x_i\}], q(x_i)) - \ell(\Pr[q(x_i)=1 \mid{} S \setminus \{x_i\}, A(S)], q(x_i))]] \\
        &\;\;\;= \E_{S \setminus \{x_i\} \sim \chi} [ D_{\textsc{kl}}(P_{q(x) \mid{} S \setminus \{x_i\}, A(S)} \| P_{q(x) \mid{} S \setminus \{x_i\}}) ] \tag{\pref{thm:predtokl}} \\
        &\;\;\;= \I(q(x_i); A(S) \mid{} S \setminus \{x_i\})
    \end{align*}
    where $\I$ is the conditional mutual information between $q(x_i)$ and $A(S)$ given $S \setminus \{x_i\}$.

    Thus, we have
    \begin{align*}
        \sup_{\chi \in \Delta(\mathcal{X}^N)} \sup_{i \in [N]} \sup_{ q \in \mathcal{Q}_{\mathrm{bin}}} \I(q(x_i); A(S) \mid{} S \setminus \{x_i\}) \leq \gamma,
    \end{align*}
    which implies
    \begin{align*}
        \sup_{\chi \in \Delta(\mathcal{X}^N)} \sup_{i \in [N]} \I(x_i; A(S) \mid{} S \setminus \{x_i\}) \leq \gamma.
    \end{align*}
    Thus, $A(S)$ is $\gamma$-mutual-information differentially private (\pref{def:mi-dp}). 
\end{proof}

\begin{restatable}{theorem}{dpimpliespred} \label{thm:dp-implies-pred}
    If $A$ is $\varepsilon$-differentially private, then for process $\mathcal{P}_{\mathrm{DP}}$, $A$ satisfies $\varepsilon$-predictability for all binary queries $\mathcal{Q}_{\mathrm{bin}}$ under the log loss.
\end{restatable}
\begin{proof}

    Plugging in the form of $\mathcal{P}_{\mathrm{DP}}$ into predictability, we look to bound
    \begin{align*}
        \sup_{S} \sup_{i \in [N]} \sup_{q \in \mathcal{Q}_{\mathrm{bin}}} \E_A[\ell(\bayes_{q \mid S \setminus \{x_i\}}, q(x_i)) - \ell(\bayes_{q \mid S \setminus \{x_i\}, A(S)}, q(x_i))]
    \end{align*}

    Since $A$ is $\varepsilon$-differentially private, we know that for all $S$ and $i$,
    \begin{align*}
        e^{-\varepsilon} \leq \frac{\Pr[A(S) \mid{} S]}{\Pr[A(S) \mid{} S \setminus \{x_i\}]} \leq e^\varepsilon.
    \end{align*}

    Applying Bayes Rule, we know
    \begin{align*}
         \Pr[q(x_i) = 1 \mid{} S \setminus \{x_i\}, A(S)] &= \frac{\Pr[A(S) \mid{} S \setminus \{x_i\}, q(x_i) = 1] \cdot \Pr[q(x_i) = 1 \mid{} S \setminus \{x_i\}]}{\Pr[A(S) \mid{} S \setminus \{x_i\}]}.
    \end{align*}

    Let $\mathcal{S} = \{S' \mid{} S' \setminus \{x_i'\} = S \setminus \{x_i\} \text{ and } q(x_i')=1$. Then
    \begin{align*}
         \Pr[q(x_i) = 1 \mid{} S \setminus \{x_i\}, A(S)] &= \frac{\sum_{S' \in \mathcal{S}} \Pr[A(S) \mid{} S'] \cdot \Pr[q(x_i) = 1 \mid{} S \setminus \{x_i\}]}{\Pr[A(S) \mid{} S \setminus \{x_i\}]}.
    \end{align*}

    Thus, plugging in the ratio from differential privacy, we have
    \begin{align*}
        e^{-\varepsilon} \cdot \Pr[q(x_i) = 1 \mid{} S \setminus \{x_i\}] \leq \Pr[q(x_i) = 1 \mid{} S \setminus \{x_i\}, A(S)] &\leq e^\varepsilon \cdot \Pr[q(x_i) = 1 \mid{} S \setminus \{x_i\}].
    \end{align*}

    We can similarly conclude
    \begin{align*}
        e^{-\varepsilon} \cdot \Pr[q(x_i) = 0 \mid{} S \setminus \{x_i\}] \leq \Pr[q(x_i) = 0 \mid{} S \setminus \{x_i\}, A(S)] &\leq e^\varepsilon \cdot \Pr[q(x_i) = 0 \mid{} S \setminus \{x_i\}].
    \end{align*}

    In other words, 
    \begin{align*}
        -\varepsilon \leq \log \left(\frac{\Pr[q(x_i)=1 \mid{} S \setminus \{x_i\}, A(S)]}{\Pr[q(x_i)=1 \mid{} S \setminus \{x_i\}]}\right), \log \left(\frac{\Pr[q(x_i)=0 \mid{} S \setminus \{x_i\}, A(S)]}{\Pr[q(x_i)=0 \mid{} S \setminus \{x_i\}]}\right) \leq \varepsilon.
    \end{align*}

    Expanding out the log loss in the predictability bound,
    \begin{align*}
        \ell(&\bayes_{q \mid S \setminus \{x_i\}}, q(x_i)) - \ell(\bayes_{q \mid S \setminus \{x_i\}, A(S)}, q(x_i)) \\
        &= \ell(\Pr[q(x_i)=1 \mid{} S \setminus \{x_i\}], q(x_i)) - \ell(\Pr[q(x_i)=1 \mid{} S \setminus \{x_i\}, A(S)], q(x_i)) \\
        &= -\Pr[q(x_i)=1]\cdot \log (\Pr[q(x_i)=1 \mid{} S \setminus \{x_i\}]) - (1-\Pr[q(x_i)=1]) \cdot \log (1-\Pr[q(x_i)=1 \mid{} S \setminus \{x_i\}]) \\
        &\;\;\;\;\;+ \Pr[q(x_i)=1] \cdot \log (\Pr[q(x_i)=1 \mid{} S \setminus \{x_i\}, A(S)]) \\
        &\;\;\;\;\;- (1-\Pr[q(x_i)=1]) \cdot \log (1-\Pr[q(x_i)=1 \mid{} S \setminus \{x_i\}, A(S)]) \\
        &= \Pr[q(x_i)=1] \cdot \log \left(\frac{\Pr[q(x_i)=1 \mid{} S \setminus \{x_i\}, A(S)]}{\Pr[q(x_i)=1 \mid{} S \setminus \{x_i\}]}\right) + \Pr[q(x_i)=0] \cdot \log \left(\frac{\Pr[q(x_i)=0 \mid{} S \setminus \{x_i\}, A(S)]}{\Pr[q(x_i)=0 \mid{} S \setminus \{x_i\}]}\right)
    \end{align*}

    Since both
    \begin{align*}
        -\varepsilon \leq \log \left(\frac{\Pr[q(x_i)=1 \mid{} S \setminus \{x_i\}, A(S)]}{\Pr[q(x_i)=1 \mid{} S \setminus \{x_i\}]}\right), \log \left(\frac{\Pr[q(x_i)=0 \mid{} S \setminus \{x_i\}, A(S)]}{\Pr[q(x_i)=0 \mid{} S \setminus \{x_i\}]}\right) \leq \varepsilon,
    \end{align*}
    we must have
    \begin{align*}
        -\varepsilon \leq \ell(\bayes_{q \mid S \setminus \{x_i\}}, q(x_i)) - \ell(\bayes_{q \mid S \setminus \{x_i\}, A(S)}, q(x_i)) \leq \varepsilon.
    \end{align*}

    The above holds for every $S$, $i$, and $q$; thus, we can conclude
    \begin{align*}
        \sup_{S} \sup_{i \in [N]} \sup_{q \in \mathcal{Q}_{\mathrm{bin}}} \E_A[\ell(\bayes_{q \mid S \setminus \{x_i\}}, q(x_i)) - \ell(\bayes_{q \mid S \setminus \{x_i\}, A(S)}, q(x_i))] \leq \gamma.
    \end{align*}
\end{proof}

\subsubsection{\pref{lem:loss-lem}}

\losslem*
\begin{proof}
    \begin{align*}
        \mathop{\E}_{\substack{x_i \sim \Pi}}[\ell(\hat{\theta}_{q \mid C(S)}, q(x))& - \ell(\hat{\theta}_{q \mid C(S), A(S)}, q(x))] \\
        &= \ell (\hat{\theta}_{q \mid C(S)}) - \ell (\hat{\theta}_{q \mid C(S), A(S)}) \\
        &\leq \ell'(\hat{\theta}_{q \mid C(S)}) |\hat{\theta}_{q \mid C(S)} - \hat{\theta}_{q \mid C(S), A(S)}| \tag{convexity} \\
        &\leq H |\hat{\theta}_{q \mid C(S)} - \theta^q_0| \cdot |\hat{\theta}_{q \mid C(S)} - \hat{\theta}_{q \mid C(S), A(S)}| \tag{smoothness}
    \end{align*}    
\end{proof}

\subsection{Post-Processing and Composition of Predictability} \label{app:pp-and-composition}

\begin{restatable}[Post-Processing]{theorem}{postprocessing}
    If $A(S)$ has $\gamma (\mathcal{P}, \mathcal{Q}, \ell, A)$-predictability $\gamma$, then $f \circ A$ has $\gamma (\mathcal{P}, \mathcal{Q}, \ell, f \circ A))$-predictability $\gamma$.
\end{restatable}

\begin{proof}
    Notice that due to the monotonicity of Bayes risk, the Bayes optimal predictor must satisfy
    \begin{align*}
        \E[\ell(\bayes_{q \mid C(S), A(S)}, q(x))] \leq \E[\ell(\bayes_{q \mid C(S), f(A(S))}, q(x))].
    \end{align*}
    Thus,
    \begin{align*}
        \sup_{S} \mathop{\E}_{C(S) \sim \mathcal{P}, A}[ \sup_{\substack{q \in \mathcal{Q}}} \mathop{\E}_{\substack{x \sim \Pi}} &[\ell(\bayes_{q \mid C(S)}, q(x))  - \ell(\bayes_{q \mid C(S), f(A(S))}, q(x))]] \\
        &\leq \sup_{S} \mathop{\E}_{C(S) \sim \mathcal{P}, A}[ \sup_{\substack{q \in \mathcal{Q}}} \mathop{\E}_{\substack{x \sim \Pi}} [\ell(\bayes_{q \mid C(S)}, q(x))  - \ell(\bayes_{q \mid C(S), A(S)}, q(x))]]
    \end{align*}
\end{proof}

\begin{restatable}[Composition]{theorem}{composition}
    For two algorithms $A_1$ and $A_2$, the $\gamma (\mathcal{P}, \mathcal{Q}, \ell, A)$-predictability of the composed $(A_1(S), A_2(S))$ is bounded by
    \begin{align*}
        \gamma (\mathcal{P}, \mathcal{Q}, \ell, (A_1, A_2)) &\leq \sup_{S} \mathop{\E}_{C(S) \sim \mathcal{P}, A_1}[ \sup_{\substack{q \in \mathcal{Q}}} \mathop{\E}_{\substack{x \sim \Pi}} [\ell(\bayes_{q \mid C(S)}, q(x))  - \ell(\bayes_{q \mid C(S), A_1(S)}, q(x))]] \\
        &\;\;\;\;\;\;+ \sup_{S} \mathop{\E}_{C(S) \sim \mathcal{P}, A_1, A_2}[ \sup_{\substack{q \in \mathcal{Q}}} \mathop{\E}_{\substack{x \sim \Pi}} [\ell(\bayes_{q \mid C(S), A_1(S)}, q(x)) - \ell(\bayes_{q \mid C(S), A_1(S), A_2(S)}, q(x))]]
    \end{align*}
\end{restatable}

\begin{proof}
    \begin{align*}
        \sup_{S} &\mathop{\E}_{C(S) \sim \mathcal{P}}[ \sup_{\substack{q \in \mathcal{Q}}} \mathop{\E}_{\substack{x \sim \Pi, A_1, A_2}} [\ell(\bayes_{q \mid C(S)}, q(x)) - \ell(\bayes_{q \mid C(S), A_1(S), A_2(S)}, q(x))]] \\
        &= \sup_{S} \mathop{\E}_{C(S) \sim \mathcal{P}}[ \sup_{\substack{q \in \mathcal{Q}}} \mathop{\E}_{\substack{x \sim \Pi, A_1, A_2}} [\ell(\bayes_{q \mid C(S)}, q(x))  - \ell(\bayes_{q \mid C(S), A_1(S)}, q(x)) \\
        &\;\;\;\;\;\;+ \ell(\bayes_{q \mid C(S), A_1(S)}, q(x)) - \ell(\bayes_{q \mid C(S), A_1(S), A_2(S)}, q(x))]] \\
        &\leq \sup_{S} \mathop{\E}_{C(S) \sim \mathcal{P}}[ \sup_{\substack{q \in \mathcal{Q}}} \mathop{\E}_{\substack{x \sim \Pi, A_1}} [\ell(\bayes_{q \mid C(S)}, q(x))  - \ell(\bayes_{q \mid C(S), A_1(S)}, q(x))]] \\
        &\;\;\;\;\;\;+ \sup_{S} \mathop{\E}_{C(S) \sim \mathcal{P}}[ \sup_{\substack{q \in \mathcal{Q}}} \mathop{\E}_{\substack{x \sim \Pi, A_1, A_2}} [\ell(\bayes_{q \mid C(S), A_1(S)}, q(x)) - \ell(\bayes_{q \mid C(S), A_1(S), A_2(S)}, q(x))]]
    \end{align*}
    where the last line follows from the subadditivity of supremum.
\end{proof}

\subsection{Proofs from \pref{sec:gmm}}

\subsubsection{Regularity Conditions} \label{app:reg-conditions}
Define sample quantities
\begin{align*}
    f_n (\theta) &= \frac{1}{n} \sum_{i=1}^n f(x_i, \theta) \\
    G_n (\theta) &= \frac{1}{n} \sum_{i=1}^n \grad f(x_i, \theta) \\
    \Omega_n (\theta) &= \frac{1}{n} \sum_{i=1}^n f(x_i, \theta) f(x_i, \theta)^\top
\end{align*}
and define population quantities
\begin{align*}
    \E[f(\theta)] &= \E[f(x, \theta)] \\
    G(\theta) &= \E[\grad f(x, \theta)] \\
    \Omega (\theta) &= \sum_{k = -\infty}^{\infty} f(x_0, \theta) f(x_k, \theta)^\top.
\end{align*}
Assume the following regularity conditions.
    \begin{itemize}[nosep, leftmargin=2em]
        \item The domain $\Theta$ of $\theta$ is a compact subset of $\mathbb{R}^r$ and $\theta_0$ is its interior.
        \item $f(x, \theta)$ is $\mathbb{R}^r \mapsto \mathbb{R}^m$ function of $\theta$, where $r \leq m$.
        \item $f$ is continuous in $\theta$ and continuously differentiable on a neighborhood of $\theta_0$ in $\theta$.
        \item $\E[f(\theta)] = 0$ if and only if $\theta = \theta_0$.
        \item $G(\theta_0)$ is a $m \times r$ matrix of rank $r$.
        \item $\Omega (\theta_0)$ is a $m \times m$ positive definite matrix.
        \item $W$ is a $m \times m$ positive definite matrix.
        \item $f(x, \theta)$, $G (\theta)$, and $\Sigma (\theta)$ are dominated (there exists a function $\beta(x)$ with $\E[\beta(x)] < +\infty$ such that $|f(x, \theta)| \leq \beta(x)$, $|G (\theta)| \leq \beta(x)$ and $|\Omega (\theta)| \leq \beta(x)$).
        \item For each $\theta$, $f_n (\theta)$ obeys a central limit theorem (\pref{thm:clt-se}) 
        \item $f_n (\theta)$, $G_n (\theta)$, and $\Omega_n (\theta)$ obey a uniform law of large numbers (\pref{thm:ulln-se}).
    \end{itemize}

\begin{theorem}[Uniform Law of Large Numbers for Stationary Ergodic Processes \citep{hansen1982}] \label{thm:ulln-se}
    Let $Y_1(\theta), \dots, Y_n (\theta)$ be from a stationary ergodic process. Define
    \begin{align*}
        \psi(\theta) = \E[Y_1(\theta)], \;\;\;\; \bar{Y}_n (\theta) = \frac{1}{n} \sum_{i=1}^n Y_i (\theta).
    \end{align*}
    Assume a finite mean $\psi (\theta) < +\infty$ for all $\theta$ in a closed bounded set $\Theta \subseteq \mathbb{R}^k$. Assume $Y_i (\theta)$ is continuous at each $\theta \in \Theta$. Assume that $Y_i (\theta)$ is dominated; there exists a random variable $Z$ with a finite mean such that $\sup_{\theta \in \Theta} |Y_i (\theta)| \leq Z$. Then
    \begin{align*}
        \sup_{\theta \in \Theta} |\bar{Y}_n (\theta) - \psi(\theta)| \rightarrow_p 0
    \end{align*}
\end{theorem}

\begin{theorem}[Central Limit Theorem for Stationary Ergodic Processes \citep{hansen1982}] \label{thm:clt-se}
    Let $X_1, \dots, X_n$ be from a stationary ergodic process. Assume $(X_t)$ is $\alpha$-mixing with mixing coefficients satisfying
    \begin{align*}
        \sum_{k=1}^\infty \alpha(k)^{\delta/(2+\delta)} < \infty.
    \end{align*}
    As $n \rightarrow \infty$, %
    \begin{align*}
        \sqrt{n} (\Bar{X}_n - \theta) \rightarrow_d \mathcal{N}(0, \Omega_{\infty})
    \end{align*}
    where $\theta$ is the population mean and $\Omega_{\infty} = \sum_{k=-\infty}^{\infty} \mathrm{Cov}(X_0, X_k)$ is the long-run covariance.
\end{theorem}

\subsubsection{\pref{thm:gmm-efficiency}, GMM Efficiency}

\gmmefficiency*

This is a standard result from \cite{hansen1982}.

For completeness, we include a proof of the efficient GMM estimator for sampling without replacement. In particular, consider sample quantities
\begin{align*}
    f_n (\theta) &= \frac{1}{n} \sum_{i=1}^n f(x_i, \theta) \\
    G_n (\theta) &= \frac{1}{n} \sum_{i=1}^n \grad f(x_i, \theta) \\
    \Omega_n (\theta) &= \frac{1}{n} \sum_{i=1}^n f(x_i, \theta) f(x_i, \theta)^\top
\end{align*}
and population quantities
\begin{align*}
    \E[f(\theta)] &= \frac{1}{N} \sum_{i=1}^N f(x_i, \theta) \\
    G(\theta) &= \frac{1}{N} \sum_{i=1}^N \grad f(x_i, \theta) \\
    \Omega (\theta) &= \frac{1}{N} \sum_{i=1}^N f(x_i, \theta) f(x_i, \theta)^\top.
\end{align*}

Consider the specific form of ULLN and CLT for sampling without replacement.

\begin{theorem}[Uniform Law of Large Numbers for Sampling without Replacement \citep{horowitz1990uniformlln}] \label{thm:ulln-swor}
    Let $Y_1(\theta), \dots, Y_n (\theta)$ be a simple random sample without replacement from finite population $\{Y_1 (\theta), \dots, Y_N (\theta)\}$, where $n/N = \nu$, $0 < \nu < 1$. Define
    \begin{align*}
        \psi(\theta) = \frac{1}{N} \sum_{i=1}^N Y_i (\theta), \;\;\;\; \bar{Y}_n (\theta) = \frac{1}{n} \sum_{i=1}^n Y_i (\theta).
    \end{align*}
    Assume a finite mean $\psi (\theta) < +\infty$ for all $\theta$ in a closed bounded set $\Theta \subseteq \mathbb{R}^k$. Assume $Y_i (\theta)$ is continuous at each $\theta \in \Theta$. Assume that $Y_i (\theta)$ is dominated; there exists a random variable $Z$ with a finite mean such that $\sup_{\theta \in \Theta} |Y_i (\theta)| \leq Z$. Then
    \begin{align*}
        \sup_{\theta \in \Theta} |\bar{Y}_n (\theta) - \psi(\theta)| \rightarrow_p 0
    \end{align*}
\end{theorem}

\begin{theorem}[Central Limit Theorem for Sampling without Replacement \citep{zuzana2009asymptotics-finite-sampling}] \label{thm:clt-swor}
    When sampling $n$ elements from $N$ elements uniformly at random, where $n/N = \nu$, $0 < \nu < 1$. As $n, N \rightarrow \infty$, %
    \begin{align*}
        \sqrt{\frac{n}{1 - \nu}} (\Bar{X}_n - \theta) \rightarrow_d \mathcal{N}(0, F)
    \end{align*}
    where $\theta, F$ are the population mean and covariance matrix.
\end{theorem}

\begin{theorem}[Optimal GMM Estimator for Sampling without Replacement] \label{thm:gmm-efficiency-swor}
    Consider a sample without replacement $x_1, \dots, x_n$ from a larger population of $N$ elements, where $n/N = \nu$ and $0 < \nu < 1$ is a constant.
    Given the sample, the GMM estimator of $\theta$ is defined as
    \begin{align*}
        \hat{\theta}_n = \arg \min_{\theta \in \Theta \subseteq \mathbb{R}^r} f_n (\theta)^\top W f_n (\theta).
    \end{align*}
    Assume the regularity conditions from \pref{app:reg-conditions} hold.
    Then $\hat{\theta}_n$ is asymptotically normal,
    \begin{align*}
        \sqrt{\frac{n}{1 - \nu}} (\hat{\theta}_n - \theta_0) \rightarrow_d \mathcal{N}(0, (G(\theta_0)^\top W G(\theta_0))^{-1} G(\theta_0)^\top W \Omega(\theta_0) W G(\theta_0) (G(\theta_0)^\top W G(\theta_0))^{-1}).
    \end{align*}
    With optimal choice $W = \Omega (\theta_0)$, $\gmm_n$ is asymptotically normal as
    \begin{align*}
        \sqrt{\frac{n}{1 - \nu}} (\gmm_n - \theta_0) \rightarrow_d \mathcal{N}(0, (G(\theta_0)^\top W G(\theta_0))^{-1}.
    \end{align*}
    and is asymptotically efficient among the class of all regular estimators that only use the information contained in the moment conditions.
\end{theorem}
\begin{proof}
    The proof follows similarly to the standard GMM proof for IID samples from \cite{mcfadden1999econometricsbook}, Chapter 3, Theorem 3.1.

    Let $Q_n = f_n (\theta)^\top W f_n (\theta)$.

    For every fixed $\theta$, $f_n (\theta)$, $G_n(\theta)$, and $\Omega_n (\theta)$ converge in probability to $\E[f(x, \theta)]$, $G(\theta)$, and $\Omega(\theta)$, respectively, using \pref{thm:ulln-swor}. Furthermore, 
    \begin{align*}
        \Omega(\theta_0)^{-1/2} \sqrt{\frac{n}{1 - \nu}} (f_n(\theta_0)) \equiv U_n \rightarrow_d U \sim \mathcal{N}(0, I).
    \end{align*}
    For each fixed $\theta$, the uniform law of large numbers implies that $f_n(\theta)\rightarrow_p \E[f(\theta)]$. This implies that $Q_n (\theta_0) \rightarrow_p \E[f(\theta_0)]^\top W \E[f(\theta_0)] = 0$. Thus, $\frac{n}{1-\nu} \cdot Q_n (\theta_0)$ is stochastically bounded ($\frac{n}{1-\nu} \cdot Q_n (\theta_0)$ does not go to infinity in probability).
    
    Consider any estimator $T_n^*$ that satisfies $Q_n(T_n^*) \rightarrow_p 0$. For each fixed $\theta$, \pref{thm:ulln-swor} implies that $f_n(\theta)\rightarrow_p \E[f(\theta)]$. Then applying \pref{thm:func-of-uniform}, $Q_n (\theta) \rightarrow_p \E[f(\theta)]^\top W \E[f(\theta)]$. Since $\E[f(\theta)] = 0$ only when $\theta = \theta_0$, then $T_n^* \rightarrow_p \theta_0$. $T_n^*$ is consistent.
    
    Consider any estimator $T_n^*$ that satisfies $\frac{n}{1-\nu} \cdot Q_n(T_n^*)$ is stochastically bounded. This implies that $Q_n(T_n^*) \rightarrow_p 0$, which means $T_n^* \rightarrow_p \theta_0$. The mean value theorem and the CLT condition for $f_n (\theta)$ give
    \begin{align*}
        \sqrt{\frac{n}{1 - \nu}} (f_n(T_n^*)) &= \sqrt{\frac{n}{1 - \nu}} (f_n(\theta_0)) - G_n (T_n^{\textsc{mid}}) \sqrt{\frac{n}{1 - \nu}} (T_n^* - \theta_0) \\
        \sqrt{\frac{n}{1 - \nu}} (f_n(T_n^*)) &= \Omega(\theta_0)^{1/2} U_n - G_n (T_n^{\textsc{mid}}) \sqrt{\frac{n}{1 - \nu}} (T_n^* - \theta_0)
    \end{align*}
    for some intermediate value $T_n^{\textsc{mid}}$. Then using the triangle inequality, $\norm{a-b}_W^2 \leq \norm{a}^2_W + \norm{b}^2_W$,
    \begin{align*}
        \norm{G_n (T_n^{\textsc{mid}}) \sqrt{\frac{n}{1 - \nu}} (T_n^* - \theta_0)}^2_{W} &\leq \norm{\sqrt{\frac{n}{1 - n/N}} (f_n(T_n^*))}^2_W + \norm{\Omega(\theta_0)^{1/2} U_n}^2_W \\
        \frac{n}{1 - \nu} \cdot (T_n^* - \theta_0) G_n (T_n^{\textsc{mid}})^\top W G_n (T_n^{\textsc{mid}}) (T_n^* - \theta_0) &\leq \frac{n}{1 - \nu} \cdot Q_n (T_n^*) + \frac{1}{2} U_n^\top \Omega(\theta_0)^{1/2} W \Omega(\theta_0)^{1/2} U_n
    \end{align*}
    $U_n$ and $\frac{n}{1 - \nu} \cdot Q_n (T_n^*)$ are stochastically bounded, so $\frac{n}{1 - \nu} \cdot (T_n^* - \theta_0) G_n (T_n^{\textsc{mid}})^\top W G_n (T_n^{\textsc{mid}}) (T_n^* - \theta_0)$ is stochastically bounded.
    
    Since we have $T_n^{\textsc{mid}} \rightarrow_p \theta_0$ and uniform convergence of $G_n (\theta)$, \pref{thm:continous-mapping} implies that $G_n(T_n^{\textsc{mid}})^\top W G_n (T_n^{\textsc{mid}}) \rightarrow_p G(\theta_0)^\top W G(\theta_0)$, where $G(\theta_0)^\top W G(\theta_0)$ is positive definite. Let $\lambda$ be the smallest eigenvalue of $G(\theta)^\top W G(\theta)$. Then in the probability limit,
    \begin{align*}
        \lambda \cdot \norm{\sqrt{\frac{n}{1 - \nu}} (T_n^* - \theta_0)}^2 \leq \frac{n}{1 - \nu} \cdot (T_n^* - \theta_0) G_n (T_n^{\textsc{mid}})^\top W G_n (T_n^{\textsc{mid}}) (T_n^* - \theta_0).
    \end{align*}
    Thus, $\sqrt{\frac{n}{1 - \nu}} \cdot (T_n^* - \theta_0)$ must be stochastically bounded. This in turn implies that $\sqrt{\frac{n}{1 - \nu}} (f_n(T_n^*))$ is stochastically bounded.
    
    Consider the estimator $T_n = \arg \min Q_n(\theta)$. We must have $Q_n(T_n) \leq Q_n(\theta_0)$. $Q_n (\theta_0)$ and $Q_n(T_n)$ are stochastically bounded, which implies that $T_n$ is consistent and $\sqrt{\frac{n}{1 - \nu}} (f_n(T_n))$ is stochastically bounded, by applying the previous steps. $T_n$ must satisfy the first order condition $G_n(T_n)^\top W \sqrt{\frac{n}{1 - \nu}} f_n(T_n) = 0$. Using the mean value theorem, the first order condition is equivalent to
    \begin{align*}
        - G_n(T_n)^\top W \Omega(\theta_0)^{1/2} U_n + G_n(T_n)^\top W G_n (T_n^{\textsc{mid}}) \sqrt{\frac{n}{1 - \nu}} (T_n - \theta_0)
    \end{align*}
    
    In the probability limit,
    \begin{align*}
        \sqrt{\frac{n}{1 - \nu}} (T_n - \theta_0) = (G_n(T_n)^\top W G_n (T_n^{\textsc{mid}}))^{-1} G_n(T_n)^\top W \Omega(\theta_0)^{1/2} U_n
    \end{align*}
    exists because $G_n(T_n^{\textsc{mid}})^\top W G_n (T_n) \rightarrow_p G(\theta_0)^\top W G(\theta_0)$, and $G(\theta)^\top W G(\theta)$ is positive definite. Then in the limit,
    \begin{align*}
        \sqrt{\frac{n}{1 - \nu}} (T_n - \theta_0) \rightarrow_d (G(\theta_0)^\top W G(\theta_0))^{-1} G(\theta_0)^\top W \Omega(\theta_0)^{1/2} U
    \end{align*}
    where $G_n(T_n^{\textsc{mid}})^\top W G_n (T_n)$ converges in probability and in distribution to $(G(\theta_0)^\top W G(\theta_0))^{-1}$, $G_n(T_n)^\top W \Omega(\theta_0)^{1/2}$ converges in probability and in distribution to $G(\theta_0)^\top W \Omega(\theta_0)^{1/2}$, and $U_n$ converges in
    distribution to $U$. Thus, 
    \begin{align*}
        \sqrt{\frac{n}{1 - \nu}} (T_n - \theta_0) \rightarrow_d \mathcal{N}(0, (G(\theta_0)^\top W G(\theta_0))^{-1} G(\theta_0)^\top W \Omega(\theta_0) W G(\theta_0) (G(\theta_0)^\top W G(\theta_0))^{-1}).
    \end{align*}
    The choice of $W = \Omega(\theta_0)$ leads to the smallest covariance of $(G(\theta_0)^\top \Omega(\theta_0) G(\theta_0))^{-1}$ using \pref{thm:opt-w}.

    Finally, following \cite{hansen1982} and Proposition 3.6 from \cite{DOVONON2025105723}, we justify why the choice of $W = \Omega(\theta_0)$ is asymptotically efficient.

    Consider any consistent regular estimator that only uses the moment conditions for estimation. It must lie in the span of the moment conditions and take the form,
    $A^\top f_n$, and have asymptotic variance $A^\top \Omega A$, since $\E[f_n(\theta_0) f_n(\theta_0)^\top] = \Omega$. The consistent estimator must also satisfy the first-order condition $A^\top G = I$. Consider the minimum possible variance that can be achieved,
    \begin{align*}
        \min_{A \text{ s.t. } A^\top G = I} A^\top \Omega A = (G^\top \Omega G)^{-1},
    \end{align*}
    thus demonstrating optimality and asymptotic efficiency among the class of all estimators that only use the moment conditions for any estimator that achieves asymptotic variance $(G^\top \Omega G)^{-1}$.
\end{proof}

\bayesandgmm*
\begin{proof}
    \begin{align*}
        \mathop{\E}_{\substack{x_i \sim \Pi}}[\ell(\bayes_{q \mid C(S)}, q&(x)) - \ell(\bayes_{q \mid C(S), A(S)}, q(x))] - \mathop{\E}_{\substack{x_i \sim \Pi}}[\ell(\gmm_{q \mid C(S)}, q(x)) - \ell(\gmm_{q \mid C(S), A(S)}, q(x))] \\
        &= \ell (\bayes_{q \mid C(S)}) - \ell (\bayes_{q \mid C(S), A(S)}) - \ell (\gmm_{q \mid C(S)}) + \ell (\gmm_{q \mid C(S), A(S)}) \\
        &= \ell (\bayes_{q \mid C(S)}) - \ell (\gmm_{q \mid C(S)}) + \ell (\gmm_{q \mid C(S), A(S)}) -\ell (\bayes_{q \mid C(S), A(S)}) \\ 
        &\leq o\left(\frac{1}{n}\right) + o\left(\frac{1}{n}\right) \tag{\pref{thm:diff-of-efficient-est}} \\
        &= o\left(\frac{1}{n}\right)
    \end{align*}    

    We can apply \pref{thm:diff-of-efficient-est} because both $\bayes$ and $\gmm$ are asymptotically efficient, asymptotically normal, and regular. 
\end{proof}

\begin{theorem} \label{thm:diff-of-efficient-est}
    Let $\hat{\theta}_1$ and $\hat{\theta}_2$ be any two asymptotically efficient and asymptotically normal regular estimators with respect to the same observation and information. Let $\ell(\theta)$ be a twice continuously differentiable loss function with $\nabla \ell(\theta_0) = 0$ and $\nabla^2 \ell(\theta_0)$ positive definite. Then $\ell(\hat{\theta}_1) - \ell(\hat{\theta}_2) = o(\frac{1}{n})$.
\end{theorem}
\begin{proof}
Since both $\hat{\theta}_1$ and $\hat{\theta}_2$ are regular and asymptotically efficient, they admit the same influence function 
$\psi(X)$. Thus, we express each estimator as
\begin{align*}
    \sqrt{n}(\hat{\theta}_1 - \theta_0) &= \frac{1}{\sqrt{n}} \sum_{i=1}^n \psi(X_i) + o(1), \\
    \sqrt{n}(\hat{\theta}_2 - \theta_0) &= \frac{1}{\sqrt{n}} \sum_{i=1}^n \psi(X_i) + o(1).
\end{align*}
Taking the difference of the above, we have
\begin{align*}
    \sqrt{n}\big(\hat{\theta}_1 - \hat{\theta}_2\big) &= o(1) \\
    \hat{\theta}_1 - \hat{\theta}_2 &= o\left(\frac{1}{\sqrt{n}}\right)
\end{align*}

Next, following the delta method from \pref{lem:delta-method}, we perform a second-order Taylor expansion of $\ell(\hat{\theta})$ around $\theta_0$.
\begin{align*}
    \ell(\hat{\theta}) &= \ell(\theta_0) + \nabla \ell(\theta_0)^\top (\hat{\theta} - \theta_0) + \frac{1}{2}(\hat{\theta} - \theta_0)^\top \nabla^2 \ell(\theta_0)(\hat{\theta} - \theta_0) + o\left(\|\hat{\theta} - \theta_0\|^2\right) \\
    &= \ell(\theta_0) + \frac{1}{2}(\hat{\theta} - \theta_0)^\top \nabla^2 \ell(\theta_0)(\hat{\theta} - \theta_0) + o\left(\|\hat{\theta} - \theta_0\|^2\right) \tag{$\nabla \ell(\theta_0)=0$}
\end{align*}

Then consider $\ell(\hat{\theta}_1) - \ell(\hat{\theta}_2)$, and plug in the second order expansion.
\begin{align*}
    \ell(\hat{\theta}_1) - \ell(\hat{\theta}_2) &= \frac{1}{2}(\hat{\theta}_1 - \theta_0)^\top \nabla^2 \ell(\theta_0)(\hat{\theta}_1 - \theta_0) - \frac{1}{2}(\hat{\theta}_2 - \theta_0)^\top \nabla^2 \ell(\theta_0)(\hat{\theta}_2 - \theta_0) + o\left(\frac{1}{n}\right) \\
    &= \frac{1}{2}(\hat{\theta}_1 - \hat{\theta}_2)^\top \nabla^2 \ell(\theta_0)(\hat{\theta}_1 - \hat{\theta}_2) + o\left(\frac{1}{n}\right) \\
    &= o\left(\frac{1}{n}\right) \tag{$\hat{\theta}_1 - \hat{\theta}_2 = o(\frac{1}{\sqrt{n}})$}
\end{align*}
\end{proof}

\begin{lemma} [Delta Method] \label{lem:delta-method}
    Let $Y_n$ be a sequence of random variables that satisfies $\sqrt{n}(Y_n - \theta_0) \rightarrow \mathcal{N}(0, \sigma^2)$ in distribution. For a given function $g$ and $\theta_0$, suppose that $g'(\theta_0) = 0$ and $g''(\theta_0)$ exists and is nonzero. Then %
    \begin{align*}
        n(g(Y_n) - g(\theta_0)) \rightarrow \sigma^2 \frac{g''(\theta_0)}{2} \chi^2_1.
    \end{align*}
\end{lemma}

\subsubsection{\pref{thm:dist-of-gmm-diff}, Difference of GMM Estimators}

\distofgmmdiff*
\begin{proof}
    We know 
\begin{align*}
    \Omega(\theta_0)^{-1/2} \sqrt{\frac{n}{1 - \nu}} (f_n(\theta_0)) \equiv U_n \rightarrow_d U \sim \mathcal{N}(0, I).
\end{align*}
We have, in the limit, for a GMM predictor $T_n$
\begin{align*}
    \sqrt{\frac{n}{1 - \nu}} (T_n - \theta_0) \rightarrow_d (G(\theta_0)^\top W G(\theta_0))^{-1} G(\theta_0)^\top W \Omega(\theta_0)^{1/2} U_n
\end{align*}
Thus, $T_n$ is asymptotically linear with 
\begin{align*}
    \sqrt{\frac{n}{1 - \nu}} (T_n - \theta_0) \rightarrow_d (G^\top W G)^{-1} G^\top W \sqrt{\frac{n}{1 - \nu}} (f_n(\theta_0)) 
\end{align*}
Let $(G^\top W G))^{-1} G^\top W = A^\top$.
This predictor has covariance
\begin{align*}
    A^\top \E\left[\frac{n}{1-\nu} f_n(\theta_0) f_n(\theta_0)^\top \right] A.
\end{align*}

All choices of $A$ must satisfy $A^\top G = I$ for the first order condition. The optimal $A$ also minimizes the above covariance.
\begin{align*}
    A^* &= \arg \min_{A \text{ s.t. } A^\top G = I} A^\top \Omega A
\end{align*}
where $\E[f_n(\theta_0) f_n(\theta_0)^\top] = \Omega$. \\

Consider another $\bar{A}$ that satisfies $\bar{A}^\top G = I$. The optimality condition of $A^*$ implies that $\Omega A^* = G \Lambda$ where $\Lambda$ is the Lagrangian multiplier. Consider
\begin{align*}
    (\bar{A} - A^*)^\top \Omega A &= (\bar{A} - A^*)^\top G \Lambda \\
    &= (\bar{A}^\top G - {A^*}^\top G) \Lambda \\
    &= (I - I) \Lambda \\
    &= 0
\end{align*}

The predictor that only uses the first half of the moment conditions in $f_n$ is asymptotically linear with 
\begin{align*}
    \sqrt{\frac{n}{1 - \nu}} (T_n^1 - \theta_0) \rightarrow_d (G_1^\top W_1 G_1)^{-1} G_1^\top W_1 \sqrt{\frac{n}{1 - \nu}} (f_1(\theta_0)) 
\end{align*}
Let $(G^\top_1 W_1 G_1))^{-1} G^\top_1 W_1 = A_1^\top$. $A_1$ satisfies $A_1^\top G_1 = I$ by the first order condition. Let $\bar{A}^\top = \begin{bmatrix}
        A_1^\top & 0
    \end{bmatrix}$. Then
\begin{align*}
    \bar{A}^\top G &= \begin{bmatrix}
        A_1^\top & 0
    \end{bmatrix} G \\
    &= \begin{bmatrix}
        A_1^\top & 0
    \end{bmatrix} \begin{bmatrix}
        G_1 \\
        G_2
    \end{bmatrix} \\
    &= A_1^\top G_1 \\
    &= I
\end{align*}
Thus, we must have $(\bar{A} - A^*)^\top \Omega A = 0$. Thus, $\bar{A}^\top \Omega {A^*} = {A^*}^\top \Omega {A^*}$. $\bar{A}^\top \Omega {A^*}$ is the covariance between $(T_n^1 - \theta_0)$ and $(T_n - \theta_0)$ (when $T_n$ is the optimal predictor). ${A^*}^\top \Omega {A^*}$ is the variance of $(T_n - \theta_0)$ (when $T_n$ is the optimal predictor). Thus, $\Cov(T_n^1, T_n) = \V[T_n]$, and $\V[T_n^1 - T_n] = \V[T_n^1] - \V[T_n]$.

Thus, in the limit, $T_n^1 - T_n$ is distributed $\mathcal{N}(0, \V[T_n^1] - \V[T_n])$. Here, $\bayes_{q \mid C(S)}$ is exactly $T_n^1$, and $\bayes_{q \mid C(S), A(S)}$ is exactly $T_n$.
\end{proof}

\subsubsection{\pref{thm:pred-bound}, Asymptotic Predictability of a Query}

\predbound*
\begin{proof}
    \begin{align*}
        \gamma(\mathcal{P}, q, \ell, A) &= \sup_{S} \mathop{\E}_{C(S) \sim \mathcal{P}}[ \mathop{\E}_{\substack{x \sim \Pi, A}} [\ell(\bayes_{q \mid C(S)}, q(x)) - \ell(\bayes_{q \mid C(S), A(S)}, q(x))]] \\
        &\leq \sup_{S} \mathop{\E}_{C(S) \sim \mathcal{P}}[ \mathop{\E}_{\substack{x \sim \Pi, A}} [\ell(\gmm_{q \mid C(S)}, q(x)) - \ell(\gmm_{q \mid C(S), A(S)}, q(x))]] + o \left( \frac{1}{n}\right) \tag{using \pref{thm:bayes-and-gmm}} \\
        &\leq \sup_{S} \mathop{\E}_{C(S) \sim \mathcal{P}}[ H \cdot |\gmm_{q \mid C(S)} - \theta^q_0 | \cdot  |\gmm_{q \mid C(S)} - \gmm_{q \mid C(S), A(S)}|]] + o \left( \frac{1}{n}\right) \tag{using \pref{lem:loss-lem}} \\
        &= \sup_{S} \mathop{\E}_{C(S) \sim \mathcal{P}}[ H \cdot |\gmm_{q \mid C(S)} - \theta^q_0 | \cdot  |\gmm_{q \mid C(S)} - \gmm_{q \mid C(S), A(S)}|]]
    \end{align*}
    where asymptotically, we drop the lower-order term.

    Then we bound $|\gmm_{q \mid C(S)} - \theta^q_0 |$ and $|\gmm_{q \mid C(S)} - \gmm_{q \mid C(S), A(S)}|$ using the behavior of the estimators.

    We know $\gmm_{q \mid C(S)} = \theta_1 \sim \mathcal{N}(\theta_0, \frac{\sigma_1^2}{n})$ and $\gmm_{q \mid C(S), A(S)} = \theta_2 \sim \mathcal{N}(\theta_0, \frac{\sigma_2^2}{n})$. Thus, 
\begin{align*}
    \Pr[|\theta_1 - \theta_0| \geq t] &\leq 2 e^{-t^2 n / \sigma_1^2}
\end{align*}
We want this probability to be less than $\delta / 2$.
\begin{align*}
    \frac{\delta}{2} &= 2 e^{-t^2 n / \sigma_1^2} \\
    \frac{\delta}{4} &= e^{-t^2 n / \sigma_1^2} \\
    \log (4 / \delta) &= \frac{t^2 n}{\sigma_1^2} \\
    \frac{\sigma_1^2 \log (4 / \delta)}{n} &= t^2 \\
    \sqrt{\frac{\sigma_1^2 \log (4 / \delta)}{n}} &= t
\end{align*}

We know $\theta_1 - \theta_2 \sim \mathcal{N}(0, \frac{\sigma_1^2}{n} - \frac{\sigma_2^2}{n})$.
\begin{align*}
    \Pr[|\theta_1 - \theta_2| \geq t] &\leq 2 e^{-t^2 n / (\sigma_1^2 - \sigma_2^2)}
\end{align*}
We want this probability to be less than $\delta / 2$.
\begin{align*}
    \frac{\delta}{2} &= 2 e^{-t^2 n / (\sigma_1^2 - \sigma_2^2)} \\
    \frac{\delta}{4} &= e^{-t^2 n / (\sigma_1^2 - \sigma_2^2)} \\
    \log (4 / \delta) &= \frac{t^2 n}{(\sigma_1^2 - \sigma_2^2)} \\
    \frac{(\sigma_1^2 - \sigma_2^2) \log (4 / \delta)}{n} &= t^2 \\
    \sqrt{\frac{(\sigma_1^2 - \sigma_2^2) \log (4 / \delta)}{n}} &= t
\end{align*}

Then consider the difference in the loss.
\begin{align*}
    \ell (\theta_1) - \ell(\theta_2) &\leq H |\theta_1 - \theta_0| \cdot |\theta_1 - \theta_2| \tag{\pref{lem:loss-lem}} \\
    &\leq H \cdot \sqrt{\frac{\sigma_1^2 \log (4 / \delta)}{n}} \cdot \sqrt{\frac{(\sigma_1^2 - \sigma_2^2) \log (4 / \delta)}{n}} \tag{with probability at least $1-\delta$} \\
    &= \frac{H \log(4/\delta) \cdot \sigma_1 \sqrt{\sigma_1^2 - \sigma_2^2}}{n} \tag{with probability at least $1-\delta$}
\end{align*}

Let $X = \ell (\theta_1) - \ell(\theta_2)$. Note that predictability is $\gamma = \E[X]$. Thus, using standard techniques to integrate the tail, we can convert the high-probability bound into a bound on the expectation, yielding the final predictability bound. We know
\begin{align*}
    \Pr[X > \frac{H \log(4 /\delta) \cdot \sigma_1 \sqrt{\sigma_1^2 - \sigma_2^2}}{n}] \leq \delta.
\end{align*}
Let $K = \frac{H \cdot \sigma_1 \sqrt{\sigma_1^2 - \sigma_2^2}}{n}$, then
\begin{align*}
    \Pr[X > t] \leq 4 e^{-t/K}.
\end{align*}

Let $t_0 = K \log (4 )$. Thus,
\begin{align*}
    \E[X] &\leq \int_0^{t_0} 1 dt + \int_{t_0}^{\infty} 4 e^{-t/K} \\
    &= \int_0^{K \log (4 )} 1 dt + \int_{K \log (4 )}^{\infty} 4 e^{-t/K} \\
    &= K \log (4 ) + 4 K e^{-K \log (4 )/K} \\
    &= K \log (4 ) + 4 K e^{- \log (4 )} \\
    &= K \log (4 ) + K \\
    &= K (1 + \log (4 ))
\end{align*}

Thus,
\begin{align*}
    \gamma &= \frac{H(1 + \log (4 )) \cdot \sigma_1 \sqrt{\sigma_1^2 - \sigma_2^2}}{n}.
\end{align*}

In particular, we consider $n, N \rightarrow \infty$, $n / N \rightarrow 0$, so we do not need to consider a finite sample correction when estimating $E_{x \sim \mathrm{Unif}(S \setminus C(S))}[q(x)]$.

Now we compute the exact form of $\sigma_1$ and $\sigma_2$ using $[G^\top \Omega^{-1} G]^{-1}$. For $\gmm_{q \mid C(S)}$, we have $\theta = (p, \lambda)$ and
\begin{align*}
        f_{C(S)}(x, \theta) = \begin{bmatrix}
            w(x)(q(x) - p)
        \end{bmatrix}.
    \end{align*}
    
    We need to compute $[G^\top \Omega^{-1} G]^{-1}$.
    \begin{align*}
        G &= \E \begin{bmatrix}
        \frac{\partial f(x, \theta)}{\partial \theta} \end{bmatrix} \\
        &= \E\begin{bmatrix}
            -w(x)
        \end{bmatrix} \\
        &= -1
    \end{align*}
    \begin{align*}
        \Omega^{-1} &= \E \left[ f(x, \theta) f(x, \theta)^\top \right]^{-1} \\
        &= \E \left[\begin{bmatrix}
            w(x)^2(q(x) - p)^2
        \end{bmatrix} \right]^{-1} \\
        &= \begin{bmatrix}
            \V[w(x) q(x)]
        \end{bmatrix}^{-1} \\
        &= \frac{1}{\V[w(x)q(x)]}
    \end{align*}
    \begin{align*}
        [G(\theta_0)^\top \Omega^{-1}(\theta_0) G(\theta_0)]^{-1} &= \frac{1}{-1 \cdot \frac{1}{\V[w(x)q(x)]} \cdot -1} \\
        &= \V[w(x)q(x)]
    \end{align*}
    
    Then,
    \begin{align*}
        \V[\gmm_{q \mid C(S)}] = \frac{\V[w(x)q(x)]}{n}.
    \end{align*}
    Thus, $\sigma_1^2 = \V[w(x)q(x)]$.

    Next, we establish the variance of optimal predictor $\gmm_{q \mid C(S), A(S)}$, where $\theta = (p, \lambda)$ and $$f_{C(S), A(S)}(x, \theta) = \begin{bmatrix}
        w(x)(q(x) - p) \\
        w(x)g(x,\lambda)  \\
        \Tilde{\lambda} - \lambda
    \end{bmatrix}. $$

    We need to compute $[G^\top \Omega^{-1} G]^{-1}$.
    \begin{align*}
        G &= \E \begin{bmatrix}
        \frac{\partial f(x, \theta)}{\partial \theta} \end{bmatrix} \\
        &= \E \begin{bmatrix}
            -w(x) & 0 \\
            0 & \E[\frac{\partial w(x)g(x, \lambda)}{\partial \lambda}] \\
            0 & -I_d
        \end{bmatrix} \\
        &= \E \begin{bmatrix}
            -1 & 0 \\
            0 & \E[\frac{\partial w(x)g(x, \lambda)}{\partial \lambda}] \\
            0 & -I_d
        \end{bmatrix}
    \end{align*}
    \begin{align*}
        \Omega^{-1} &= \E \left[ f(x, \theta) f(x, \theta)^\top \right]^{-1} \\
        &= \E \left[ \begin{bmatrix}
            w(x)^2(q(x) - p)^2 & w(x)^2(q(x) - p) g(x, \lambda)^\top & w(x)(q(x) - p)(\Tilde{\lambda} - \lambda)^\top \\
            w(x)^2g(x, \lambda)(q(x) - p) & w(x)^2g(x, \lambda) g(x, \lambda)^\top & w(x)g(x, \lambda)(\Tilde{\lambda} - \lambda)^\top \\
            w(x)(\Tilde{\lambda} - \lambda)(q(x) - p) & w(x)(\Tilde{\lambda} - \lambda) g(x, \lambda)^\top & (\Tilde{\lambda} - \lambda)(\Tilde{\lambda} - \lambda)^\top
        \end{bmatrix}\right]^{-1} \\
        &= \begin{bmatrix}
            \V[w(x)q(x)] & \E[w(x)^2(q(x) - p) g(x, \lambda)^\top] & 0 \\
            \E[w(x)^2g(x, \lambda)(q(x) - p)] & \E[w(x)^2g(x, \lambda) g(x, \lambda)^\top] & 0 \\
            0 & 0 & \E[(\Tilde{\lambda} - \lambda)(\Tilde{\lambda} - \lambda)^\top]
        \end{bmatrix}^{-1}
    \end{align*}
    
    Let $U^{-1} = \begin{bmatrix}
        \V[w(x)q(x)] & \E[w(x)^2(q(x) - p) g(x, \lambda)^\top] \\
        \E[w(x)^2 g(x, \lambda)(q(x) - p)] & \E[w(x)^2 g(x, \lambda) g(x, \lambda)^\top]
    \end{bmatrix}^{-1}$.
    
    \begin{align*}
        \Omega^{-1} &= \begin{bmatrix}
            U^{-1}[1,1] & U^{-1}[1,2] & 0 \\
            U^{-1}[2,1] & U^{-1}[2,2] & 0 \\
            0 & 0 & \E[(\Tilde{\lambda} - \lambda)(\Tilde{\lambda} - \lambda)^\top]^{-1}
        \end{bmatrix}
    \end{align*}

    Putting it together,
    \begin{align*}
        G^\top \Omega^{-1} G &= \begin{bmatrix}
            -1 & 0 & 0 \\
            0 & \E[\frac{\partial w(x) g(x, \lambda)}{\partial \lambda}]^\top & -I_d
        \end{bmatrix}\begin{bmatrix}
            U^{-1}[1,1] & U^{-1}[1,2] & 0 \\
            U^{-1}[2,1] & U^{-1}[2,2] & 0 \\
            0 & 0 & \E[(\Tilde{\lambda} - \lambda)(\Tilde{\lambda} - \lambda)^\top]^{-1}
        \end{bmatrix} \begin{bmatrix}
            -1 & 0 \\
            0 & \E[\frac{\partial w(x) g(x, \lambda)}{\partial \lambda}] \\
            0 & -I_d
        \end{bmatrix} \\
        &= \begin{bmatrix}
            -1 & 0 & 0 \\
            0 & \E[\frac{\partial w(x)g(x, \lambda)}{\partial \lambda}]^\top & -I_d
        \end{bmatrix} \begin{bmatrix}
            -U^{-1}[1,1] & U^{-1}[1,2] \E[\frac{\partial w(x)g(x, \lambda)}{\partial \lambda}] \\
            -U^{-1}[2,1] & U^{-1}[2,2] \E[\frac{\partial w(x)g(x, \lambda)}{\partial \lambda}] \\
            0 & -\E[(\Tilde{\lambda} - \lambda)(\Tilde{\lambda} - \lambda)^\top]^{-1}
        \end{bmatrix} \\
        &= \begin{bmatrix}
            U^{-1}[1,1] & -U^{-1}[1,2] \E[\frac{\partial w(x)g(x, \lambda)}{\partial \lambda}] \\
            -\E[\frac{\partial w(x)g(x, \lambda)}{\partial \lambda}]^\top U^{-1}[2,1] & \E[\frac{\partial w(x)g(x, \lambda)}{\partial \lambda}]^\top U^{-1}[2,2] \E[\frac{\partial w(x)g(x, \lambda)}{\partial \lambda}] + \E[(\Tilde{\lambda} - \lambda)(\Tilde{\lambda} - \lambda)^\top]^{-1}
        \end{bmatrix} \\
        &= \begin{bmatrix}
            U^{-1}[1,1] & -U^{-1}[1,2] \E[\frac{\partial w(x)g(x, \lambda)}{\partial \lambda}] \\
            -\E[\frac{\partial w(x)g(x, \lambda)}{\partial \lambda}]^\top U^{-1}[2,1] & \E[\frac{\partial w(x)g(x, \lambda)}{\partial \lambda}]^\top U^{-1}[2,2] \E[\frac{\partial w(x)g(x, \lambda)}{\partial \lambda}] + \V[\Delta]^{-1}
        \end{bmatrix}
    \end{align*}
    
    Using Schur's complement, we can compute the entries of $U^{-1}$. Let $U = \begin{bmatrix}
        A & B \\
        C & D
    \end{bmatrix}$, then
    \begin{align*}
        U^{-1}[1, 1] &= (A - BD^{-1}C)^{-1} \\
        &= \frac{1}{\V[w(x)q(x)] - \E[w(x)^2(q(x) - p) g(x, \lambda)^\top] \E[w(x)^2 g(x, \lambda) g(x, \lambda)^\top]^{-1} \E[w(x)^2(q(x) - p) g(x, \lambda)]} \\
        U^{-1}[1, 2] &= -(A - BD^{-1}C)^{-1} BD^{-1} = -U^{-1}[1, 1] \cdot \E[w(x)^2(q(x) - p) g(x, \lambda)^\top] \E[w(x)^2g(x, \lambda) g(x, \lambda)^\top]^{-1} \\
        U^{-1}[2, 1] &= -D^{-1}C (A - BD^{-1}C)^{-1} = -U^{-1}[1, 1] \cdot \E[w(x)^2g(x, \lambda) g(x, \lambda)^\top]^{-1} \E[w(x)^2(q(x) - p) g(x, \lambda)] \\
        U^{-1}[2, 2] &= D^{-1} + D^{-1} C (A - BD^{-1}C)^{-1} BD^{-1} \\
        &= \E[w(x)^2g(x, \lambda) g(x, \lambda)^\top]^{-1} + (U^{-1}[1, 1] \E[w(x)^2g(x, \lambda) g(x, \lambda)^\top]^{-1} \E[w(x)^2(q(x) - p) g(x, \lambda)] \\
        &\;\;\;\;\; \cdot \E[w(x)^2(q(x) - p) g(x, \lambda)^\top] \E[w(x)^2g(x, \lambda) g(x, \lambda)^\top]^{-1}).
    \end{align*}
    
    The $[1, 1]$ entry of $[G^\top \Omega^{-1} G]^{-1}$ characterizes the variance of $\hat{p}$. For a matrix $M = \begin{bmatrix}
        A & B \\
        C & D
    \end{bmatrix}$,  $M^{-1}[1, 1] = A^{-1} + A^{-1}B (D - CA^{-1}B)^{-1} CA^{-1}$, using Schur's complement.
    \begin{align*}
        [G^\top \Omega^{-1} G]^{-1}[1,1] &= \frac{1}{U^{-1}[1,1]} + \frac{1}{U^{-1}[1,1]^2} U^{-1}[1,2] \E\left[\tfrac{\partial w(x)g(x, \lambda)}{\partial \lambda}\right] \bigg( \E\left[\tfrac{\partial w(x) g(x, \lambda)}{\partial \lambda}\right]^\top U^{-1}[2,2] \E\left[\tfrac{\partial w(x)g(x, \lambda)}{\partial \lambda}\right]  \\
        &\;\;\;\;\;\;+ \V[\Delta]^{-1} - \frac{1}{U^{-1}[1,1]} \E\left[\tfrac{\partial w(x)g(x, \lambda)}{\partial \lambda}\right]^\top U^{-1}[2,1] U^{-1}[1,2] \E\left[\tfrac{\partial w(x)g(x, \lambda)}{\partial \lambda}\right]\bigg )^{-1}  \\
        &\;\;\;\;\;\; \cdot \E\left[\tfrac{\partial w(x)g(x, \lambda)}{\partial \lambda}\right]^\top U^{-1}[2,1] \\
        &= \frac{1}{U^{-1}[1,1]} + \frac{1}{U^{-1}[1,1]^2} \cdot U^{-1}[1,2] \E\left[\tfrac{\partial w(x)g(x, \lambda)}{\partial \lambda}\right] \\
        &\;\;\;\;\;\; \cdot \bigg (\E\left[\tfrac{\partial w(x)g(x, \lambda)}{\partial \lambda}\right]^\top \E[w(x)^2 g(x, \lambda) g(x, \lambda)^\top]^{-1} \E\left[\tfrac{\partial w(x)g(x, \lambda)}{\partial \lambda}\right] + \V[\Delta]^{-1} \bigg )^{-1}  \\
        &\;\;\;\;\;\;\cdot \E\left[\tfrac{\partial w(x)g(x, \lambda)}{\partial \lambda}\right]^\top U^{-1}[2,1] \\
        &= \frac{1}{U^{-1}[1,1]} + \E[w(x)^2(q(x) - p) g(x, \lambda)^\top] \E[w(x)^2g(x, \lambda) g(x, \lambda)^\top]^{-1} \E\left[\tfrac{\partial w(x)g(x, \lambda)}{\partial \lambda}\right] \\
        &\;\;\;\;\;\;\cdot \bigg (\E\left[\tfrac{\partial w(x)g(x, \lambda)}{\partial \lambda}\right]^\top \E[w(x)^2 g(x, \lambda) g(x, \lambda)^\top]^{-1} \E\left[\tfrac{\partial g(x, \lambda)}{\partial \lambda}\right] + \V[\Delta]^{-1} \bigg )^{-1} \\
        &\;\;\;\;\;\;\cdot \E\left[\tfrac{\partial w(x)g(x, \lambda)}{\partial \lambda}\right]^\top \E[w(x)^2 g(x, \lambda) g(x, \lambda)^\top]^{-1} \E[w(x)^2 (q(x) - p) g(x, \lambda)] \\
        &= \frac{1}{U^{-1}[1,1]} + \V[w(x)q(x)] \cdot \Cw \\
        &= \V[w(x)q(x)] \\
        &\;\;\;\;\;\;- \E[w(x)^2(q(x) - p) g(x, \lambda)^\top] \E[w(x)^2 g(x, \lambda) g(x, \lambda)^\top]^{-1} \E[w(x)^2 (q(x) - p) g(x, \lambda)] \\
        &\;\;\;\;\;\;+ \V[w(x)q(x)] \cdot \Cw
    \end{align*}
    where
    \begin{align*}
        \Cw &= \frac{1}{\V[w(x)q(x)]} \cdot \E[w(x)^2(q(x) - p) g(x, \lambda)^\top] \E[w(x)^2g(x, \lambda) g(x, \lambda)^\top]^{-1} \E\left[\tfrac{\partial w(x)g(x, \lambda)}{\partial \lambda}\right] \\
        &\;\;\;\;\;\; \cdot \bigg (\E\left[\tfrac{\partial w(x)g(x, \lambda)}{\partial \lambda}\right]^\top \E[w(x)^2g(x, \lambda) g(x, \lambda)^\top]^{-1} \E\left[\tfrac{\partial w(x)g(x, \lambda)}{\partial \lambda}\right] + \V[\Delta]^{-1} \bigg )^{-1} \\
        &\;\;\;\;\;\; \cdot \E\left[\tfrac{\partial w(x)g(x, \lambda)}{\partial \lambda}\right]^\top \E[w(x)^2g(x, \lambda) g(x, \lambda)^\top]^{-1} \E[w(x)^2(q(x) - p) g(x, \lambda)].
    \end{align*}

    Recall from \pref{def:can-corr},
    \begin{align*}
        \cc (w(x)q(x), w(x)g(x, \lambda)) &= \max_{c \in \mathbb{R}^d} \rho (w(x)q(x), c^\top w(x)g(x, \lambda)) \\
        &= \max_{c \in \mathbb{R}^d} \frac{\Cov(w(x)q(x), c^\top w(x)g(x, \lambda))}{\sqrt{\V[w(x)q(x)] \cdot \V[c^\top w(x)g(x, \lambda)]}} \\
        &= \max_{c \in \mathbb{R}^d} \frac{c^\top \Cov(w(x)q(x), w(x)g(x, \lambda))}{\sqrt{\V[w(x)q(x)] \cdot c^\top \Cov(w(x)g(x, \lambda), w(x)g(x, \lambda))c}}.
    \end{align*}

    Squaring the objective,
    \begin{align*}
        \cc (w(x)q(x), w(x)g(x, \lambda))^2 = \max_{c \in \mathbb{R}^d} \frac{c^\top \Cov(w(x)g(x, \lambda), w(x)q(x)) \Cov(w(x)g(x, \lambda), w(x)q(x))^\top c}{\V[q(x)] \cdot c^\top \Cov(w(x)g(x, \lambda), w(x)g(x, \lambda))c}
    \end{align*}
    This is a generalized Rayleigh quotient with solution
    \begin{align*}
        \cc (w(x)q(x), w(x)&g(x, \lambda))^2 \\
        &= \frac{\Cov(w(x)g(x, \lambda), w(x)q(x))^\top \Cov(w(x)g(x, \lambda), w(x)g(x, \lambda))^{-1} \Cov(w(x)g(x, \lambda), w(x)q(x))}{\V[w(x)q(x)]} \\
        &= \frac{\E[w(x)^2(q(x) - p) g(x, \lambda)^\top] \E[w(x)^2g(x, \lambda) g(x, \lambda)^\top]^{-1} \E[w(x)^2(q(x) - p) g(x, \lambda)]}{\V[w(x)q(x)]}.
    \end{align*}

    Thus,
    \begin{align*}
        [G^\top \Omega^{-1} G]^{-1}[1,1] &= \V[w(x)q(x)] - \V[w(x)q(x)] \cdot \cc (w(x)q(x), w(x)g(x, \lambda))^2 + \V[w(x)q(x)] \cdot \Cw \\
        &= \V[w(x)q(x)] \cdot (1 - \cc (q(x), g(x, \lambda))^2 + \Cw)
    \end{align*}

    Then,
    \begin{align*}
        \V[\gmm_{q \mid C(S), A(S)}] = \frac{\V[w(x)q(x)](1 - \cc (w(x)q(x), w(x)g(x, \lambda))^2 + \Cw)}{n}.
    \end{align*}
    Thus, $\sigma_2^2 = \V[w(x)q(x)](1 - \cc (w(x)q(x), w(x)g(x, \lambda))^2 + \Cw)$.
    \begin{align*}
        \sigma_1 \sqrt{\sigma_1^2 - \sigma_2^2} &= \sqrt{\V[w(x)q(x)]} \cdot \sqrt{\V[w(x)q(x)] - \V[w(x)q(x)](1 - \cc (w(x)q(x), w(x)g(x, \lambda))^2 - \Cw)} \\
        &= \sqrt{\V[w(x)q(x)]} \cdot \sqrt{\V[w(x)q(x)](\cc (w(x)q(x), w(x)g(x, \lambda))^2 - \Cw)} \\
        &= \V[w(x)q(x)] \cdot \sqrt{\cc (w(x)q(x), w(x)g(x, \lambda))^2 - \Cw}
    \end{align*}

    Plugging in the above into the predictability bound, 
    \begin{align*}
        \gamma \leq \frac{H(1 + \log(4))\cdot \V[w(x)q(x)] \cdot \sqrt{\cc (w(x)q(x), w(x)g(x, \lambda))^2 - \Cw}}{n}
    \end{align*}
\end{proof}
\pref{thm:pred-bound-d-prime} follows analogously by plugging in $w'(x) (q(x)-p)$ in place of $w(x) (q(x)-p)$ is the proof above.

\begin{restatable}{theorem}{predmethod} \label{thm:pred-method}
    Let $S$ be a dataset of size $N$ and let $C(S)$ be random sample without replacement of size $n = \alpha N$, $0 < \alpha < 1$. For any dataset $S$, given moment conditions $f_{K(S)}(x, \theta)$, the optimal predictor $\gmm_{q \mid K(S)}$ of $p_{S \setminus C(S)} = \E_{x \sim \mathrm{Unif}(S \setminus C(S))}[q(x)]$ 
    has asymptotic distribution
    \begin{align*}
        \sqrt{(1-\alpha) \alpha N} (\gmm_{q \mid K(S)} - \theta_0) \sim \mathcal{N}(0, [ G(\theta_0)^\top \Omega^{-1}(\theta_0) G(\theta_0)]^{-1}\!).
    \end{align*}
\end{restatable}
\begin{proof}
    The optimal predictor $\gmm$ of $p = \E_{\mathrm{Unif}(S)}[q(x)]$ is 
    \begin{align*}
        \sqrt{\frac{\alpha N}{1-\alpha}} (\gmm - \theta_0) \sim \mathcal{N}(0, [ G(\theta_0)^\top \Omega^{-1}(\theta_0) G(\theta_0)]^{-1}),
    \end{align*}
    by plugging in $n = \alpha N$ and $n$ appropriately into \pref{thm:gmm-efficiency-swor}.

    For any fixed sample $C(S)$, an estimate for $p$ implies an estimate for $p_{S \setminus C(S)}$ using a simple translation and scaling, $p = (1-\alpha) p_{S \setminus C(S)} + \alpha \E_{x \sim \mathrm{Unif}(C(S))}[q(x)]$. $\V[\hat{p}_{S \setminus C(S)}] = \frac{1}{(1-\alpha)^2} \cdot \V[\hat{p}]$. Thus, the optimal predictor $\gmm_{q \mid K(S)}$ of $p_{S \setminus C(S)} = \E_{\mathrm{Unif}(S \setminus C(S))}[q(x)]$ is
    \begin{align*}
        \sqrt{(1-\alpha) \alpha N} (\gmm_{q \mid K(S)} - \theta_0) \sim \mathcal{N}(0, [ G(\theta_0)^\top \Omega^{-1}(\theta_0) G(\theta_0)]^{-1}).
    \end{align*}
\end{proof}

\predboundsigmaderived*
\begin{proof}
    We would like to find the efficient GMM estimator for $p_{S \setminus C(S)} = \E_{x \sim \mathrm{Unif}(S \setminus C(S))}[q(x_i)]$. Typical GMM for sampling without replacement finds the efficient estimator for $p_{S} = \E_{x \sim \mathrm{Unif}(S)}[q(x_i)]$. We find that the efficient GMM estimator for $p_S$ implies an efficient GMM estimator for $p_{S \setminus C(S)}$. We use \pref{thm:pred-method} to find the efficient estimator $\gmm_{q \mid K(S)}$ for $p_{S \setminus C(S)}$. \pref{thm:pred-method} essentially performs a finite sample correction for estimation.

    From \pref{thm:pred-method}, we have
    \begin{align*}
        \V[\gmm_{q \mid C(S)}] = \frac{\sigma_1}{(1-\alpha)\alpha N}
    \end{align*}
    and
    \begin{align*}
        \V[\gmm_{q \mid C(S), A(S)}] = \frac{\sigma_2}{(1-\alpha)\alpha N}.
    \end{align*}

    Following the proof of \pref{thm:pred-bound}, we have $\sigma_1^2 = \V[w(x)q(x)]$,  $\sigma_2^2 = \V[w(x)q(x)](1 - \cc (w(x)q(x), w(x)g(x, \lambda))^2 + \Cw)$, and
    \begin{align*}
        \sigma_1 \sqrt{\sigma_1^2 - \sigma_2^2} &= \V[w(x)q(x)] \cdot \sqrt{\cc (w(x)q(x), w(x)g(x, \lambda))^2 - \Cw}
    \end{align*}

    We plug in $w(x)=1$ to get
    \begin{align*}
        \sigma_1 \sqrt{\sigma_1^2 - \sigma_2^2} &= \V[q(x)] \cdot \sqrt{\cc (q(x), g(x, \lambda))^2 - \Cwone}
    \end{align*}

    Thus, we have 
    \begin{align*}
        \gamma \leq \frac{H(1 + \log(4))\cdot \V[q(x)] \cdot \sqrt{\cc (q(x), g(x, \lambda))^2 - \Cwone}}{(1-\alpha) \alpha N}
    \end{align*}

\end{proof}

\subsection{Proofs from \pref{sec:pred-of-algs}}

\subsubsection{\pref{thm:noisy-stat}, Private Dataset Statistics}

\noisystat*
\begin{proof}
    From \pref{corr:pred-bound-sigma-derived}, we know
    \begin{align*}
        \Cwone &= \frac{1}{\V[q(x)]} \cdot \E[(q(x) - p) g(x, \lambda)^\top] \E[g(x, \lambda) g(x, \lambda)^\top]^{-1} \E\left[\tfrac{\partial g(x, \lambda)}{\partial \lambda}\right] \\
        &\;\;\;\;\;\; \cdot \bigg (\E\left[\tfrac{\partial g(x, \lambda)}{\partial \lambda}\right]^\top \E[g(x, \lambda) g(x, \lambda)^\top]^{-1} \E\left[\tfrac{\partial g(x, \lambda)}{\partial \lambda}\right] + \V[\Delta]^{-1} \bigg )^{-1} \\
        &\;\;\;\;\;\; \cdot \E\left[\tfrac{\partial g(x, \lambda)}{\partial \lambda}\right]^\top \E[g(x, \lambda) g(x, \lambda)^\top]^{-1} \E[(q(x) - p) g(x, \lambda)].
    \end{align*}

    When $g(x, \lambda) = \Gamma (x) - \lambda$, $\E\left[\frac{\partial g(x, \lambda)}{\partial \lambda}\right] = 1$. Also notice that $\E[g(x, \lambda) g(x, \lambda)^\top] = \E[(\Gamma (x) - \lambda)(\Gamma (x) - \lambda)^\top] = \V[\Gamma (x)]$. Thus,
    \begin{align*}
        \Cwone &= \frac{1}{\V[q(x)]} \cdot \E[(q(x) - p) g(x, \lambda)^\top] \E[g(x, \lambda) g(x, \lambda)^\top]^{-1} \left(\E[g(x, \lambda) g(x, \lambda)^\top]^{-1} + \V[\Delta]^{-1} \right)^{-1} \\
        &\;\;\;\;\;\; \cdot \E[g(x, \lambda) g(x, \lambda)^\top]^{-1} \E[(q(x) - p) g(x, \lambda)] \\
        &= \frac{1}{\V[q(x)]} \cdot \E[(q(x) - p) g(x, \lambda)^\top] \V[\Gamma(x)]^{-1} \left(\V[\Gamma(x)]^{-1} + \V[\Delta]^{-1} \right)^{-1} \V[\Gamma(x)]^{-1} \\
        &\;\;\;\;\;\; \cdot \E[(q(x) - p) g(x, \lambda)] \\
        &= \frac{1}{\V[q(x)]} \cdot \E[(q(x) - p) g(x, \lambda)^\top] \cdot \frac{1}{\V[\Gamma(x)]} \left(\frac{1}{\V[\Gamma(x)]} + \frac{1}{\V[\Delta]} \right)^{-1} \frac{1}{\V[\Gamma(x)]} \\
        &\;\;\;\;\;\;\cdot \E[(q(x) - p) g(x, \lambda)] \\
        &= \frac{1}{\V[q(x)]} \cdot \E[(q(x) - p) g(x, \lambda)^\top] \cdot \frac{1}{\V[\Gamma(x)]} \left(\frac{\V[\Gamma(x)] + \V[\Delta]}{\V[\Gamma(x)]\V[\Delta]} \right)^{-1} \frac{1}{\V[\Gamma(x)]} \\
        &\;\;\;\;\;\;\cdot \E[(q(x) - p) g(x, \lambda)] \\
        &= \frac{1}{\V[q(x)]} \cdot \E[(q(x) - p) g(x, \lambda)^\top] \cdot \frac{1}{\V[\Gamma(x)]} \left(\frac{\V[\Gamma(x)]\V[\Delta]}{\V[\Gamma(x)] + \V[\Delta]} \right) \frac{1}{\V[\Gamma(x)]} \cdot \E[(q(x) - p) g(x, \lambda)] \\
        &= \frac{1}{\V[q(x)]} \cdot \E[(q(x) - p) g(x, \lambda)^\top] \cdot \frac{1}{\V[\Gamma(x)]} \left(\frac{\V[\Delta]}{\V[\Gamma(x)] + \V[\Delta]} \right) \E[(q(x) - p) g(x, \lambda)] \\
        &= \frac{\Cov (q(x)-p, g(x, \lambda))}{\V[q(x)]\V[\Gamma(x)]} \cdot \frac{\V[\Delta]}{\V[\Gamma(x)] + \V[\Delta]} \\
        &= \frac{\Cov (q(x), \Gamma(x))}{\V[q(x)]\V[\Gamma(x)]} \cdot \frac{\V[\Delta]}{\V[\Gamma(x)] + \V[\Delta]} \\
        &= \rho (q(x), \Gamma(x))^2 \cdot \frac{\V[\Delta]}{\V[\Gamma(x)] + \V[\Delta]} \\
        &= \rho (q(x), \Gamma(x))^2 \cdot \left( 1- \frac{\V[\Gamma(x)]}{\V[\Gamma(x)] + \V[\Delta]} \right)
    \end{align*}

    Plugging the above into \pref{corr:pred-bound-sigma-derived} along with $\cc (q(x), g(x, \lambda))^2 = \rho (q(x), \Gamma(x))^2$,
    \begin{align*}
        \gamma &= \frac{H(1 + \log(4))\cdot \V[q(x)] \cdot \sqrt{\rho (q(x), \Gamma(x))^2 - \rho (q(x), \Gamma(x))^2 \cdot \left( 1- \frac{\V[\Gamma(x)]}{\V[\Gamma(x)] + \V[\Delta]} \right)}}{(1-\alpha) \alpha N} \\
        &= \frac{H(1 + \log(4))\cdot \V[q(x)] \cdot \sqrt{\rho (q(x), \Gamma(x))^2 \cdot \frac{\V[\Gamma(x)]}{\V[\Gamma(x)] + \V[\Delta]}}}{(1-\alpha) \alpha N} \\
        &= \frac{H(1 + \log(4))\cdot \V[q(x)] \cdot \rho (q(x), \Gamma(x)) \cdot \sqrt{\frac{\V[\Gamma(x)]}{\V[\Gamma(x)] + \V[\Delta]}}}{(1-\alpha) \alpha N}
    \end{align*}
\end{proof}

\subsubsection{\pref{thm:noisy-erm}, Private ERM}

\noisyerm*
\begin{proof}
    The result directly follows by plugging in $g(z, \wE) = \grad \ell (z, \wE) = \grad_z$ into \pref{corr:pred-bound-sigma-derived} with $\E \left[ \frac{\partial g(x, \lambda)}{\partial \lambda}\right] = \grad^2 \ell (x, \wE) = H_z$.

    In particular, $\Cwone = \frac{1}{\V[q(z)]} \cdot \E[(q(x) - p) \grad_z^\top] \allowbreak \E[\grad_z \grad_z^\top]^{-1} \allowbreak \E[H_z] \allowbreak (\E[H_z] \allowbreak \E[\grad_z \grad_z^\top]^{-1} \allowbreak \E[H_z] + \allowbreak \V[\Delta]^{-1})^{-1} \allowbreak \E[H_z] \allowbreak \E[\grad_z \grad_z^\top]^{-1} \allowbreak \E[(q(x) - p) \grad_z]$

\end{proof}

\subsubsection{\pref{lem:targeted-noise}, Calibrated Noise Scheme}

\targetednoise*
\begin{proof}
    Let $A = \E[w(z) H_z]^\top \E[w(z)^2 \grad_z \grad_z^\top]^{-1} \E[w(z) H_z]$. Then 
    \begin{align*}
        \Cw &= \E[w(z)^2 (q(z) - p) \grad_z^\top] \E[w(z)^2 \grad_z \grad_z^\top]^{-1} \E[w(z) H_z] \left(A + \E[(\tilde{w}-\wE)(\tilde{w}-\wE)^\top]^{-1}\right)^{-1} \\
        &\;\;\;\;\;\; \cdot \E[w(z) H_z]^\top \E[w(z)^2 \grad_z \grad_z^\top]^{-1} \E[w(z)^2(q(z) - p) \grad_z]
    \end{align*}
    and
    \begin{align*}
        \E[(\tilde{w}-\wE)(\tilde{w}-\wE)^\top]^{-1} = \frac{1}{\sigma^2} \cdot A
    \end{align*}

    Plugging in $\E[(\tilde{w}-\wE)(\tilde{w}-\wE)^\top]^{-1}$, we have
    \begin{align*}
        \Cwone &= \frac{1}{\V[w(z)q(z)]} \cdot \E[w(z)^2(q(z) - p) \grad_z^\top] \E[w(z)^2 \grad_z \grad_z^\top]^{-1} \E[w(z) H_z] \left(A + \frac{1}{\sigma^2} \cdot A\right )^{-1} \\
        &\;\;\;\;\;\;\cdot \E[w(z) H_z]^\top \E[w(z)^2 \grad_z \grad_z^\top]^{-1} \E[w(z)^2 (q(z) - p) \grad_z] \\
        &= \frac{1}{\V[w(z)q(z)]} \cdot \E[w(z)^2(q(z) - p) \grad_z^\top] \E[w(z)^2 \grad_z \grad_z^\top]^{-1} \E[w(z) H_z] \left( \left(1 + \frac{1}{\sigma^2}\right)A \right)^{-1} \\
        &\;\;\;\;\;\;\cdot \E[w(z) H_z]^\top \E[w(z)^2 \grad_z \grad_z^\top]^{-1} \E[w(z)^2 (q(z) - p) \grad_z] \\
        &= \frac{1}{\V[w(z)q(z)]} \cdot \frac{1}{1 + \frac{1}{\sigma^2}} \cdot \E[w(z)^2(q(z) - p) \grad_z^\top] \E[w(z)^2 \grad_z \grad_z^\top]^{-1} \E[w(z) H_z] A^{-1} \\
        &\;\;\;\;\;\;\cdot \E[w(z) H_z]^\top \E[w(z)^2 \grad_z \grad_z^\top]^{-1} \E[w(z)^2 (q(z) - p) \grad_z]  \\
        &= \frac{1}{\V[w(z)q(z)]} \cdot \frac{\sigma^2}{\sigma^2 +1} \cdot \E[w(z)^2(q(z) - p) \grad_z^\top] \E[w(z)^2 \grad_z \grad_z^\top]^{-1} \E[w(z) H_z] \E[w(z) H_z]^{-1}   \\
        &\;\;\;\;\;\; \cdot \E[w(z)^2 \grad_z \grad_z^\top] [\E[w(z)H_z]^\top]^{-1} \E[w(z)H_z]^\top \E[w(z)^2 \grad_z \grad_z^\top]^{-1} \E[w(z)^2 (q(z) - p) \grad_z] \\
        &= \frac{1}{\V[w(z)q(z)]} \cdot \frac{\sigma^2}{\sigma^2 +1} \cdot \E[w(z)^2(q(z) - p) \grad_z^\top] \E[[w(z)^2\grad_z \grad_z^\top]^{-1} \E[[w(z)^2(q(z) - p) \grad_z] \\
        &= \frac{\sigma^2}{\sigma^2 +1} \cdot \cc (q', g')^2
    \end{align*}

    Plugging in $\Cwone$ to the predictability bound from \pref{thm:pred-bound},
    \begin{align*}
        \gamma &= \frac{H (1 + \log(4)) \cdot \V[q']}{(1-\alpha) \alpha N} \cdot \sqrt{(\cc (q', g')^2 - \Cwone)} \\
        &= \frac{H (1 + \log(4)) \cdot \V[q']}{(1-\alpha)\alpha N} \cdot \sqrt{\cc (q', g')^2 - \frac{\sigma^2}{\sigma^2 +1} \cdot \cc (q', g')^2} \\
        &= \frac{H (1 + \log(4)) \cdot \V[q']}{(1-\alpha) \alpha N} \cdot \sqrt{\frac{\cc (q', g')^2}{\sigma^2 +1}}.
    \end{align*}
\end{proof}

\accuracychangetargeted*
\begin{proof}
    \begin{align*}
        \E \left[ \frac{1}{N} \sum_{i=1}^N (\tilde{w}^\top x_i - \wE^\top x_i)^2 \right] &= \E \left[ \frac{1}{N} \sum_{i=1}^N (\Delta^\top x_i)^2 \right] \\
        &= \E \left[ \frac{1}{N} \sum_{i=1}^N \Delta^\top x_i x_i^\top \Delta \right] \\
        &= \E \left[\Delta^\top \E[xx^\top] \Delta \right] \\
        &= \E\left[\mathrm{tr}(\Delta^\top \E[xx^\top] \Delta)\right] \\
        &= \E\left[\mathrm{tr}( \E[xx^\top] \Delta \Delta^\top)\right] \\
        &= \mathrm{tr}( \E[xx^\top] \cdot \E[\Delta \Delta^\top] ) \\
        &= \mathrm{tr}( \E[xx^\top] \cdot \sigma^2 \E[xx^\top]^{-1} \E[(\wE^\top x - y)^2 xx^\top] \E[xx^\top]^{-1}) \\
        &= \sigma^2 \cdot \mathrm{tr}(\E[(\wE^\top x - y)^2 xx^\top] \E[xx^\top]^{-1}) \\
        &= \sigma^2 \cdot \E \left[ \mathrm{tr}((\wE^\top x - y)^2 xx^\top \E[xx^\top]^{-1})\right] \\
        &= \sigma^2 \cdot \E \left[ (\wE^\top x - y)^2 \cdot \mathrm{tr}( xx^\top \E[xx^\top]^{-1})\right] \\
        &= \sigma^2 \cdot \E \left[ (\wE^\top x - y)^2 \cdot \mathrm{tr}(x^\top \E[xx^\top]^{-1}x)\right] \\
        &= \sigma^2 \cdot \E \left[ (\wE^\top x - y)^2 \cdot x^\top \E[xx^\top]^{-1}x\right]
    \end{align*}

    From here, we establish two upper bounds. First, we establish an upper bound based on the average empirical loss.
    \begin{align*}
        \E \left[ \frac{1}{N} \sum_{i=1}^N (\tilde{w}^\top x_i - \wE^\top x_i)^2 \right] &= \sigma^2 \cdot \E \left[ (\wE^\top x - y)^2 \cdot x^\top \E[xx^\top]^{-1}x\right] \\
        &\leq \sigma^2 \cdot \E \left[ (\wE^\top x - y)^2 \cdot \frac{\|x\|^2}{\lambda_{\min}(\E[xx^\top])}\right] \\
        &\leq \sigma^2 \cdot \E \left[ (\wE^\top x - y)^2 \cdot \frac{\max_i \|x_i\|^2}{\lambda_{\min}(\E[xx^\top])}\right] \\
        &\leq \frac{\sigma^2 \cdot (\max_i \|x_i\|^2)}{\lambda_{\min}(\E[xx^\top])} \cdot \E[(\wE^\top x - y)^2] \\
        &= \frac{\sigma^2 \cdot (\max_i \|x_i\|^2)}{\lambda_{\min}(\E[xx^\top])} \cdot \left(\frac{1}{N} \sum_{i=1}^N (\wE^\top x_i - y_i)^2\right)
    \end{align*}
    
    Next, we establish an upper bound based on the maximum empirical loss.
    \begin{align*}
        \E \left[ \frac{1}{N} \sum_{i=1}^N (\tilde{w}^\top x_i - \hat{w}^\top x_i)^2 \right] &= \sigma^2 \cdot \E \left[ (\hat{w}^\top x - y)^2 \cdot x^\top \E[xx^\top]^{-1}x\right] \\
        &\leq \left(\max_{i} (\hat{w}^\top x_i - y_i)^2\right) \cdot \sigma^2 \cdot \E \left[x^\top \E[xx^\top]^{-1}x\right] \\
        &= \left(\max_{i} (\hat{w}^\top x_i - y_i)^2\right) \cdot \sigma^2 \cdot \E \left[\mathrm{tr}(\E[xx^\top]^{-1}xx^\top)\right] \\
        &= \left(\max_{i} (\hat{w}^\top x_i - y_i)^2\right) \cdot \sigma^2 \cdot \mathrm{tr}(\E[xx^\top]^{-1} \E[xx^\top]) \\
        &= \left(\max_{i} (\hat{w}^\top x_i - y_i)^2\right) \cdot \sigma^2 \cdot \mathrm{tr}(I_d) \\
        &= \left(\max_{i} (\hat{w}^\top x_i - y_i)^2\right) \cdot \sigma^2 \cdot d
    \end{align*}    
\end{proof}

\accuracychangeuniform*
\begin{proof}
    \begin{align*}
        \E \left[ \frac{1}{N} \sum_{i=1}^N (\tilde{w}^\top x_i - \wE^\top x_i)^2 \right] &= \E \left[ \frac{1}{N} \sum_{i=1}^N (\Delta^\top x_i)^2 \right] \\
        &= \E \left[ \frac{1}{N} \sum_{i=1}^N \Delta^\top x_i x_i^\top \Delta \right] \\
        &= \E \left[ \frac{1}{N} \sum_{i=1}^N \mathrm{tr}(\Delta^\top x_i x_i^\top \Delta) \right] \\
        &= \E \left[ \frac{1}{N} \sum_{i=1}^N \mathrm{tr}(\Delta \Delta^\top x_i x_i^\top) \right] \\
        &= \frac{1}{N} \sum_{i=1}^N \mathrm{tr}(\E[\Delta \Delta^\top] \cdot x_i x_i^\top) \\
        &= \frac{1}{N} \sum_{i=1}^N \mathrm{tr}(\sigma^2 I_d \cdot x_i x_i^\top) \\
        &= \sigma^2 \cdot \left(\frac{1}{N} \sum_{i=1}^N \mathrm{tr}(x_i x_i^\top)\right) \\
        &= \sigma^2 \cdot \left(\frac{1}{N} \sum_{i=1}^N \mathrm{tr}(x_i^\top x_i)\right) \\
        &= \sigma^2 \cdot \left(\frac{1}{N} \sum_{i=1}^N \|x_i\|^2\right)
    \end{align*}  
\end{proof}

\subsection{Proofs from \pref{sec:query-families}}

\subsubsection{\pref{thm:pred-finite}, Finite Family of Queries and Processes}
\predfinite*
\begin{proof}
    Let $n = (1-\alpha) \alpha N$. For a weight $w \in \mathcal{W}_{\Pfam}$ and query $q \in \mathcal{Q}$, let $X_{w,q} = \ell (\gmm_{q \mid C(S)}) - \ell(\gmm_{q \mid C(S), A(S)})$. From \pref{thm:pred-bound}, we know
    \begin{align*}
        \Pr[X_{w,q} > \frac{H \log(4  /\delta) \cdot \sigma_{w,q,1} \sqrt{\sigma_{w,q,1}^2 - \sigma_{w,q,2}^2}}{n}] \leq \delta.
    \end{align*}
    Furthermore, every $X_{w,q}$ satisfies
    \begin{align*}
        \Pr[X_{w,q} > \frac{H \log(4  /\delta) \cdot \sup_{q \in \mathcal{Q}} \sup_{w \in \mathcal{W}_{\Pfam}} \sigma_{w,q,1} \sqrt{\sigma_{w,q,1}^2 - \sigma_{w,q,2}^2}}{n}] \leq \delta.
    \end{align*}
    Let $Y = \sup_{q \in \mathcal{Q}} \sup_{w \in \mathcal{W}_{\Pfam}} X_{w,q}$. Then, using a union bound, with probability at least $1-\delta$,
    \begin{align*}
        \Pr[Y > \frac{H \log(4 |\mathcal{Q}| |\mathcal{W}_{\Pfam}| /\delta) \cdot \sup_{q \in \mathcal{Q}} \sup_{w \in \mathcal{W}_{\Pfam}} \sigma_{w,q,1} \sqrt{\sigma_{w,q,1}^2 - \sigma_{w,q,2}^2}}{n}] \leq \delta.
    \end{align*}
    Now, consider $\E[Y]$. Thus, using standard techniques to integrate the tail, we can convert the high-probability bound into a bound on the expectation, yielding the final predictability bound. Following the proof from \pref{thm:pred-bound}, we can conclude that
    \begin{align*}
        \E[Y] \leq \frac{H (1+\log(4 |\mathcal{Q}| |\mathcal{W}_{\Pfam}| )) \cdot \sup_{q \in \mathcal{Q}} \sigma_{w,q,1} \sqrt{\sigma_{w,q,1}^2 - \sigma_{w,q,2}^2}}{n}
    \end{align*}
    where $n = (1-\alpha) \alpha N$.

    From \pref{thm:pred-bound}, we have that 
    \begin{align*}
        \sigma_{w,q,1}^2 &= \V[w(x)q(x)], \\
        \sigma_{w,q,1}^2 - \sigma_{w,q,2}^2 &= \V[w(x)q(x)] \cdot \sqrt{\cc (w(x)q(x), w(x)g(x, \lambda))^2 - \Cwone}
    \end{align*}

    Thus, we can conclude that
    \begin{align*}
        \E[Y] &\leq \frac{H (1+\log(4 |\mathcal{Q}| |\mathcal{W}_{\Pfam}| )) \cdot \sup_{q \in \mathcal{Q}} \V[q(x)] \cdot \sqrt{\cc (w(x)q(x), w(x)g(x, \lambda))^2 - \Cwone}}{n}.
    \end{align*}

    Lastly, note that 
    \begin{align*}
        \sup_{\mathcal{P} \in \Pfam} \gamma (\mathcal{P}, \mathcal{Q}, \ell, A) \leq \E[Y]
    \end{align*}
    by pulling out the supremum using Jensen's inequality.
\end{proof}

\subsubsection{\pref{thm:pred-euclid-ball}, Linear Queries}

\begin{definition} [Gaussian Complexity \citep{bartlett2003rademacherandgaussiancomplexities}] \label{def:gaussian-complexity}
    Let $\xi \sim \mathcal{N}(0, I_d)$. The Gaussian complexity of a set $K$ is defined as $\mathcal{G}(K) = \E \left[ \sup_{u \in K} \langle u, \xi \rangle \right]$.
\end{definition}
\begin{restatable}{theorem}{predlinear} \label{thm:pred-linear}
    Let query family $\mathcal{Q}_{\mathrm{lin}} = \{q(x) = u^\top x \mid{} u \in \mathcal{U}_{\mathrm{lin}} \subseteq \mathbb{R}^d\}$. Let diameter $\diam_{\Sigma}(K) = \sup_{u \in K} u^\top \Sigma u$. Let $\mathcal{U}_{\mathrm{lin}, \Sigma} = \{ \Sigma^{1/2} u \mid{} u \in \mathcal{U}_{\mathrm{lin}}\}$. Using moment conditions
    \begin{equation}
        f_{C(S)}(x, \theta) = [x - \mu ] \text{ and } f_{C(S), A(S)}(x, \theta) = [x - \mu, g(x, \lambda), \tilde{\lambda} - \lambda]^\top,
    \end{equation}
    we have %
    \begin{align*}
        \sqrt{(1-\alpha) \alpha N} (\gmm_{x \mid C(S)} - \mu) \sim \mathcal{N}(0, \Sigma_1) \text{ and } \sqrt{(1-\alpha) \alpha N} (\gmm_{x \mid C(S), A(S)} - \mu) \sim \mathcal{N}(0, \Sigma_2)
    \end{align*}
    where %
    \begin{align*}
        \Sigma_1 &= \V[x], \\
        \Sigma_2 &= \V[x] - \E[(x - \mu) g(x, \lambda)^\top] \E[g(x, \lambda) g(x, \lambda)^\top]^{-1} \E[g(x, \lambda)(x - \mu)^\top] - \Cx.
    \end{align*}
    and $\Cx = \E[(x - \mu) g(x, \lambda)^\top] \allowbreak \E[g(x, \lambda) g(x, \lambda)^\top]^{-1} \allowbreak \E\left[\frac{\partial g(x, \lambda)}{\partial \lambda}\right] \allowbreak \Big (\E\left[\frac{\partial g(x, \lambda)}{\partial \lambda}\right] \allowbreak \E[g(x, \lambda) g(x, \lambda)^\top]^{-1} \allowbreak \E\left[\frac{\partial g(x, \lambda)}{\partial \lambda}\right] \allowbreak + \Sigma_\Delta^{-1} \Big )^{-1} \allowbreak \E\left[\frac{\partial g(x, \lambda)}{\partial \lambda}\right] \allowbreak \E[g(x, \lambda) g(x, \lambda)^\top]^{-1} \allowbreak \E[g(x, \lambda)(x - \mu)^\top]$.
    
    The predictability of an algorithm $A$ with respect to $\mathcal{Q}_{\mathrm{lin}}$ under a convex, $H$-smooth loss $\ell$ is

    \vspace{-2\belowdisplayskip}
    {\small \begin{align*}
        \gamma &\leq \frac{H}{(1-\alpha) \alpha N} \left( \mathcal{G}(\mathcal{U}_{\mathrm{lin}, \Sigma_1}) + \sqrt{2 (1+\log (2)) \cdot \diam_{\Sigma_1}(\mathcal{U}_{\mathrm{lin}})} \right) \left( \mathcal{G}(\mathcal{U}_{\mathrm{lin}, \Sigma_1 - \Sigma_2}) + \sqrt{2 (1+\log (2)) \cdot \diam_{\Sigma_1 - \Sigma_2}(\mathcal{U}_{\mathrm{lin}})} \right).
    \end{align*}}
\end{restatable}

\begin{proof}
    For each $u \in \mathcal{U}_{\mathrm{lin}}$, $\gmm_{u \mid C(S)}$ estimates $\E[u^\top x \mid C(S)]$ and $\gmm_{u \mid C(S), A(S)}$ estimates $\E[u^\top x \mid C(S), A(S)]$. Consider $\gmm_{x \mid C(S)}$ which estimates $\E[x \mid C(S)]$ and $\gmm_{x \mid C(S), A(S)}$ estimates $\E[x \mid C(S), A(S)]$. \\

We want to bound $\sup_{u \in \mathcal{U}_{\mathrm{lin}}} \ell (\gmm_{u \mid C(S)}) - \ell (\gmm_{u \mid C(S), A(S)})$ where $\ell$ is the expected log loss.
\begin{align*}
    \sup_{u \in \mathcal{U}_{\mathrm{lin}}} \ell (\gmm_{u \mid C(S)}) - \ell (\gmm_{u \mid C(S), A(S)}) &\leq \sup_{u \in \mathcal{U}_{\mathrm{lin}}} H \cdot |\gmm_{u \mid C(S)} - \theta_0^u| \cdot |\gmm_{u \mid C(S)} - \gmm_{u \mid C(S), A(S)}| \\
    &= \sup_{u \in \mathcal{U}_{\mathrm{lin}}} H \cdot |u^\top (\gmm_{x \mid C(S)} - \mu)| \cdot |u^\top (\gmm_{x \mid C(S)} - \gmm_{x \mid C(S), A(S)})|
\end{align*}

From GMM and \pref{thm:dist-of-gmm-diff}, we know that $\sqrt{n} (\gmm_{x \mid C(S)} - \mu) \sim \mathcal{N}(0, \Sigma_1)$, $\sqrt{n} (\gmm_{x \mid C(S), A(S)} - \mu) \sim \mathcal{N}(0, \Sigma_2)$, and $\sqrt{n} (\gmm_{x \mid C(S)} - \gmm_{x \mid C(S), A(S)}) \sim \mathcal{N}(0, \Sigma_1 - \Sigma_2)$, $n = (1-\alpha) \alpha N$. Using \pref{thm:linear-concentration}, we can bound the $|u^\top (\gmm_{x \mid C(S)} - \mu)|$ and $|u^\top (\gmm_{x \mid C(S)} - \gmm_{x \mid C(S), A(S)})| $ with high probability. Thus, with probability at least $1 - \delta$,
\begin{align*}
    \sup_{u \in \mathcal{U}_{\mathrm{lin}}} \ell (\gmm_{u \mid C(S)}) - \ell (\gmm_{u \mid C(S), A(S)}) &\leq \sup_{u \in \mathcal{U}_{\mathrm{lin}}} H \cdot |u^\top (\gmm_{x \mid C(S)} - \mu)| \cdot |u^\top (\gmm_{x \mid C(S)} - \gmm_{x \mid C(S), A(S)})| \\
    &\leq \frac{H}{n} \cdot \left( \mathcal{G}(\mathcal{U}_{\mathrm{lin}, \Sigma_1}) + \sqrt{2 \log (2 / \delta) \cdot \sup_{u \in \mathcal{U}_{\mathrm{lin}}} u^\top \Sigma_1 u} \right) \\
    &~~~~~~\cdot \left( \mathcal{G}( \mathcal{U}_{\mathrm{lin}, \Sigma_1 - \Sigma_2}) + \sqrt{2 \log (2 / \delta) \cdot \sup_{u \in \mathcal{U}_{\mathrm{lin}}} u^\top (\Sigma_1 - \Sigma_2) u} \right) \\
    &= \frac{H}{n} \cdot \left( \mathcal{G}(\mathcal{U}_{\mathrm{lin}, \Sigma_1}) + \sqrt{2 \log (2  / \delta) \cdot \diam_{\Sigma_1}(\mathcal{U}_{\mathrm{lin}})} \right) \\
    &\;\;\;\;\;\cdot \left( \mathcal{G}( \mathcal{U}_{\mathrm{lin}, \Sigma_1 - \Sigma_2}) + \sqrt{2 \log (2  / \delta) \cdot \diam_{\Sigma_1 - \Sigma_2}(\mathcal{U}_{\mathrm{lin}})} \right).
\end{align*}

Let $X = \sup_{u \in \mathcal{U}_{\mathrm{lin}}} \ell (\gmm_{u \mid C(S)}) - \ell (\gmm_{u \mid C(S), A(S)})$. Note that predictability is $\gamma = \E[X]$. Thus, using standard techniques to integrate the tail, we can convert the high-probability bound into a bound on the expectation, yielding the final predictability bound. Let $\mathcal{G}_1 = \mathcal{G}(\mathcal{U}_{\mathrm{lin}, \Sigma_1})$, $\mathcal{G}_2 = \mathcal{G}( \mathcal{U}_{\mathrm{lin}, \Sigma_1 - \Sigma_2})$, $\diam_1 = \diam_{\Sigma_1}(\mathcal{U}_{\mathrm{lin}})$, and $\diam_2 = \diam_{\Sigma_1 - \Sigma_2}(\mathcal{U}_{\mathrm{lin}})$. We know
\begin{align*}
    \Pr[X > \frac{H}{n} \cdot \left( \mathcal{G}_1 + \sqrt{2 \log (2 / \delta) \cdot \diam_1} \right) \cdot \left( \mathcal{G}_2 + \sqrt{2 \log (2 / \delta) \cdot \diam_2} \right)] \leq \delta.
\end{align*}
Let $u = \log(2 / \delta)$, then
\begin{align*}
    \Pr[X > \frac{H}{n} \cdot \left( \mathcal{G}_1 + \sqrt{2\diam_1 u} \right) \cdot \left( \mathcal{G}_2 + \sqrt{2 \diam_2 u} \right)] &\leq 2 e^{-u} \\
    \Pr[X > f(u)] &\leq 2 e^{-u}
\end{align*}

Let $t_0 = f( \log (2))$. Thus,
\begin{align*}
    \E[X] &\leq \int_0^{t_0} 1 dt + \int_{t_0}^{\infty} \Pr[X > t] \\
    &= t_0 + \int_{t_0}^{\infty} \Pr[X > t] \\
    &= f( \log (2)) + \int_{\log (2)}^{\infty} \Pr[X > f(u)] f'(u) du \\
    &= f( \log (2)) + \int_{\log (2)}^{\infty} 2 e^{-u} f'(u) du \\
    &= f( \log (2)) - f( \log (2)) + \int_{\log (2)}^{\infty} 2 e^{-u} f(u) du \\
    &= \int_{\log (2)}^{\infty} 2 e^{-u} f(u) du \\
    &= \int_{0}^{\infty} e^{-u'} f(u' + \log (2)) du \tag{$u' = u - \log (2)$} \\
    &= \E[f(\beta + \log (2))] \tag{$\beta \sim \mathrm{Exp}(1)$} \\
    &= \frac{H}{n} \cdot \left( \mathcal{G}_1 \mathcal{G}_2 + \mathcal{G}_1 \E\left[\sqrt{2 \diam_2 (\beta + \log (2))}\right] + \mathcal{G}_2 \E\left[\sqrt{2\diam_1 (\beta + \log (2))} \right] + 2 \sqrt{\diam_1 \diam_2} \cdot \E \left[ \beta + \log (2)) \right]\right) \\
    &= \frac{H}{n} \cdot \left( \mathcal{G}_1 \mathcal{G}_2 + \mathcal{G}_1 \sqrt{2 \diam_2 (1 + \log (2))} + \mathcal{G}_2 \sqrt{2\diam_1 (1 + \log (2))}]  + 2 \sqrt{\diam_1 \diam_2} (1 + \log (2))\right) \tag{Jensen's inequality and $\E[\beta]=1$} \\
    &= \frac{H}{n} \cdot \left( \mathcal{G}_1 + \sqrt{2\diam_1 (1 + \log (2))} \right) \cdot \left( \mathcal{G}_2 + \sqrt{2 \diam_2 (1 + \log (2))} \right) \\
    &= \frac{H}{(1-\alpha) \alpha N} \left( \mathcal{G}(\mathcal{U}_{\mathrm{lin}, \Sigma_1}) + \sqrt{2 (1+\log (2 )) \cdot \diam_{\Sigma_1}(\mathcal{U}_{\mathrm{lin}})} \right) \left( \mathcal{G}(\mathcal{U}_{\mathrm{lin}, \Sigma_1 - \Sigma_2}) + \sqrt{2 (1+\log (2 )) \cdot \diam_{\Sigma_1 - \Sigma_2}(\mathcal{U}_{\mathrm{lin}})} \right).
\end{align*}

The exact form of $\Sigma_1$ and $\Sigma_2$ are described in \pref{thm:sigma-x-form}.
\end{proof}

\begin{theorem} \label{thm:sigma-x-form}
    Applying GMM with moment conditions
    \begin{align*}
        f_{C(S)}(x, \theta) = \begin{bmatrix}
            x - \mu
        \end{bmatrix}
        \text{ and }
        f_{C(S), A(S)}(x, \theta) = \begin{bmatrix}
        x - \mu \\
        g(x,\lambda)  \\
        \Tilde{\lambda} - \lambda
    \end{bmatrix},
    \end{align*}
    we have $\sqrt{n} (\gmm_{x \mid C(S)} - \mu) \sim \mathcal{N}(0, \Sigma_1)$, $\sqrt{n} (\gmm_{x \mid C(S), A(S)} - \mu) \sim \mathcal{N}(0, \Sigma_2)$, where
    \begin{align*}
        \Sigma_1 &= \V[x], \\
        \Sigma_2 &= \V[x] - \E[(x - \mu) g(x, \lambda)^\top] \E[g(x, \lambda) g(x, \lambda)^\top]^{-1} \E[g(x, \lambda)(x - \mu)^\top] + \Cx, \\
        \Sigma_1 - \Sigma_2 &= \E[(x - \mu) g(x, \lambda)^\top] \E[g(x, \lambda) g(x, \lambda)^\top]^{-1} \E[g(x, \lambda)(x - \mu)^\top] - \Cx
    \end{align*}
    where $\Cx = \E[(x - \mu) g(x, \lambda)^\top] \E[g(x, \lambda) g(x, \lambda)^\top]^{-1} \E\left[\frac{\partial g(x, \lambda)}{\partial \lambda}\right] \Big (\E\left[\frac{\partial g(x, \lambda)}{\partial \lambda}\right]^\top \E[g(x, \lambda) g(x, \lambda)^\top]^{-1} \E\left[\frac{\partial g(x, \lambda)}{\partial \lambda}\right] + \V[\Delta]^{-1} \Big )^{-1} \E\left[\frac{\partial g(x, \lambda)}{\partial \lambda}\right]^\top \E[g(x, \lambda) g(x, \lambda)^\top]^{-1} \E[g(x, \lambda)(x - \mu)^\top]$
\end{theorem}
\begin{proof}
    First, we establish the variance of the optimal predictor $\gmm_{x \mid C(S)}$ where $\theta = (\mu)$ and
    \begin{align*}
        f_{C(S)}(x, \theta) = \begin{bmatrix}
            x - \mu
        \end{bmatrix}.
    \end{align*}
    
    We need to compute $[G^\top \Omega^{-1} G]^{-1}$.
    \begin{align*}
        G &= \E \begin{bmatrix}
        \frac{\partial f(x, \theta)}{\partial \theta} \end{bmatrix} \\
        &= \begin{bmatrix}
            -I_d
        \end{bmatrix}
    \end{align*}
    \begin{align*}
        \Omega^{-1} &= \E \left[ f(x, \theta) f(x, \theta)^\top \right]^{-1} \\
        &= \E \left[(x - \mu)^2 \right]^{-1} \\
        &= \V[x]^{-1}
    \end{align*}
    \begin{align*}
        [G(\theta_0)^\top \Omega^{-1}(\theta_0) G(\theta_0)]^{-1} &= [I_d \V[x]^{-1} I_d]^{-1} \\
        &= \V[q(x)]
    \end{align*}
    Thus, $\Sigma_1 = \V[q(x)]$.

    Next, we compute the variance of the optimal predictor $\gmm_{x \mid C(S), A(S)}$ where $\theta = (\mu, \lambda)$ and
    $$f_{C(S), A(S)}(x, \theta) = \begin{bmatrix}
        x - \mu \\
        g(x,\lambda)  \\
        \Tilde{\lambda} - \lambda
    \end{bmatrix}. $$

We need to compute $[G^\top \Omega^{-1} G]^{-1}$.
    \begin{align*}
        G &= \E \begin{bmatrix}
        \frac{\partial f(x, \theta)}{\partial \theta} \end{bmatrix} \\
        &= \begin{bmatrix}
            -I_d & 0 \\
            0 & \E[\frac{\partial g(x, \lambda)}{\partial \lambda}] \\
            0 & -I_d
        \end{bmatrix}
    \end{align*}
    \begin{align*}
        \Omega^{-1} &= \E \left[ f(x, \theta) f(x, \theta)^\top \right]^{-1} \\
        &= \E \left[ \begin{bmatrix}
            (x - \mu)^2 & (x - \mu) g(x, \lambda)^\top & (x - \mu)(\Tilde{\lambda} - \lambda)^\top \\
            g(x, \lambda)(x - \mu)^\top & g(x, \lambda) g(x, \lambda)^\top & g(x, \lambda)(\Tilde{\lambda} - \lambda)^\top \\
            (\Tilde{\lambda} - \lambda)(x - \mu)^\top & (\Tilde{\lambda} - \lambda) g(x, \lambda)^\top & (\Tilde{\lambda} - \lambda)(\Tilde{\lambda} - \lambda)^\top
        \end{bmatrix}\right]^{-1} \\
        &= \begin{bmatrix}
            \V[x] & \E[(x - \mu) g(x, \lambda)^\top] & 0 \\
            \E[g(x, \lambda)(x - \mu)^\top] & \E[g(x, \lambda) g(x, \lambda)^\top] & 0 \\
            0 & 0 & \E[(\Tilde{\lambda} - \lambda)(\Tilde{\lambda} - \lambda)^\top]
        \end{bmatrix}^{-1}
    \end{align*}
    
    Let $U^{-1} = \begin{bmatrix}
        \V[x] & \E[(x - \mu) g(x, \lambda)^\top] \\
        \E[g(x, \lambda)(x - \mu)^\top] & \E[g(x, \lambda) g(x, \lambda)^\top]
    \end{bmatrix}^{-1}$.
    
    \begin{align*}
        \Omega^{-1} &= \begin{bmatrix}
            U^{-1}[1,1] & U^{-1}[1,2] & 0 \\
            U^{-1}[2,1] & U^{-1}[2,2] & 0 \\
            0 & 0 & \E[(\Tilde{\lambda} - \lambda)(\Tilde{\lambda} - \lambda)^\top]^{-1}
        \end{bmatrix}
    \end{align*}

    Putting it together,
    \begin{align*}
        G^\top \Omega^{-1} G &= \begin{bmatrix}
            -I_d & 0 & 0 \\
            0 & \E[\frac{\partial g(x, \lambda)}{\partial \lambda}] & -I_d
        \end{bmatrix}\begin{bmatrix}
            U^{-1}[1,1] & U^{-1}[1,2] & 0 \\
            U^{-1}[2,1] & U^{-1}[2,2] & 0 \\
            0 & 0 & \E[(\Tilde{\lambda} - \lambda)(\Tilde{\lambda} - \lambda)^\top]^{-1}
        \end{bmatrix} \begin{bmatrix}
            -I_d & 0 \\
            0 & \E[\frac{\partial g(x, \lambda)}{\partial \lambda}] \\
            0 & -I_d
        \end{bmatrix} \\
        &= \begin{bmatrix}
            -I_d & 0 & 0 \\
            0 & \E[\frac{\partial g(x, \lambda)}{\partial \lambda}] & -I_d
        \end{bmatrix} \begin{bmatrix}
            -U^{-1}[1,1] & U^{-1}[1,2] \E[\frac{\partial g(x, \lambda)}{\partial \lambda}] \\
            -U^{-1}[2,1] & U^{-1}[2,2] \E[\frac{\partial g(x, \lambda)}{\partial \lambda}] \\
            0 & -\E[(\Tilde{\lambda} - \lambda)(\Tilde{\lambda} - \lambda)^\top]^{-1}
        \end{bmatrix} \\
        &= \begin{bmatrix}
            U^{-1}[1,1] & -U^{-1}[1,2] \E[\frac{\partial g(x, \lambda)}{\partial \lambda}] \\
            -\E[\frac{\partial g(x, \lambda)}{\partial \lambda}] U^{-1}[2,1] & \E[\frac{\partial g(x, \lambda)}{\partial \lambda}] U^{-1}[2,2] \E[\frac{\partial g(x, \lambda)}{\partial \lambda}] + \E[(\Tilde{\lambda} - \lambda)(\Tilde{\lambda} - \lambda)^\top]^{-1}
        \end{bmatrix} \\
        &= \begin{bmatrix}
            U^{-1}[1,1] & -U^{-1}[1,2] \E[\frac{\partial g(x, \lambda)}{\partial \lambda}] \\
            -\E[\frac{\partial g(x, \lambda)}{\partial \lambda}] U^{-1}[2,1] & \E[\frac{\partial g(x, \lambda)}{\partial \lambda}] U^{-1}[2,2] \E[\frac{\partial g(x, \lambda)}{\partial \lambda}] + \V[\Delta]^{-1}
        \end{bmatrix}
    \end{align*}
    
    Using Schur's complement, we can compute the entries of $U^{-1}$. Let $U = \begin{bmatrix}
        A & B \\
        C & D
    \end{bmatrix}$, then
    \begin{align*}
        U^{-1}[1, 1] &= (A - BD^{-1}C)^{-1} = [\V[x] - \E[(x - \mu) g(x, \lambda)^\top] \E[g(x, \lambda) g(x, \lambda)^\top]^{-1} \E[g(x, \lambda)(x - \mu)^\top]]^{-1} \\
        U^{-1}[1, 2] &= -(A - BD^{-1}C)^{-1} BD^{-1} = -U^{-1}[1, 1] \cdot \E[(x - \mu) g(x, \lambda)^\top] \E[g(x, \lambda) g(x, \lambda)^\top]^{-1} \\
        U^{-1}[2, 1] &= -D^{-1}C (A - BD^{-1}C)^{-1} = -U^{-1}[1, 1] \cdot \E[g(x, \lambda) g(x, \lambda)^\top]^{-1} \E[g(x, \lambda)(x - \mu)^\top] \\
        U^{-1}[2, 2] &= D^{-1} + D^{-1} C (A - BD^{-1}C)^{-1} BD^{-1} \\
        &= \E[g(x, \lambda) g(x, \lambda)^\top]^{-1} \\
        &\;\;\;\;\;+ U^{-1}[1, 1] \E[g(x, \lambda) g(x, \lambda)^\top]^{-1} \E[g(x, \lambda)(x - \mu)^\top] \E[(x - \mu) g(x, \lambda)^\top] \E[g(x, \lambda) g(x, \lambda)^\top]^{-1}.
    \end{align*}
    
    The $[1, 1]$ entry of $[G^\top \Omega^{-1} G]^{-1}$ characterizes the variance of $\hat{p}$. For a matrix $M = \begin{bmatrix}
        A & B \\
        C & D
    \end{bmatrix}$,  $M^{-1}[1, 1] = A^{-1} + A^{-1}B (D - CA^{-1}B)^{-1} CA^{-1}$, using Schur's complement.
    \begin{align*}
        [G^\top \Omega^{-1} G]^{-1}[1,1] &= [U^{-1}[1,1]]^{-1} + U^{-1}[1,1]^{-1} U^{-1}[1,2] \E\left[\frac{\partial g(x, \lambda)}{\partial \lambda}\right] \bigg( \E\left[\frac{\partial g(x, \lambda)}{\partial \lambda}\right] U^{-1}[2,2] \E\left[\frac{\partial g(x, \lambda)}{\partial \lambda}\right]  \\
        &\;\;\;\;\;\; + \V[\Delta]^{-1} - \E\left[\frac{\partial g(x, \lambda)}{\partial \lambda}\right] U^{-1}[2,1] [U^{-1}[1,1]]^{-1} U^{-1}[1,2] \E\left[\frac{\partial g(x, \lambda)}{\partial \lambda}\right]\bigg )^{-1}  \E\left[\frac{\partial g(x, \lambda)}{\partial \lambda}\right] \\
        &\;\;\;\;\;\; \cdot U^{-1}[2,1] [U^{-1}[1,1]]^{-1} \\
        &= [U^{-1}[1,1]]^{-1} + [U^{-1}[1,1]]^{-1} U^{-1}[1,2] \E\left[\frac{\partial g(x, \lambda)}{\partial \lambda}\right] \\
        &\;\;\;\;\;\; \cdot \bigg (\E\left[\frac{\partial g(x, \lambda)}{\partial \lambda}\right] \E[g(x, \lambda) g(x, \lambda)^\top]^{-1} \E\left[\frac{\partial g(x, \lambda)}{\partial \lambda}\right] + \V[\Delta]^{-1} \bigg )^{-1}  \\
        &\;\;\;\;\;\;\cdot \E\left[\frac{\partial g(x, \lambda)}{\partial \lambda}\right] U^{-1}[2,1] [U^{-1}[1,1]]^{-1} \\
        &= [U^{-1}[1,1]]^{-1} + \E[(x - \mu) g(x, \lambda)^\top] \E[g(x, \lambda) g(x, \lambda)^\top]^{-1} \E\left[\frac{\partial g(x, \lambda)}{\partial \lambda}\right] \\
        &\;\;\;\;\;\;\cdot \bigg (\E\left[\frac{\partial g(x, \lambda)}{\partial \lambda}\right] \E[g(x, \lambda) g(x, \lambda)^\top]^{-1} \E\left[\frac{\partial g(x, \lambda)}{\partial \lambda}\right] + \V[\Delta]^{-1} \bigg )^{-1} \\
        &\;\;\;\;\;\;\cdot \E\left[\frac{\partial g(x, \lambda)}{\partial \lambda}\right] \E[g(x, \lambda) g(x, \lambda)^\top]^{-1} \E[g(x, \lambda)(x - \mu)^\top] \\
        &= [U^{-1}[1,1]]^{-1} + \Cx \\
        &= \V[x] - \E[(x - \mu) g(x, \lambda)^\top] \E[g(x, \lambda) g(x, \lambda)^\top]^{-1} \E[g(x, \lambda)(x - \mu)^\top] + \Cx
    \end{align*}
    where
    \begin{align*}
        \Cx &= \E[(x - \mu) g(x, \lambda)^\top] \E[g(x, \lambda) g(x, \lambda)^\top]^{-1} \E\left[\frac{\partial g(x, \lambda)}{\partial \lambda}\right] \\
        &\;\;\;\;\;\; \cdot \bigg (\E\left[\frac{\partial g(x, \lambda)}{\partial \lambda}\right] \E[g(x, \lambda) g(x, \lambda)^\top]^{-1} \E\left[\frac{\partial g(x, \lambda)}{\partial \lambda}\right] + \V[\Delta]^{-1} \bigg )^{-1} \\
        &\;\;\;\;\;\; \cdot \E\left[\frac{\partial g(x, \lambda)}{\partial \lambda}\right] \E[g(x, \lambda) g(x, \lambda)^\top]^{-1} \E[g(x, \lambda)(x - \mu)^\top].
    \end{align*}
    Thus, $\Sigma_2 = \V[x] - \E[(x - \mu) g(x, \lambda)^\top] \E[g(x, \lambda) g(x, \lambda)^\top]^{-1} \E[(x - \mu) g(x, \lambda)] + \Cx$.

    Finally, we have $\Sigma_1 - \Sigma_2 = \E[(x - \mu) g(x, \lambda)^\top] \E[g(x, \lambda) g(x, \lambda)^\top]^{-1} \E[(x - \mu) g(x, \lambda)] - \Cx$
\end{proof}

\begin{theorem} \label{thm:linear-concentration}
    Let $\xi_\Sigma \sim \mathcal{N}(0, \Sigma)$. With probability at least $1 - \delta$,
    \begin{align*}
        \sup_{u \in K} \langle u, \xi_\Sigma \rangle \leq \mathcal{G}(\Sigma^{1/2} K) + \sqrt{2 \log (1 / \delta) \cdot \sup_{u \in K} u^\top \Sigma u}
    \end{align*}
\end{theorem}
\begin{proof}
    Start by considering $\mathcal{G}_{\Sigma}(K)$,
    \begin{align*}
        \mathcal{G}_{\Sigma}(K) = \E \left[ \sup_{u \in K} \langle u, \xi_\Sigma \rangle \right].
    \end{align*}
    Let $\xi \sim \mathcal{N}(0, I_d)$. Note that $\xi_\Sigma = \Sigma^{1/2} \xi$. Thus, 
    \begin{align*}
        \mathcal{G}_{\Sigma}(K) &= \E \left[ \sup_{u \in K} \langle u, \xi_\Sigma \rangle \right] \\
        &= \E \left[ \sup_{u \in K} \langle u, \Sigma^{1/2} \xi \rangle \right] \\
        &= \E \left[ \sup_{u \in K} \langle \Sigma^{1/2} u, \xi \rangle \right] \tag{since $\Sigma^{1/2}$ is symmetric} \\
        &= \E \left[ \sup_{u' \in K_{\Sigma^{1/2}}} \langle u', \xi \rangle \right] \tag{$K_{\Sigma^{1/2}} = \{\Sigma^{1/2} u \mid u \in K\}$} \\
        &= \mathcal{G}(K_{\Sigma^{1/2}})
    \end{align*}
    Let $\xi \sim \mathcal{N}(0, I_d)$. The mapping $g(\xi) : \xi \mapsto \sup_{u \in K} \langle \Sigma^{1/2} u, \xi \rangle$ is Lipschitz with $L = \sup_{u \in K} \sqrt{u^\top \Sigma u}$. Thus, we can use the concentration for Lipschitz functionals of standard Gaussians (Borel-CIS inequality) to get a high probability bound.
    
    Let $Z = \sup_{u \in K} \langle u, \xi_\Sigma \rangle$. Then
    \begin{align*}
        \Pr[Z \geq \E[Z] + t] = \Pr[Z \geq \mathcal{G}(\Sigma^{1/2} K) + t] \leq e^{-t^2 / 2L^2}.
    \end{align*}

    Let 
    \begin{align*}
        \delta &= e^{-t^2 / 2L^2} \\
        \log (1 / \delta) &= \frac{t^2}{2L^2} \\
        2L^2 \log (1 / \delta) &= t^2 \\
        \sqrt{2L^2 \log (1 / \delta)} &= t
    \end{align*}

    Thus, with probability at least $1 - \delta$,
    \begin{align*}
        \sup_{u \in K} \langle u, \xi_\Sigma \rangle \leq \mathcal{G}(K_{\Sigma^{1/2}}) + \sqrt{2 \log (1 / \delta) \cdot \sup_{u \in K} u^\top \Sigma u}.
    \end{align*}
\end{proof}

We can use standard techniques to compute Gaussian complexity $\mathcal{G}(\Sigma^{1/2} \mathcal{Q}_{\mathrm{lin}})$. For example, if $\mathcal{Q}_{\mathrm{lin}}$ is the Euclidean ball of radius $B$, then $\mathcal{G}(\Sigma^{1/2} \mathcal{Q}_{\mathrm{lin}}) = B \sqrt{\mathrm{tr}(\Sigma)}$. If $\mathcal{Q}_{\mathrm{lin}}$ is the L1 ball of radius $B$, then $\mathcal{G}(\Sigma^{1/2} \mathcal{Q}_{\mathrm{lin}}) = B \sqrt{\log d} \cdot \max_i \sqrt{\Sigma[i,i]}$. We refer the reader to \cite{bartlett2003rademacherandgaussiancomplexities, kakade2008complexityoflinearpred} for more details.

\predeuclidball*
\begin{proof}
    From \pref{thm:pred-linear}, we know that
    \begin{align*}
        \gamma \leq \frac{H}{(1-\alpha) \alpha N} \left( \mathcal{G}(\mathcal{U}_{\mathrm{lin}, \Sigma_1}) + \sqrt{2 (1+\log (2 )) \cdot \diam_{\Sigma_1}(\mathcal{U}_{\mathrm{lin}})} \right) \left( \mathcal{G}(\mathcal{U}_{\mathrm{lin}, \Sigma_1 - \Sigma_2}) + \sqrt{2 (1+\log (2 )) \cdot \diam_{\Sigma_1 - \Sigma_2}(\mathcal{U}_{\mathrm{lin}})} \right)
    \end{align*}

    When $\mathcal{U}_{\mathrm{lin}}$ is contained within the Euclidean ball of radius $B$,
    \begin{align*}
        \mathcal{G}(\mathcal{U}_{\mathrm{lin}, \Sigma}) &= B \sqrt{\mathrm{tr}(\Sigma)}, \\
        \diam_{\Sigma} (\mathcal{U}_{\mathrm{lin}}) &= B^2 \cdot \lambda_{\max}(\Sigma).
    \end{align*}

    From \pref{thm:sigma-x-form},
    \begin{align*}
        \Sigma_1 &= \V[x], \\
        \Sigma_2 &= \V[x] - \E[(x - \mu) g(x, \lambda)^\top] \E[g(x, \lambda) g(x, \lambda)^\top]^{-1} \E[g(x, \lambda)(x - \mu)^\top] + \Cx, \\
        \Sigma_1 - \Sigma_2 &= \E[(x - \mu) g(x, \lambda)^\top] \E[g(x, \lambda) g(x, \lambda)^\top]^{-1} \E[g(x, \lambda)(x - \mu)^\top] - \Cx
    \end{align*}

    Then
    \begin{align*}
        \lambda_{\max}(\Sigma_1) &= \lambda_{\max}(\V[x])
    \end{align*}
    \begin{align*}
        \lambda_{\max}(\Sigma_1 - \Sigma_2) &= \lambda_{\max}(\E[(x - \mu) g(x, \lambda)^\top] \E[g(x, \lambda) g(x, \lambda)^\top]^{-1} \E[g(x, \lambda)(x - \mu)^\top] - \Cx) \\
        &\leq \lambda_{\max}(\E[(x - \mu) g(x, \lambda)^\top] \E[g(x, \lambda) g(x, \lambda)^\top]^{-1} \E[g(x, \lambda)(x - \mu)^\top]) - \lambda_{\min}(\Cx) \\
        &= \lambda_{\max}(\V[x] \V[x]^{-1}\E[(x - \mu) g(x, \lambda)^\top] \E[g(x, \lambda) g(x, \lambda)^\top]^{-1} \E[g(x, \lambda)(x - \mu)^\top]) \\
        &\;\;\;\;\;\;- \lambda_{\min}(\Cx) \\
        &\leq \lambda_{\max}(\V[x]) \cdot \lambda_{\max} (\V[x]^{-1}\E[(x - \mu) g(x, \lambda)^\top] \E[g(x, \lambda) g(x, \lambda)^\top]^{-1} \E[g(x, \lambda)(x - \mu)^\top]) \\
        &\;\;\;\;\;\;- \lambda_{\min}(\Cx) \\
        &= \lambda_{\max}(\Sigma_1) \cdot \cc^1 (x, g(x, \lambda))^2 - \lambda_{\min}(\Cx) \\
        &= \lambda_{\max}(\Sigma_1) \cdot \left( \cc^1 (x, g(x, \lambda))^2 - \frac{\lambda_{\min}(\Cx)}{\lambda_{\max}(\Sigma_1)} \right) \\
        &= \lambda_{\max}(\Sigma_1) \cdot \left( \cc^1 (x, g(x, \lambda))^2 - \frac{\lambda_{\min}(\Cx)}{\mathrm{tr}(\Sigma_1)} \right)
    \end{align*}

    Next, consider
    \begin{align*}
        \mathrm{tr}(\Sigma_1 - \Sigma_2) &= \mathrm{tr}(\E[(x - \mu) g(x, \lambda)^\top] \E[g(x, \lambda) g(x, \lambda)^\top]^{-1} \E[g(x, \lambda)(x - \mu)^\top] - \Cx) \\
        &= \mathrm{tr}(\E[(x - \mu) g(x, \lambda)^\top] \E[g(x, \lambda) g(x, \lambda)^\top]^{-1} \E[g(x, \lambda)(x - \mu)^\top]) - \mathrm{tr}(\Cx) \\
        &= \mathrm{tr}(\V[x] \V[x]^{-1}\E[(x - \mu) g(x, \lambda)^\top] \E[g(x, \lambda) g(x, \lambda)^\top]^{-1} \E[g(x, \lambda)(x - \mu)^\top]) - \mathrm{tr}(\Cx) \\
        &\leq \mathrm{tr}(\V[x]) \cdot \lambda_{\max} (\V[x]^{-1} \E[(x - \mu) g(x, \lambda)^\top] \E[g(x, \lambda) g(x, \lambda)^\top]^{-1} \E[g(x, \lambda) (x - \mu)^\top]) - \mathrm{tr}(\Cx) \tag{Holder's inequality for matrices} \\
        &= \mathrm{tr}(\Sigma_1) \cdot \cc^1 (x, g(x, \lambda))^2 - \mathrm{tr}(\Cx) \\
        &= \mathrm{tr}(\Sigma_1) \cdot \cc^1 (x, g(x, \lambda))^2 - \lambda_{\min}(\Cx) \\
        &= \mathrm{tr}(\Sigma_1) \cdot \left( \cc^1 (x, g(x, \lambda))^2 - \frac{\lambda_{\min}(\Cx)}{\mathrm{tr}(\Sigma_1)} \right)
    \end{align*}

    Thus,
    \begin{align*}
        \gamma &= \frac{H}{n} \cdot \left( \mathcal{G}(\mathcal{U}_{\mathrm{lin}, \Sigma_1}) + \sqrt{2(1 + \log (2)) \diam_{\Sigma_1} (\mathcal{U}_{\mathrm{lin}})}\right) \cdot \left( \mathcal{G}(\mathcal{U}_{\mathrm{lin}, \Sigma_1 - \Sigma_2}) + \sqrt{2(1 + \log (2)) \diam_{\Sigma_1 - \Sigma_2} (\mathcal{U}_{\mathrm{lin}})} \right) \\
        &\leq \frac{H}{n} \cdot \left( B \sqrt{\mathrm{tr}(\Sigma_1)} + B \sqrt{2(1 + \log (2)) \cdot \lambda_{\max}(\Sigma_1)}\right) \left( B \sqrt{\mathrm{tr}(\Sigma_1 - \Sigma_2)} + B\sqrt{2(1 + \log (2)) \cdot \lambda_{\max}(\Sigma_1 - \Sigma_2)} \right) \\
        &\leq \frac{HB}{n} \cdot \left( \sqrt{\mathrm{tr}(\Sigma_1)} + \sqrt{2(1 + \log (2)) \cdot \lambda_{\max}(\Sigma_1)}\right) \cdot \Bigg ( \sqrt{\mathrm{tr}(\Sigma_1) \cdot \left( \cc^1 (x, g(x, \lambda))^2 - \frac{\lambda_{\min}(\Cx)}{\mathrm{tr}(\Sigma_1)} \right)} \\
        &\;\;\;\;\;\;+ \sqrt{2(1 + \log (2)) \cdot \lambda_{\max}(\Sigma_1) \cdot \left( \cc^1 (x, g(x, \lambda))^2 - \frac{\lambda_{\min}(\Cx)}{\mathrm{tr}(\Sigma_1)} \right)} \Bigg ) \\
        &= \frac{HB}{n} \cdot \sqrt{\cc^1 (x, g(x, \lambda))^2 - \frac{\lambda_{\min}(\Cx)}{\mathrm{tr}(\Sigma_1)} } \cdot \left( \sqrt{\mathrm{tr}(\Sigma_1)} + \sqrt{2(1 + \log (2)) \cdot \lambda_{\max}(\Sigma_1)}\right)^2 \\
        &= \frac{2HB}{n} \cdot \sqrt{\cc^1 (x, g(x, \lambda))^2 - \frac{\lambda_{\min}(\Cx)}{\mathrm{tr}(\Sigma_1)}} \cdot \left( \mathrm{tr}(\Sigma_1) + 2(1 + \log (2)) \cdot \lambda_{\max}(\Sigma_1)\right) \\
        &\leq \frac{2HB}{n} \cdot \sqrt{\cc^1 (x, g(x, \lambda))^2 - \frac{\lambda_{\min}(\Cx)}{\mathrm{tr}(\Sigma_1)}} \left( \mathrm{tr}(\Sigma_1) + 2(1 + \log (2)) \cdot \mathrm{tr}(\Sigma_1) \right) \\
        &\leq \frac{4HB}{n} \cdot \sqrt{\cc^1 (x, g(x, \lambda))^2 - \frac{\lambda_{\min}(\Cx)}{\mathrm{tr}(\Sigma_1)}} \left((2 + \log (2)) \cdot \mathrm{tr}(\Sigma_1) \right) \\
        &= \frac{4HB (2 + \log (2)) \cdot \mathrm{tr}(\Sigma_1)}{n} \cdot \sqrt{\cc^1 (x, g(x, \lambda))^2 - \frac{\lambda_{\min}(\Cx)}{\mathrm{tr}(\Sigma_1)}} \\
        &= \frac{4HB (2 + \log (2)) \cdot \mathrm{tr}(\V[x])}{(1-\alpha) \alpha N} \cdot \sqrt{\cc^1 (x, g(x, \lambda))^2 - \frac{\lambda_{\min}(\Cx)}{\mathrm{tr}(\V[x])}}
    \end{align*}
\end{proof}

\subsubsection{\pref{thm:pred-lipschitz}, Lipschitz Queries}

\predlipschitz*
\begin{proof}
    We want to bound $\sup_{f \in \mathcal{Q}} \ell (\gmm_{f \mid C(S)}) - \ell (\gmm_{f \mid C(S), A(S)})$ where $\ell$ is the expected log loss.
\begin{align*}
    \sup_{f \in \mathcal{Q}} \ell (\gmm_{f \mid C(S)}) - \ell (\gmm_{f \mid C(S), A(S)}) &\leq \sup_{f \in \mathcal{Q}} H \cdot |\gmm_{f \mid C(S)} - \theta_0^f| \cdot |\gmm_{f \mid C(S)}) - \gmm_{f \mid C(S), A(S)}| \\
    &= \sup_{f \in \mathcal{Q}} H \cdot |f(\gmm_{x \mid C(S)}) - f(\mu^x)| \cdot |f(\gmm_{x \mid C(S)}) - f(\gmm_{x \mid C(S), A(S)})| \\
    &= HL^2 \cdot \norm{\gmm_{x \mid C(S)} - \mu} \cdot \norm{\gmm_{x \mid C(S)} - \gmm_{x \mid C(S), A(S)}}
\end{align*}

Let $n = (1-\alpha) \alpha n$. From GMM and \pref{thm:dist-of-gmm-diff}, we know that $\sqrt{n} (\gmm_{x \mid C(S)} - \mu) \sim \mathcal{N}(0, \Sigma_1)$, $\sqrt{n} (\gmm_{x \mid C(S), A(S)} - \mu) \sim \mathcal{N}(0, \Sigma_2)$, and $\sqrt{n} (\gmm_{x \mid C(S)} - \gmm_{x \mid C(S), A(S)}) \sim \mathcal{N}(0, \Sigma_1 - \Sigma_2)$. Thus, we can use standard concentration bounds for Gaussian random variables.

With probability at least $1-\delta$,
\begin{align*}
    \sup_{f \in \mathcal{Q}} \ell (\gmm_{f \mid C(S)}) - \ell &(\gmm_{f \mid C(S), A(S)}) \\
    &\leq HL^2 \cdot \norm{\gmm_{x \mid C(S)} - \mu} \cdot \norm{\gmm_{x \mid C(S)} - \gmm_{x\mid C(S), A(S)}} \\
    &\leq \frac{HL^2}{n} \cdot \left(\sqrt{\mathrm{tr}(\Sigma_1)} + \sqrt{2\lambda_{\max} (\Sigma_1) \log (2/\delta)} \right) \cdot \left(\sqrt{\mathrm{tr}(\Sigma_1 - \Sigma_2)} + \sqrt{2\lambda_{\max} (\Sigma_1 - \Sigma_2) \log (2/\delta)} \right)
\end{align*}

Let $X = \sup_{f \in \mathcal{Q}} \ell (\gmm_{f \mid C(S)}) - \ell (\gmm_{f \mid C(S), A(S)})$. Note that predictability is $\gamma = \E[X]$. Thus, using standard techniques to integrate the tail, we can convert the high-probability bound into a bound on the expectation, yielding the final predictability bound. Following the proof in \pref{thm:pred-linear},
\begin{align*}
    \gamma &= \E[X] \leq \frac{HL^2}{n} \cdot \left(\sqrt{\mathrm{tr}(\Sigma_1)} + \sqrt{2\lambda_{\max} (\Sigma_1) (1 + \log (2))} \right) \cdot \left(\sqrt{\mathrm{tr}(\Sigma_1 - \Sigma_2)} + \sqrt{2\lambda_{\max} (\Sigma_1 - \Sigma_2) (1+\log (2))} \right)
\end{align*}

We can further simplify this by following \pref{thm:pred-euclid-ball} and substituting in $\Sigma_1$ and $\Sigma_2$ from \pref{thm:sigma-x-form}
\begin{align*}
    \gamma &= \frac{4HL^2 (2 + \log (2)) \cdot \mathrm{tr}(\Sigma_1)}{(1-\alpha) \alpha N} \cdot \sqrt{\cc^1 (x, g(x, \lambda))^2 - \frac{\lambda_{\min}(\Cx)}{\mathrm{tr}(\Sigma_1)}}
\end{align*}
\end{proof}

\subsection{Helpful Theorems}

\begin{theorem}[\cite{mcfadden1999econometricsbook} Chapter 3, Lemma 3.1] \label{thm:opt-w}
    $(G^\top W G)^{-1} G^\top W \Omega W G (G^\top W G)^{-1} - (G^\top \Omega^{-1} G)^{-1}$ is positive semidefinite.
\end{theorem}

\begin{theorem} [\cite{mcfadden1999econometricsbook} Chapter 3, Lemma 3.2] \label{thm:func-of-uniform}
    If $Y_n \rightarrow_d Y_0$ and $f$ is a continuous function of an open set that contains the support of $Y_0$, then $$f(Y_n) \rightarrow_d f(Y_0).$$
\end{theorem}

\begin{theorem} [\cite{mcfadden1999econometricsbook} Chapter 3, Lemma 3.4] \label{thm:continous-mapping}
    If $Y_n (\theta) \rightarrow_p Y_0 (\theta)$ uniformly for $\theta \in \Theta \subseteq \mathbb{R}^d$, random vectors $\tau_0, \tau_n \in \Theta$ satisfy $\tau_n \rightarrow_p \tau_0$, and $Y_0(\theta)$ is almost surely continuous at $\tau_0$, then $$Y_n (\tau_n) \rightarrow_p Y_0(\tau_0).$$
\end{theorem}

\end{document}